\documentclass{article}

\usepackage{microtype}
\usepackage{graphicx}
\usepackage{subfigure}
\usepackage{booktabs} 

\usepackage{hyperref}

\usepackage{amsmath}
\usepackage{amssymb}
\usepackage{mathtools}
\usepackage{amsthm}

\usepackage[capitalize,noabbrev]{cleveref}

\usepackage{amsfonts}    
\usepackage{color}
\usepackage{multirow}
\usepackage{graphicx}
\usepackage{enumitem}
\usepackage{threeparttable}
\usepackage{tablefootnote}
\usepackage{makecell}
\usepackage{pifont}
\usepackage{comment}
\usepackage{caption}
\usepackage{float} 

\definecolor{mydarkgreen}{RGB}{39,130,67}

\definecolor{mydarkred}{RGB}{192,47,25}

\newcommand{\cmark}{{\small \color{mydarkgreen}\ding{51}}}%
\newcommand{\xmark}{{\small \color{mydarkred} \ding{55}}}%

\date{}

\usepackage{bm}

\newcommand{\mat}[1]{\bm{#1}}

\newcommand{\dotprod}[1]{\left< #1\right>}
\newcommand{\norm}[1]{ \|#1 \|}
 
\DeclareMathOperator{\argmininn}{argmin} 
 
\newcommand{\argmin}[1]{ \underset{#1}{\argmininn} \;}

\newcommand{\cP}{\mathcal{P}}
\newcommand{\cO}[1]{\mathcal{O}\left(#1\right)}

\newcommand{\mD}{\mat{ D}}

\newcommand{\R}{\mathbb{R}}

\newcommand{\E}[1]{\mathbb{E}\left[#1\right] } 
\newcommand{\EE}[2]{\mathbb{E}_{#1}\left[#2\right] } 

\usepackage{amsbsy}

\usepackage{thmtools}
\usepackage{thm-restate}

\usepackage{sidecap} 
\sidecaptionvpos{figure}{t}

\definecolor{shadecolor}{gray}{0.90}
\declaretheoremstyle[
headfont=\normalfont\bfseries,
notefont=\mdseries, notebraces={(}{)},
bodyfont=\normalfont,
postheadspace=0.5em,
spaceabove=6pt,
mdframed={
  skipabove=8pt,
  skipbelow=8pt,
  hidealllines=true,
  backgroundcolor={shadecolor},
  innerleftmargin=4pt,
  innerrightmargin=4pt}
]{shaded}

\declaretheorem[style=shaded,within=section]{definition}
\declaretheorem[style=shaded,sibling=definition]{theorem}
\declaretheorem[style=shaded,sibling=definition]{proposition}

\declaretheorem[style=shaded,sibling=definition]{lemma}
\declaretheorem[style=shaded,sibling=definition]{remark}

\usepackage[colorinlistoftodos,bordercolor=orange,backgroundcolor=orange!20,linecolor=orange]{todonotes}

\newcommand{\guillaume}[1]{\todo[inline]{\textbf{Guillaume: }#1}}

\definecolor{kleinblue}{RGB}{0, 47, 167}
\hypersetup{
	colorlinks = True,
	linkcolor = kleinblue,
	citecolor=kleinblue,
	urlcolor=kleinblue
} 

\newcommand{\SGD}{\texttt{SGD}}
\newcommand{\SPS}{\texttt{SPS}}

\newcommand{\IAM}{\texttt{IAM}}

\newcommand{\gamSPS}{\gamma^{\SPS{*}}}

\newcommand{\GPT}{\texttt{GPT2}}

\usepackage[accepted]{icml/icml2025}

\usepackage[T1]{fontenc}

\icmltitlerunning{Idealized Polyak and Black-Box distillation}

\begin{document}

\twocolumn[
\icmltitle{Analysis of an Idealized Stochastic Polyak Method and its Application to Black-Box Model Distillation}



\icmlsetsymbol{equal}{*}

\begin{icmlauthorlist}
\icmlauthor{Robert M. Gower}{yyy}
\icmlauthor{Guillaume Garrigos}{comp}
\icmlauthor{Nicolas Loizou}{sch}
\icmlauthor{Dimitris Oikonomou}{sch}
\icmlauthor{Konstantin Mishchenko}{meta}
\icmlauthor{Fabian Schaipp}{xx}
\end{icmlauthorlist}

\icmlaffiliation{yyy}{Flatiron Institute, New York.}
\icmlaffiliation{comp}{Université Paris Cité and Sorbonne Université, CNRS,
Laboratoire de Probabilités, Statistique et Modélisation, Paris.}
\icmlaffiliation{sch}{AMS \& MINDS
Johns Hopkins University.}
\icmlaffiliation{meta}{Meta, Paris.}
\icmlaffiliation{xx}{Inria, Departement d’Informatique de l’Ecole Normale Superieure, PSL Research University, Paris}

\icmlcorrespondingauthor{Robert M. Gower}{rgower@flatironinstitute.org}


\vskip 0.3in
]



\printAffiliationsAndNotice{}  

\begin{abstract}
We provide a general convergence theorem of an idealized stochastic Polyak step size called \SPS*. Besides convexity, we only assume a local expected gradient bound, that includes locally smooth and locally Lipschitz losses as special cases. We refer to \SPS* as idealized because it requires access to the loss for every training batch evaluated at a solution. It is also ideal, in that it achieves the optimal lower bound for globally Lipschitz function, and is the first Polyak step size to have an $\mathcal{O}(1/\sqrt{t})$ anytime convergence in the smooth setting.  We show how to combine \SPS* with momentum to achieve the same favorable rates for the last iterate. We  conclude with several experiments to validate our theory, and a more practical setting  showing how we can distill a teacher GPT-2 model into a smaller student model without any hyperparameter tuning.
\end{abstract}

\section{Introduction}
Consider the problem 
\begin{equation} \label{eq:prob}
x_* \in \argmin{x\in \R^d} f(x), \quad f(x):=\EE{\xi\sim \cP}{f_{\xi}(x)}, 
\end{equation} 
where $\cP$ is a distribution over the data, the loss functions $f_\xi$ are convex.
We assume that the minimizer  $x_*\in\R^d$ exists, and that $\mathbb{E}_\xi \left[ \inf f_\xi \right] > - \infty$ (e.g. the losses are nonnegative).

One of the main costs in developing new machine learning models is training them, that is, finding an approximate solution to~\eqref{eq:prob}. 
The training of GPT-4  is estimated to have cost over \$40M~\cite{cottier2024risingcoststrainingfrontier}. The elevated cost of training bigger models, and the success of Adam~\cite{adam}, has sparked an intense research effort into developing new stochastic optimization methods.  Yet the performance difference among many newly developed methods is minimal when the step size is tuned~\cite{pmlr-v139-schmidt21a}. 
Finding a good step size often involves multiple re-runs on a subset of the data, which adds considerably to this cost. 

Here we advance the theory of an adaptive stochastic Polyak step size. The Polyak step size uses both the current loss and gradient norm to compute a step size at each iteration.

We show that if we had access to
$f_{\xi}(x_*)$, the value of the loss at the solution for each batch $\xi$ of data, a variant of the stochastic Polyak step we call \SPS* achieves the best known rates across several subclasses of convex functions. 
Specifically, we show that \SPS*
achieves either the optimal rate when known, or the best known rate, for convex functions, including Lipschitz, smooth, and strongly convex. Furthermore we only require that these assumptions hold in a ball around the solution.  This mirrors the same result in the deterministic setting for the  Polyak step size~\cite{hazan2022revisiting}.

We also prove convergence in the finite-sum, convex and continuous setting, without any additional assumption, for which we are unaware of any other stochastic method that provably converges. 

We then show how to combine this Polyak step size with momentum, in such a way that the last-iterate converges at the optimal (competitive) rate in the Lipschitz (smooth) setting.  
For this we use \emph{iterate averaging}, which is one of the many equivalent ways of writing momentum \citep{pmlr-v134-sebbouh21a}.



These fast and adaptive convergence results speak to the strength of the \SPS* method. However, they also show that having access to $f_{\xi}(x_*)$ for every $\xi$ is a strong assumption, which we can not expect to hold in general. But we do consider two settings where $f_{\xi}(x_*)$ is known or can be approximated. The first setting is that of interpolation, where typically $f_{\xi}(x_*) =0$ or is relatively easy to compute~\cite{SPS}. The second setting is one we call \emph{blackbox model distillation}. In this setting, we can query a teacher (a larger pretrained model) with any input, but we do not have access to the teachers architecture or weights. Our objective is to train the student (a smaller model) on one of the tasks that the teacher is accomplished. The teacher's loss on each input serves as an approximation of $f_{\xi}(x_*)$ for the student. This enables us to use \SPS* with momentum to set the step size for the student, and train it efficiently without having to tune any hyper-parameters.


\subsection{Stochastic Polyak Step Size}

Here we analyse the following variant of the \SPS{} (Stochastic Polyak step size) method
\begin{equation}\label{eqn:sps-iter} 
    x_{t+1} \; = \; x_t - \gamSPS_t g_t, \;  \; \gamSPS_t := 
        \frac{(f_{t}(x_t) - f_{t}(x_*))_+}{\|g_t\|^2}
\end{equation} 
where $\xi_t\sim \cP$ is sampled i.i.d at each iteration, and $g_t$ denotes either a gradient (smooth setting) or a subgradient (non-smooth setting) of $f_{t}:=f_{\xi_t}$ evaluated at $x_t$. Throughout, we use the notation $(z)_+:=\max\{z,0\}$ for $z\in \R$.
We refer to~\eqref{eqn:sps-iter} as a the \SPS* method. 
We will prove several anytime convergence rates for \SPS*.  By \emph{anytime}, we mean a proof that the method converges to
any predefined tolerance without prior knowledge of that tolerance.



 See~\Cref{tab:sps-compare} for a comparison between our rates of convergence, that of other variants of \SPS{}, and the best known anytime rates for \SGD{} in each setting. For the \SGD{} rates within each setting, we included rates that rely on the global problem constants. For instance, to achieve the $GD/\sqrt{t}$ rate in the $G$-Lipschitz setting, we need to set the step size as $\gamma = \frac{D}{G}\frac{1}{\sqrt{t}}$, 
 and we need to project the iterates of \SGD{} back onto the ball of radius $D:= \|x_0-x_*\|.$ 
 In contrast, \SPS* achieves this rate without  without access to $G$ or $D$, but with access to $f_{\xi}(x_*)$ instead.

 
The main downside to~\eqref{eqn:sps-iter} is that it requires access  to $f_{t}(x_*)$. This is why we refer to \SPS{*} as an idealized variant, both because of its ideal convergence rates, and this idealized setting of assuming access to 
  $f_{t}(x_*)$.  In this sense, the comparisons in \cref{tab:sps-compare} to alternative Polyak type methods are not entirely fair, because they do not require such access to  $f_{t}(x_*)$. Our message here is not that \SPS{*} is a better method than $\SPS_{\max}$, \texttt{NGN} or \texttt{DecSPS}, but rather that $f_{t}(x_*)$ is the object that we should try to approximate, or learn on the fly.

Despite our claim that \SPS{*} is an idealized method, we do consider two settings where access to, or approximating, $f_{t}(x_*)$ is reasonable. 
One setting where $f_{t}(x_*)$ is often known  is the interpolation setting, where we assume that there exists a minimizer $x_* \in \R^d$ such that the loss over every data is simultaneously minimized, in other words
\begin{equation}\label{eq:interpolation}
    f_{\xi}(x_*) \; = \; \inf_{x\in\R^d} f_{\xi}(x), \quad \forall \xi \in \mbox{support }(\cP).
\end{equation} 
Thus under interpolation, our model has a perfect fit (as measured by  $f_{\xi}(x)$) for every data point.
Typically the loss is a non-negative function and its infimum is zero~\cite{SPS}, that is 
$ \inf_{x\in\R^d} f_{\xi}(x) =0 $. 
When this is the case, we have access to every $f_{\xi}(x_*) $, which happens to be zero. Alternatively when $\inf  f_{\xi}(x)$ is close to zero, then using zero as approximation is reasonable. Finally, even when $\inf  f_{\xi}(x)$ is far from zero, it can sometimes be efficiently approximated~\cite{SPS}. 

 The ease of approximating $\inf  f_{\xi}(x)$ is what motivated  $\SPS{_{\max}}$~\cite{SPS} which uses the step size
   \begin{equation}\label{eq:SPSmax-intro}
     \gamma^{\SPS_{\max}}_t := 
       \min \left\{ \frac{f_{t}(x_t) -\inf_x f_{t}(x)}{\|g_t\|^2}, \gamma_b \right\},
\end{equation}
where  $\gamma_b >0$ is an additional hyperparameter to safe-guard against excessively large step sizes.
\citet{SPS} present a comprehensive analysis of $\SPS{_{\max}}$ in the non-smooth, smooth and strongly convex setting. But in all these cases, $\SPS{_{\max}}$ is only guaranteed to converge when interpolation holds. Outside of interpolation, $\SPS{_{\max}}$ converges to a neighborhood of the solution. Here we show that it is not necessary to assume that interpolation holds to establish convergence of a \SPS{} type method. Having access to $f_{\xi}(x_*)$ is sufficient.

To be clear,  assuming access to 
$f_{\xi}(x_*)$ is not the same as assuming that interpolation holds. Interpolation~\eqref{eq:interpolation} imposes constraints on the data and the model, usually requiring the model to be overparameterized~\cite{pmlr-v80-ma18a,LIU202285,SGDstruct}. In contrast having access to $f_{\xi}(x_*)$ imposes no constraints on the model and data. Furthermore, there are settings outside of interpolation where $f_{\xi}(x_*)$ can be known or reasonably approximated, such as model distillation which we consider in~\Cref{sec:distillation}.

As a secondary objective of our work, we also present \IAM{} (Iterate Averaging Adaptive method), a variant of \SPS* with momentum.   We prove that in the smooth and Lipschitz setting the \emph{last} iterate of \IAM{} converges as fast as the \emph{average} iterate of  \SPS*. As the last iterate is usually more relevant in practice, this is the first time that a version of \SPS{} with momentum has some theoretical advantage.

\begin{table*}[t]
    \centering
\setlength\tabcolsep{3.5pt} 
  \begin{threeparttable}[b]
    {
	\renewcommand\arraystretch{1.8}
	\caption{ \small A summary of anytime convergence rates for variants of stochastic Polyak step size. Notation: $D = \|x_0-x_*\|,$ $\Delta_* = \inf f - \E{ \inf f_{\xi}} $, $\Delta_{pos} = \E{ \inf f_{\xi}}$,  $G^2 = \max_{x} \EE{\xi}{\| \nabla f_{\xi}(x)\|^2}.$ We compare to the stochastic Polyak methods  \texttt{DecSPS}$^{\tnote{\color{green}(3)}}$,   \SPS{$\max$} \cite{SPS}  and \texttt{NGN} \cite{orvieto2024NGN}. The proof of convergence for \SPS* in the Lipschitz convex and strongly convex setting was first given in \cite{garrigos2023handbook} and~\cite{pedregosa2023sps2}, respectively. 
    For the results making use of the gradient variance constant $\sigma_*^2 = \mathbb{V}_\xi \left[ \nabla f_\xi(x_*) \right]$, we replaced it with its upper bound $L \Delta_*$ for a more uniform comparison.}
  \label{tab:sps-compare}
	\centering 
        {
	\begin{tabular}{ccccccc}\toprule[.1em]
            \textbf{Algorithm} & \makecell{\textbf{Convex}\\ \textbf{finite sum}} & \makecell{ \textbf{$G$-Lipschitz} \\ \textbf{problems}} & \makecell{ $L$-\textbf{Smooth} \\ \textbf{problems}}  & \makecell{$L$-\textbf{Smooth}\\ $\mu$-\textbf{Convex} } & \makecell{$G$-Lipschitz\\ $\mu$-\textbf{Convex} } \\
            \midrule
            \texttt{DecSPS}$^{\tnote{\color{green}(3)}}$  & \xmark & \xmark & \xmark & $\frac{L D^2 + \Delta_*}{\sqrt{t}}$ & \xmark   \\
            \SPS{$\max$} & \xmark & \xmark &  $\frac{LD^2}{t} + \Delta_* L$ 
            &  $\left(1-\frac{\mu}{L}\right)^t D^2 + \frac{\Delta_* L}{\mu} $ 
  & \xmark   \\
            \texttt{NGN} & \xmark & \xmark 
            &  $\frac{L^2 D^2}{\sqrt{t}} + \frac{L(\Delta_* + L\Delta_{pos})\log(t)}{\sqrt{t}}$& \cmark\tnote{\color{purple}(4)}  & \xmark  \\[7pt]
               \hline 
            \SGD$^*$\tnote{\color{blue}(2)} & \xmark  
            & $\frac{G D}{\sqrt{t}}$ 
            &  $\frac{LD^2}{\sqrt{t}} + \frac{ \Delta_* \log(t)}{\sqrt{t}}$ 
            & $\frac{L \Delta_*}{\mu^2} \frac{1}{t} + \frac{L^2D^2}{\mu^2t^2}$ & $\frac{B^2}{\mu^2} \frac{1}{t}$  \\[7pt]
              \hline 
            \SPS*  
             & \makecell{ $\frac{G D}{\sqrt{t}}$\tnote{\color{red}(1)} \\ \Cref{rem:finitesum} } & \makecell{ $\frac{G D}{\sqrt{t}}$ \\ \Cref{cor:spsnonsmooth}  } & \makecell{$\frac{LD^2}{t} + \frac{ \sqrt{L\Delta_*}D}{\sqrt{t}}$ \\ \Cref{cor:spssmooth} } & \makecell{ $\frac{\Delta_*}{\mu^2} \frac{1}{t}$\tnote{\color{yellow}(5)}\\ \Cref{thm:better-bounds-sps-strong} } & \makecell{$\frac{B^2}{\mu^2} \frac{1}{t}$    \\ \Cref{thm:better-bounds-sps-strong} } \\[0.3cm]  
             \IAM{}  (new)
             & \makecell{ $\frac{G D}{\sqrt{t}}$\tnote{\color{red}(1)} \\ \Cref{rem:finitesum} } & \makecell{ $\frac{G D}{\sqrt{t}}$ \\ \Cref{theo:nonsmooth} } & \makecell{$\frac{LD^2 \log(t+1)}{t} + \frac{\sqrt{L\Delta_*}D}{\sqrt{t}}$ \\ \Cref{theo:smooth}} & \xmark  & \xmark  \\
		\bottomrule[.1em]
    \end{tabular}
    }
    }
    \begin{tablenotes}
      {\footnotesize
        \item [\color{red}(1)] The convex finite sum result  assumes  $\EE{\xi}{f_{\xi}} = \frac{1}{n} \sum_{i=1}^n f_i$
        \item [\color{blue}(2)] $\SGD^*$ denotes \SGD{} where we can use all the global constants $D, G, L, \Delta_*$ and $\mu$ to set the step size. For the left to right, these results can be found in 
        Thm.\ 9.12~\cite{garrigos2023handbook},
        Thm.\ 4.1~\cite{SGDstruct},
        Thm.\ 3.1~\cite{gower_sgd_2019}, Section 3.2~\cite{lacostesimple1t}. 
        \item [\color{green}(3)]  Under the additional assumption that the iterates of \texttt{DecSPS} are bounded, we have from \cite{Orvieto2022}  that \texttt{DecSPS} converges at a $\cO{1/\sqrt{t}}$ rate in the $G$-Lipschitz and $L$-smooth setting.  
        \item [\tnote{\color{purple}(4)}] The paper claims an $\cO{\log(t)/t}$ anytime rate is possible, but does not give the explicit proof or constants.
        \item [\tnote{\color{yellow}(5)}] Here we have an anytime rate only for $t \geq \cO{\frac{L}{\mu}}$
      }
    \end{tablenotes}
  \end{threeparttable}
\end{table*}

Next we describe the related work to ours, and use the context to detail our specific contributions. See~\Cref{tab:compare_with_others} for a high-level resume of our results.

\subsection{Related Work and Contributions}

\paragraph{Polyak step size.}
The Polyak step size was first introduced by \citet{Polyak87} in the deterministic setting, where he also proved convergence for non-smooth and convex functions.
\citet{hazan2022revisiting} revisited the Polyak step size and showed that for the class of gradient descent methods (where we can only choose the step size),  it has  near-optimal convergence rate in the  Lipschitz, smooth, and strongly convex setting. Furthermore, it is optimal without having access to any of the Lipschitz ($G$), smoothness ($L$),  or strong convexity ($\mu$) parameters. Recently,  the proof of convergence in the smooth setting has been generalized to a broader class of relatively smooth functions~\cite{takezawa2024} and locally smooth functions~\cite{richtarik2024localcurvaturedescentsqueezing}.
In the smooth and strongly convex setting, \citet{pmlr-v125-barre20a} show how to accelerate gradient descent with the Polyak step size, and without having access to the strong convexity parameter, but estimating it instead.


\paragraph{The stochastic Polyak step size.}

The current research into the stochastic Polyak step size was kick-started by the \texttt{ALI-G} method~\cite{ALI-G} and $\SPS_{\max}$~\cite{SPS}.
Both \texttt{ALI-G} and $\SPS_{\max}$ offered a practical stochastic variant of the Polyak step size with strong empirical results to support their use. In terms of convergence theory, for smooth and convex functions, $\SPS_{\max}$ was shown to converge to a neighborhood of the solution~\cite{SPS}. To enforce that $\SPS_{\max}$ does converge in the smooth setting, \citet{Orvieto2022} proposed the \texttt{DecSPS} method that combines $\SPS_{\max}$ with a decreasing step size sequence, and show that if the stochastic loss functions are strongly convex and smooth, then suboptimality converges at a rate of $\mathcal{O}(1/\sqrt{T})$, where $T$ is the number of iterations. This rate is slower than \SGD{} in the same setting, which is $\cO{1/T}.$

As for \SPS*, \citet{garrigos2023function} showed that it converges with the optimal rate in the Lipschitz non-smooth setting. Convergence in the smooth setting was shown in \cite{garrigos2023function,SGDstruct}, but under interpolation.

A proximal version of \SPS{} was introduced in~\cite{schaipp2023a}  in order to handle regularization terms.
More recently, a new variant of $\SPS{}$ called \texttt{NGN} was introduced in~\cite{orvieto2024NGN} for specifically non-negative functions.  \texttt{NGN} uses a combination of Gauss-Newton and truncation to introduce a dampened version of the Polyak step sizes. Though \texttt{NGN} also converges to a neighborhood of the solution for smooth functions, \citet{orvieto2024NGN} prove a $\mathcal{O}(1/\sqrt{T})$  and $\cO{1/T}$ complexity for convex and strongly convex functions, respectively.  \citet{orvieto2024NGN} also give a $\mathcal{O}(\log(T)/\sqrt{T})$ anytime result in the smooth and convex setting. 

\emph{Contributions.} 
We present a unifying anytime convergence in the smooth and non-smooth setting in~\Cref{thm:better-bounds-sps} for \SPS*. Besides convexity, \Cref{thm:better-bounds-sps} only makes local 
 assumptions and thus applies to a broader class of functions as compared to prior results.
 We then specialize this result into the locally Lipschitz and locally smooth setting in~\Cref{cor:spsnonsmooth} and~\Cref{cor:spssmooth}, respectively. Our proof also leverages a new trick, where we explicitly invert a convex monotone function (\Cref{L:psi rate reciprocal}). We show how this trick is used in  a sketch of the proof of~\Cref{thm:better-bounds-sps} in Section~\ref{sec:sketchproof}. 
 Finally, our convergence result in the smooth setting in~\Cref{cor:spssmooth} is the first $\mathcal{O}(1/\sqrt{T})$ anytime result which benefits from interpolation (see details in Section \ref{S:discussion}).

\paragraph{Momentum.}
\citet{Polyak1964} introduced the momentum method through the heavy-ball viewpoint. In the deterministic setting, \citet{Polyak1964} showed that it converges at an accelerated rate for strongly convex quadratic functions. Only rather recently, a global convergence was established for smooth and non-smooth functions without strong convexity~\cite{ghadimi2014HB}.

In the stochastic setting, there is little to no theoretical advantage for using momentum for \SGD{}, unless we consider the specialized setting of minimizing a quadratic~\cite{LeePaquettePaquette2024,Bollapragadamomentum2024}. The main theoretical improvement from using momentum in the stochastic setting for general convex functions is  that the last iterate $x_t$ of momentum converges at the same favourable rate as the average iterate of the \SGD{} iterates~\cite{pmlr-v134-sebbouh21a,pmlr-v157-defazio21a}. The analysis in~\cite{pmlr-v134-sebbouh21a} relies on an equivalent reformulation of momentum known as the \emph{iterate  averaging} viewpoint,
which we also use in this work. Recent online-to-batch conversion techniques can also achieve the same rate of convergence of the last iterate of \SGD{} without momentum, albeit with slightly worse constants~\cite{pmlr-v97-cutkosky19a}.  These online-to-batch techniques rely on monotonic step sizes, and thus are not applicable to Polyak-type step sizes.
\paragraph{Stochastic Polyak with momentum.}
In the stochastic setting, some very recent works have considered different ways of blending \SPS{} with momentum~\cite{Schaipp2024a,pmlr-v202-wang23l}. 
The first analysis of a variant of \SPS{} with momentum was  developed in~\citet{pmlr-v202-wang23l}. Their \texttt{ALR-SMAG} method is the result of choosing a learning rate that minimizes a particular upper bound on $\|x_{t+1}-x_*\|$ for the iterates of momentum or heavy-ball. The current analysis for \texttt{ALR-SMAG}  shows that it has a slower convergence as compared to \SPS{} unless $\beta_t=0$, which corresponds to using no momentum.
The same issue holds for the recently introduced \texttt{MoMo} method~\cite{Schaipp2024a},
which empirically reduces the tuning effort for the learning rate across many tasks, but theoretically has best bounds with no momentum, that is, when the method is equal to \SPS{}. 
Another recent approach that combines \SPS{} with momentum is proposed by \citet{oikonomou2024stochastic}, introducing \texttt{MomSPSmax} and its variants, \texttt{MomDecSPS} and \texttt{MomAdaSPS}. These step sizes guarantee convergence in the stochastic setting without relying on the interpolation condition. Instead, they assume in addition that the iterates remain bounded. Specifically, \texttt{MomSPSmax} achieves an $\mathcal{O}(1/t)$ convergence rate to a neighborhood of the solution, while \texttt{MomDecSPS} and \texttt{MomAdaSPS} converge to the exact solution with a rate of $\mathcal{O}(1/\sqrt{t})$.

\emph{Contributions.} We prove that the last iterate of our momentum variant of \SPS{} (\cref{alg:IAM}) converges anytime in (i)~the convex and \emph{locally} Lipschitz case (see \Cref{theo:nonsmooth})  and (ii)~ the \emph{locally} smooth case (see ~\Cref{theo:smooth}). Furthermore, in the non-smooth setting, the convergence rate in~\Cref{theo:nonsmooth} is \emph{at least as fast} as the corresponding rate for \SPS* in \Cref{cor:spsnonsmooth}. 


\paragraph{Adaptive methods.} Historically, line search procedures, such as Armijo line search \cite{armijo1966minimization}, used to be commonly employed to estimate the smoothness around the current point when the exact smoothness constant $L$ was not known. More recent works have shown that it is also possible to estimate the value of $L$ using the previously observed gradients \cite{malitsky20adaptive,latafat2024adaptive}. Furthermore, in the last decade, more line-search \cite{nesterov14universal} and bisection \cite{carmon2022making} methods have been proposed that adapt simultaneously to smooth and non-smooth objectives. Unfortunately, most of the approaches either don't have strong guarantees in the stochastic case or require large batch sizes.

In online learning, when the Lipschitz constant of the objective is known, coin-betting approaches \cite{orabona16coin} can be used to adaptively estimate distances to a solution. When the Lipschitz constant is not known, one can either use restarts \cite{mhammedi2020lipschitz}, which require a lot of extra work, or use a technique called hints \cite{cutkosky19artifical}, but the latter introduces even more hyperparameters.

\texttt{AdaGrad}~\cite{streeter12noregret,duchi2011adaptive} and its variants offer an alternative by estimating the gradient magnitudes instead of estimating smoothness. These methods can be combined with momentum and achieve strong complexity results, but they either require bounded domain \cite{levy2018online,kavis2019unixgrad} or are only studied in the deterministic setting \cite{li2023simple}. Furthermore, most variants use step sizes that can only decrease over time, meaning they will not adapt if the problem curvature becomes flatter. This has been partially addressed by a series of new methods that have an increasing estimate of distances to the solution set \cite{defazio23learning,ivgi23dog,khaled2023dowg}, but their stochastic guarantees are provided only for large batch sizes \cite{ivgi23dog} or the interpolation setting. We compare to the most relevant of these works in Table~\ref{tab:compare_with_others}.

\emph{Contributions.} Our theoretical results show that the \SPS* method is adaptive to the following settings and parameters: smoothness ($L$), initial distance ($D$), Lipschitz ($G$) and strong convexity $(\mu).$ The precise definition of these parameters and constants are given later.


\section{Stochastic Polyak Step Size}



Before giving our convergence proofs, we first will motivate \SPS{*} as the step size that minimizes an upper bound on the distance to a minimizer. Suppose we are at iteration $t$, have drawn a batch of data $\xi_t$, and let $g_t: =g_{\xi_t}(x_t)$ be the stochastic (sub)gradient evaluated at $x_t$. For short-hand we will also use $f_t:= f_{\xi_t}.$ Consider an iterate of \SGD,
\[x_{t+1} = x_t - \gamma_t g_t,\]
where $\gamma_t>0$ is the step size. The subgradient $g_t\in \partial f_{t}(x_t),$ by definition satisfies
\begin{equation} \label{eq:convexitysub}
     f_{t}(x) \geq  f_{t}(x_t) + \dotprod{g_t, x-x_t}, \quad \forall x\in \R^d.
\end{equation}

Now consider the task of choosing $\gamma_t$ that brings $x_{t+1}$ as close as possible to the solution $x_*.$  In general, this is impossible since we do not know $x_*$. However, we can minimize the upper bound
\begin{align}
    &\hspace{-4ex}\norm{x_{t+1}-x_*}^2 - \norm{x_{t}-x_*}^2 \nonumber \\
    & = -2\gamma_t \dotprod{g_t, x_t-x_*} + \gamma_t^2 \norm{g_t}^2 \nonumber \\
     & \leq
     -2\gamma_t (f_{t}(x_t) -f_{t}(x_*) ) + \gamma_t^2 \norm{g_t}^2,    \label{eq:iterdistance}
\end{align}
where we use \eqref{eq:convexitysub} in the inequality.
Minimizing the right-hand side under the constraint $\gamma_t \geq 0$ gives the step size
\begin{equation}\label{eq:SPS}
    \gamSPS_t 
    \;=\;
        \frac{(f_{t}(x_t) - f_{t}(x_*))_+}{\|g_t\|^2},
\end{equation}  
 which together with \SGD{} gives the  \SPS* method~\eqref{eqn:sps-iter}. 

Note that in \eqref{eq:SPS} we divide by the squared norm of the stochastic gradient, which could be equal to zero. This is only possible if $(f_{t}(x_t) - f_{t}(x_*))_+=0$, so we define  $\gamSPS_t:=0$ if $g_t=0$. That is, if the stochastic gradient is zero,  no step is taken.

\subsection{Convergence Theory for Convex Problems} \label{sec:spsconv}

We now give in \cref{thm:better-bounds-sps} our unifying convergence theorem for \SPS*, that aside from convexity only assumes in~\eqref{spscv2} that the expected norm of the stochastic gradients verifies a certain bound within 
\[
\mathbb{B}_D(x_*) := \{x\in \R^d: \; \norm{x-x_*} \leq D\}, \ D:= \norm{x_0-x_*}.
\]
 Because our proof makes use of a new technical lemma that may find uses elsewhere, we give a sketch of the proof in~\Cref{sec:sketchproof}. The full proof is also in~\Cref{sec:better-bounds-sps}.

\begin{restatable}[Convergence of \SPS*]{theorem}{theospsgen}\label{thm:better-bounds-sps}
Consider problem~\eqref{eq:prob} and 
let $(x_t)_{t \geq 0}$ be the iterates of \SPS* given by \eqref{eqn:sps-iter}.
Then the iterates are almost surely monotone:
\begin{equation*}
    \norm{x_{t+1} -x_*}^2 \; \leq \; \norm{x_t -x_*}^2 \; \text{with probability 1}.
\end{equation*}
If there exists $A,B \geq 0$ with $A+B\neq 0$ and such that
\begin{equation}\label{spscv2}
    \EE{\xi}{\norm{g_{\xi}(x)}^2}
    \leq 
    A(f(x) - f(x_*)) + B
\end{equation}
for all $x \in \mathbb{B}_D(x_*)$, then the averaged iterates of \SPS*
$\bar x_T := \tfrac{1}{T} \sum_{t=0}^{T-1} x_t$ 
verify: 
\begin{equation*}
    \mathbb{E}\left[   f(\bar x_T) - \inf f \right]
    \leq 
    \frac{D^2 A}{T}
    +
    \sqrt{\frac{D^2B}{T}}.
\end{equation*}
\end{restatable}
Next we specialize~\Cref{thm:better-bounds-sps} to the non-smooth and smooth settings, 
where the local bound \eqref{spscv2} is always true.
The proof of the next two corollaries can be found in Appendix \ref{S:proof SPS nonsmooth} and \ref{S:proof SPS smooth}.

\begin{restatable}[Non-smooth setting]{corollary}{corspsnonsmooth} \label{cor:spsnonsmooth}
Consider problem~\eqref{eq:prob}, and assume that the losses $f_\xi$ are locally Lipschitz in expectation (see Definition \ref{D:local lipschitz expectation}). In particular   there exists $G \geq 0$ such that
\begin{align} \label{eq:exp-lipschitz}
    \EE{\xi}{\norm{g_{\xi}(x)}^2} & \leq  G^2, 
\end{align}
for all $x \in \mathbb{B}_D(x_*)$. Then the averaged iterates of \SPS*
$\bar x_T := \tfrac{1}{T} \sum_{t=0}^{T-1} x_t$ 
verify: 
\begin{equation*}\label{eq:spsnonsmooth}
    \mathbb{E}\left[ f(\bar x_T) - \inf f \right]
    \leq
    \frac{G D}{\sqrt{T}}. 
\end{equation*}
\end{restatable}

\begin{restatable}[Smooth setting]{corollary}{corspssmooth}\label{cor:spssmooth}
Consider problem~\eqref{eq:prob}, and assume that the losses $f_\xi$ are uniformly locally smooth (see Definition~\ref{D:locally smooth uniformly}). In particular there exists $L \geq 0$ s.t.
 \begin{equation}\label{eq:expsmooth-sps}
    \EE{\xi}{\norm{ \nabla f_{\xi}(x)}^2} \leq  2 L \big( f(x)- \mathbb{E}_\xi \left[ \inf f_\xi \right] \big), 
\end{equation} 
for all $x \in \mathbb{B}_D(x_*)$. Then the averaged iterates of \SPS*
$\bar x_T := \tfrac{1}{T} \sum_{t=0}^{T-1} x_t$ 
verify, with $\Delta_* := \inf f - \mathbb{E}_\xi \left[ \inf f_\xi \right]$:
\begin{equation*}\label{eq:spssmooth}
    \mathbb{E}\left[ f(\bar x_T) - \inf f \right]
    \leq
    \frac{2LD^2}{T} + \frac{\sqrt{2L \Delta_*} D}{\sqrt{T}}.
\end{equation*}
\if{
Consider problem~\eqref{eq:prob}, and assume that the losses $f_\xi$ are uniformly locally smooth (see Definition~\ref{D:locally smooth uniformly}). In particular there exists $L \geq 0$ s.t.
 \begin{equation}\label{eq:expsmooth-sps}
    \EE{\xi}{\norm{ \nabla f_{\xi}(x) - \nabla f_{\xi}(x^*)}^2} \leq  2 L \big( f(x)-\inf f \big), 
\end{equation} 
for all $x \in \mathbb{B}_D(x_*)$. Then the averaged iterates of \SPS*
$\bar x_T := \tfrac{1}{T} \sum_{t=0}^{T-1} x_t$ 
verify, with $\sigma_*^2 := \mathbb{E}_\xi \left[ \Vert \nabla f_\xi(x_*)\Vert^2 \right]$:
\begin{equation*}\label{eq:spssmooth}
    \mathbb{E}\left[ f(\bar x_T) - \inf f \right]
    \leq
    \frac{4LD^2}{T} + \frac{\sqrt{2} D\sigma_*}{\sqrt{T}}.%
\end{equation*}
}\fi 
\end{restatable}

For the non-smooth setting,  it is typically assumed that the loss functions are \emph{globally} Lipschitz, {uniformly} with respect to $\xi$, which in turn gives a global bound on the stochastic subgradients. 
Here instead we require very little: the convexity of our losses entails that $f_\xi$ is $G_\xi$-Lipschitz on $\mathbb{B}_D(x^*)$ (see Proposition \ref{P:convex functions are locally lipschitz}), so we only need to assume that the expectation $\EE{\xi}{G_\xi}$ is finite.
An advantage of this local Lipschitz assumption is that it is \emph{always  true for finite sums} (see Proposition \ref{P:finite family always loclip expectation}).
Another advantage of our local assumption is that it is compatible with strong convexity. 
Indeed there is no function which is both globally Lipschitz and strongly convex, see e.g. Lemma 9.13 in~\cite{garrigos2023handbook}.
We provide an additional result analogous to \cref{thm:better-bounds-sps} in the strongly convex setting in~\Cref{sec:stronglycvx}.

Despite this additional generality, we achieve a $\mathcal{O}(1/\sqrt{T})$ convergence rate which is the optimal lower bound for the class of convex Lipschitz functions~\cite{drori2016optimal}.
Currently this rate can only be achieved by combining adaptive methods such as AdaGrad~\cite{duchi2011adaptive} together with knowing and using $\| x_0-x_*\|$ to set the learning rate or a projection radius~\cite{orabonabook}. In contrast, our oracle requires knowing $f_{\xi}(x_*)$. 
We note that a weaker version of~\Cref{cor:spsnonsmooth} was first established in 
\citep[Thm.\ 2.3]{garrigos2023function}, where the losses $f_\xi$ are assumed to be globally Lipschitz. 

As for the smooth setting, it is typically assumed in the literature that the loss functions $f_{\xi}$ are \emph{globally} smooth, uniformly with respect to $\xi$ \cite{GowerRichBach2018,SGDstruct,gower_sgd_2019}.
Our result instead requires much less: all we need is that the losses $f_\xi$ are \emph{locally} smooth, and that their local smoothness constants are uniformly bounded with respect to $\xi$. 
We defer to \Cref{L:local smooth uniformly expected smoothness} in the appendix for a formal proof that such local smoothness implies \eqref{eq:expsmooth-sps}.
In particular, it is remarkable that our assumption is \emph{always true for finite sums} of $C^2$ losses (see Proposition \ref{P:C2 sum finite is uniformly locally smooth}).


Our smooth result in~\Cref{cor:spssmooth} is, as far as we know, the first $\mathcal{O}(1/\sqrt{T})$ anytime convergence rate for a stochastic variant of the Polyak step size, assuming only smoothness and convexity. Note that \SGD{} has a $\cO{\log(T)/\sqrt{T}}$ anytime rate in this setting, see~\Cref{sec:sps-smooth-details}.

Another interesting aspect of the convergence rate in~\Cref{cor:spsnonsmooth} is that it benefits from interpolation: when $\Delta_*=0$ the convergence rate in~\Cref{cor:spsnonsmooth} becomes $\mathcal{O}(1/T),$ which is the expected accelerated rate of \SGD{} under interpolation~\cite{fastersgd2019}.
We are unaware of prior work that establishes an anytime rate of convergence that is adaptive to interpolation.  
As a comparison, we illustrate in~\Cref{theo:SGDadaptsigma} in the appendix how the \emph{complexity} of \SGD{} can benefit from interpolation, to the price of knowing the value of the interpolation constant $\Delta_*$. 
We further contrast our rate to the best known anytime rate for \SGD{} and \texttt{SPS}$_{\max}$ in~\Cref{sec:sps-smooth-details}.


\begin{remark}[Finite sum] \label{rem:finitesum}
We emphazise that for finite-sum minimization $f = \tfrac{1}{m}\sum_{i=1}^m f_i$, our assumptions are drastically simplified.
The local Lipschitz assumption in \Cref{cor:spsnonsmooth} is automatically true ; and the local smoothness assumption in \Cref{cor:spssmooth} is true if the $f_i$ are $C^2$.
\end{remark}

We have shown that
\SPS* has  the optimal rate of convergence in the non-smooth setting, and  a fast adaptive anytime rate in the smooth setting. This motivates us to think of \SPS* as an idealized variant of the stochastic Polyak step size. In~\Cref{sec:approx} we show how several practical variants of the stochastic Polyak step size, that do not need access to $f_{\xi}(x_*),$ can be viewed as approximations of \SPS*.

\section{Momentum and the Iterate Moving Average Method}
\label{sec:iams}

The \SPS* method is missing one important and practical ingredient, which is momentum. Furthermore, our previous convergence only holds for the average iterate, whereas the last iterate is often preferred since it is used in practice.

Momentum is often presented as replacing the gradient with an exponential moving average of gradients as follows: for $\gamma_t > 0$ and $\beta_t \in [0,1)$, let
\begin{align}\label{eq:momentum}
    m_t = \beta_t m_{t-1} + g_t, \quad x_{t+1}  = x_t - \gamma_t m_t.
\end{align}
To derive our momentum variant of \SPS{}, 
we will make use of the equivalent reformulation given by
\begin{eqnarray}
z_{t} & = & z_{t-1}-\eta_{t}g_t,\label{eq:zup}\\
x_{t+1} & = &\frac{\lambda_{t+1}}{1+\lambda_{t+1}} x_{t}+\frac{1}{1+\lambda_{t+1}}z_{t}, \label{eq:xup}
\end{eqnarray}
where $\eta_t >0$ and $\lambda_t \in [0,1]$ are hyperparameters. Though not obvious, the $x_t$ iterates in~\eqref{eq:xup} are equivalent to the $x_t$ iterates of Momentum~\eqref{eq:momentum} by choosing a particular mapping between $ (\beta_t, \gamma_t)$ and $ (\lambda_{t}, \eta_{t})$, see \citet[Thm.\ 1]{pmlr-v157-defazio21a} and~\cref{lem:mom-and-iterate-av} for convenience.


Inspired by both~\citet{pmlr-v202-wang23l} and~\citet{Schaipp2024a}, we  now choose the learning rate $\eta_t$ in \eqref{eq:zup} that minimizes an upper bound on $D_t:=\norm{z_t -x_*}^2$.
We have
\begin{align}
   D_t\
    & = 
    D_{t-1}  -2\eta_t \dotprod{g_t, z_{t-1}-x_*} + \eta_t^2 \norm{g_t}^2.  \nonumber 
\end{align}
Using convexity we have that
\begin{align*}
     \dotprod{g_t, z_{t-1}-x_*} &=  \dotprod{g_t,x_t-x_*} + \dotprod{g_t,z_{t-1}-x_t}  \\
     & \geq  f_{t}(x_t) -f_{t}(x_*)  + \dotprod{g_t,z_{t-1}-x_t} . \end{align*}
     With this bound we have that
    \begin{align}
    \begin{split}
     & \hspace{-2ex} D_t \leq  D_{t-1}+ \eta_t^2 \norm{g_t}^2 \\
     & -2\eta_t \big[ f_{t}(x_t) -f_{t}(x_*)  + \dotprod{g_t,z_{t-1}-x_t} \big].
    \end{split}
    \label{eq:expsquare}
\end{align}
We will use this upper bound to choose an adaptive learning rate. Minimizing the right-hand side over $\eta_t\geq 0 $ gives
\begin{eqnarray}\label{eq:IMA-step}
    \eta_t = \frac{\big[ f_{t}(x_t) -f_{t}(x_*) +\dotprod{g_t,z_{t-1}-x_t}\big]_+}{\norm{g_t}^2}.
\end{eqnarray}
We refer to~\eqref{eq:xup} with the learning rate~\eqref{eq:IMA-step} as the \emph{Iterate Averaging Adaptive Method} (\IAM{}) method, for which we give the complete pseudo-code in \cref{alg:IAM}.


%
\begin{algorithm}[t]
    \caption{\IAM:\hspace{-0.5mm} Iterate Averaging Adaptive Method. }
    \label{alg:IAM}
\begin{algorithmic}
    \STATE \textbf{Input:} $z_{-1}=x_0 \in \R^d$, $\lambda_t > 0$, for $t = 0, \ldots, T$
    \FOR{$t=0$ to $T-1$}
    	\STATE $ \eta_t \;= \; \displaystyle\frac{\big[ f_{t}(x_t) -f_{t}(x_*)+\dotprod{g_t,z_{t-1}-x_t}\big]_+}{\norm{g_t}^2} $ 
		\STATE $ z_{t} \; = \; z_{t-1}-\eta_{t}\nabla f_{t}(x_{t}) $ \label{ln:zup}
		\STATE $x_{t+1}  \;= \; \displaystyle \frac{\lambda_{t+1}}{1+\lambda_{t+1}} x_{t}+\frac{1}{1+\lambda_{t+1}}z_{t} 
$ \label{ln:xup}
    \ENDFOR
    \STATE \textbf{Return:} $x_T$
\end{algorithmic}
\end{algorithm}

\subsection{Convergence Theorems}
Again  $D_t =  \norm{z_{t} - x_*}^2$.
Our proofs all start from plugging in the step size~\eqref{eq:IMA-step} into~\eqref{eq:expsquare} giving
\begin{align*}
   D_t\ & \leq
   D_{t-1}  -\frac{\big(f_{t}(x_t) -f_{t}(x_*)  + \dotprod{g_t,z_{t-1}-x_t} \big)_+^2}{ \norm{g_t}^2}. 
\end{align*}
\begin{lemma}\label{lem:iam-iterates-monotone}
    Let $f_{\xi}$ be convex for every $\xi$. The distances of iterates $z_t$ of \cref{alg:IAM} to a solution $x_* \in \R^d$ decreases monotonically, that is, with probability one
    \begin{eqnarray*}
        \norm{z_{t} - x_*}^2 \; \leq \; \norm{z_{t-1} - x_*}^2 \; \leq \; \cdots \; \leq \;  \norm{z_{0} - x_*}^2.
    \end{eqnarray*}
\end{lemma}
This type of monotonicity for stochastic methods is very rare, with the only other example that we are aware of being \SPS* (cf.\ \cref{thm:better-bounds-sps}). To complete the convergence proofs, we will telescope the recurrence on $D_t$ and bound the gradient norm on the denominator. 

\subsection{Non-smooth Setting}
For our first result
we consider the setting where $f_{\xi}$ could be non-smooth,  thus  $g_{\xi}$ denotes a subgradient of $f_{\xi}$.

\begin{restatable}[Non-smooth setting]{theorem}{theononsmooth}\label{theo:nonsmooth}
Consider problem~\eqref{eq:prob},
and let $D := \Vert x_0 - x_* \Vert$.
Assume that the losses $f_\xi$ are locally Lipschitz in expectation. In particular there exists $G \geq 0$ such that \eqref{eq:exp-lipschitz} holds 
for all $x \in \mathbb{B}_D(x_*)$. 
Then the iterates of \cref{alg:IAM} (\IAM{}) started from $x_0$ with  $\lambda_t =t$ verify the \emph{last iterate} bound
\begin{align*}
\E{f(x_T) -f(x_*)}  &+\frac{1}{T+1}\sum_{t=1}^T t \,\E{B_{f}(x_{t-1},x_t) }  \nonumber\\
 & \quad  \leq  \frac{GD}{\sqrt{T+1}}, 
\end{align*}  
where $B_{f}(x,y) := f(x) -f(y)  -\dotprod{\nabla f(y),x-y}.$
\if{
    Consider the iterates of \IAM{} in \cref{alg:IAM} with the learning rate~\eqref{eq:IMA-step} and $\lambda_t =t.$  
     Let $f_{\xi}$ be convex for all $\xi.$ Let $D:=\norm{x_{0} - x_*}$,
\begin{align*}
    G^2 &\;:=\; \max_{x \in \mathbb{B}_D(x_*) } \; \mathbb{E}_{\xi}\norm{g_{\xi}(x)}^2,\\
    B_{f}(x,y) & := f(x) -f(y)  -\dotprod{\nabla f(y),x-y}.
\end{align*}
    The suboptimality of the \emph{last iterate} $x_T$ is bounded by
    \vspace{-0.2cm}
  \begin{align} \label{eq:nonsmoothconv}
\E{f(x_T) -f(x_*)}  &+\frac{1}{T+1}\sum_{t=1}^T t \,\E{B_{f}(x_{t-1},x_t) }  \nonumber\\
 & \quad  \leq  \frac{GD}{\sqrt{T+1}}. 
\end{align}  
}\fi 
\end{restatable} 

This rate of convergence is the same as \SPS* (\cref{cor:spsnonsmooth}), with two notable differences.
First this rate holds for the last iterate, as opposed to the average of the iterates.
Second, this rate for \IAM{} can be faster than that of \SPS* due to the additional Bregman divergence term.

\Cref{theo:nonsmooth} restricts the parameter choice of $\lambda_t =t,$ 
which when translated back (See~\Cref{Asec:iterate-averaging} for details) to the  momentum method~\eqref{eq:momentum} restricts the parameters $(\gamma_t, \beta_t)$ to $\beta_t  =  \frac{t}{t+1 } \frac{\eta_{t-1} }{\eta_t}$ and $\gamma_t =\frac{\eta_t}{t+2}$ for all $t$.
To allow for other parameter settings, we provide Thm.~\ref{thm:iaam-relax-nonsmooth} in the Appendix, which allows for any deceasing $(\lambda_t)_t$, but  does not establish a last-iterate convergence. 

\begin{figure*}[h!]
    \centering
     \includegraphics[width=0.7\textwidth]{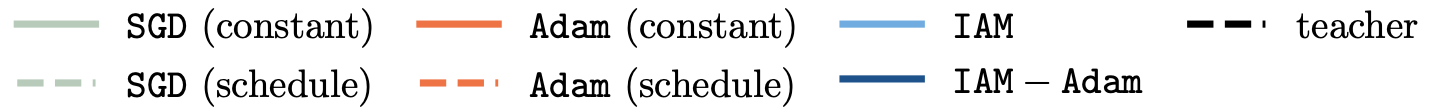}\\
    \begin{minipage}[t]{0.32\textwidth}
        \centering
        \includegraphics[width=\textwidth, trim=0 0 0 0, clip]{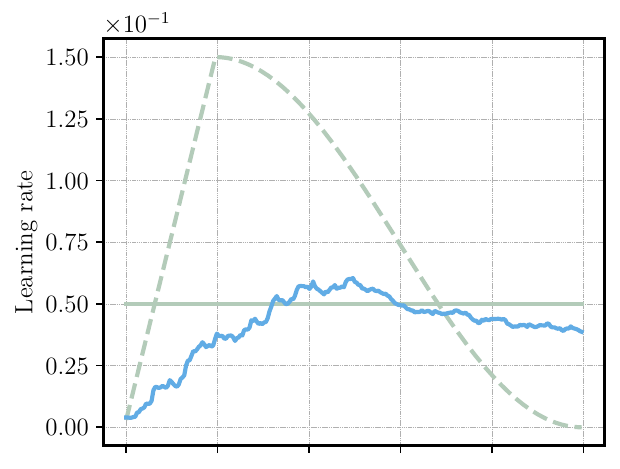}
    \end{minipage}
    \hfill
    \begin{minipage}[t]{0.32\textwidth}
        \centering
        \includegraphics[width=\textwidth, trim=0in 0 0 0, clip ]{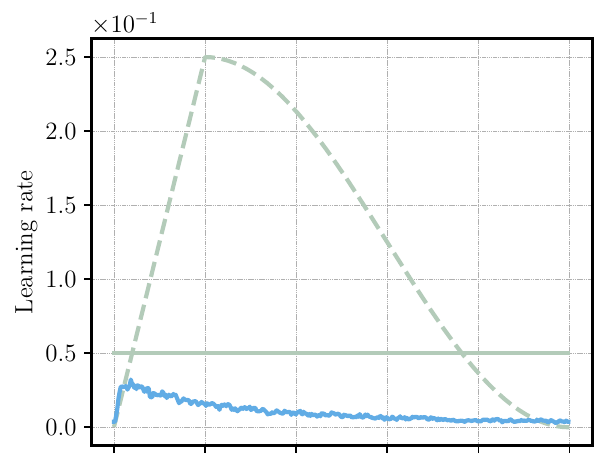}
    \end{minipage}
    \hfill
    \begin{minipage}[t]{0.32\textwidth}
        \centering
        \includegraphics[width=\textwidth, trim=0in 0 0 0, clip]{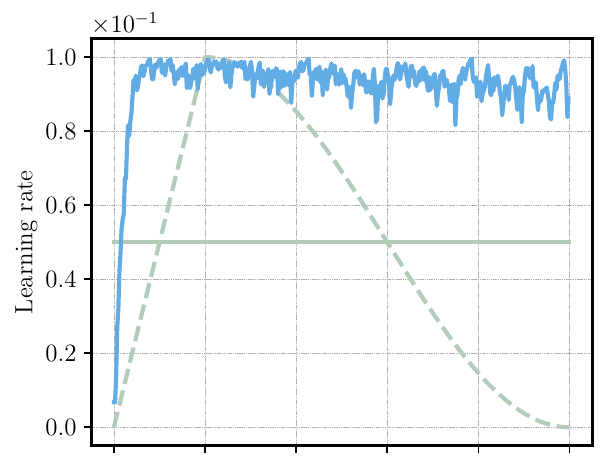}
    \end{minipage}

    \begin{minipage}[t]{0.32\textwidth}
        \centering
        \includegraphics[width=\textwidth, trim=0 0 0 0, clip]{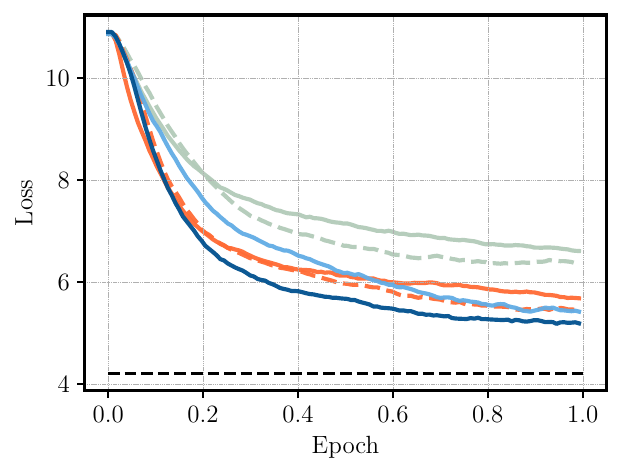}
        {\small  \texttt{tinyShakespeare}}
    \end{minipage}
    \hfill
        \begin{minipage}[t]{0.32\textwidth}
        \centering
        \includegraphics[width=\textwidth, trim=0in 0.0 0 0, clip]{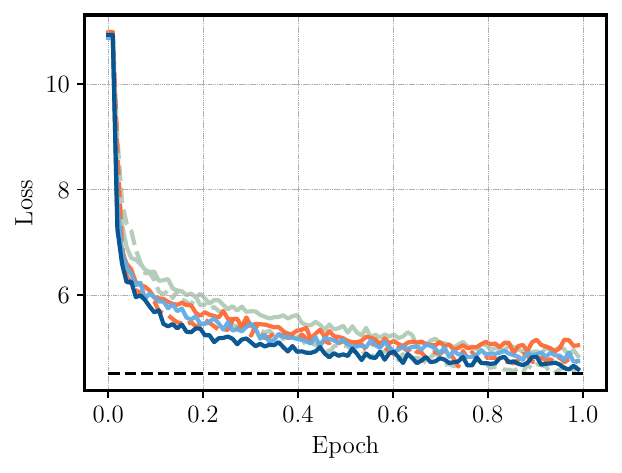}
        {\small  \texttt{PTB}}
    \end{minipage}
    \hfill
    \begin{minipage}[t]{0.32\textwidth}
        \centering
        \includegraphics[width=\textwidth, trim=0 0 0 0, clip]{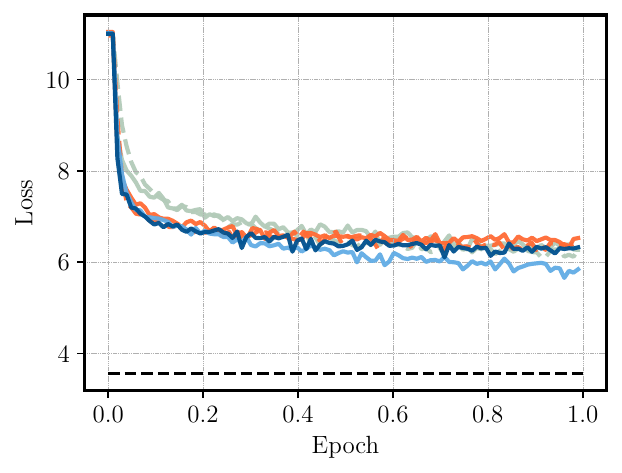}
        {\small  \texttt{Wikitext2}}
    \end{minipage}

    \caption{Distilling a teacher \GPT{} on three datasets. Adaptive learning rate of \IAM{} and learning rates of \SGD{} \textbf{(top)} and cross-entropy training loss \textbf{(bottom)}. Black line marks the average teacher loss.}
    \label{fig:distill}
\end{figure*}

\subsection{Smooth Setting}
Here we consider the setting where we assume that the loss functions $f_{\xi}$ satisfy a local \emph{smoothness} condition.

\begin{restatable}[Smooth setting]{theorem}{theosmooth}\label{theo:smooth}
Consider problem~\eqref{eq:prob}, and let $D:= \Vert x_0 - x_* \Vert$.
Assume that the losses $f_\xi$ are uniformly locally smooth. In particular there exists $L \geq 0$ such that
\eqref{eq:expsmooth-sps} holds for all $x \in \mathbb{B}_D(x_*)$. 
Then the iterates of \cref{alg:IAM} (\IAM{}) started from $x_0$ with  $\lambda_t =t$ verify this \emph{last iterate} bound, with $\Delta_* := \inf f - \mathbb{E}_\xi \left[ \inf f_\xi \right]$:
   \begin{align*}
        \E{f(x_{T-1}) -f(x_*)}
       &\le \frac{2LD^2 (\log(T)+1)}{T} + \frac{\sqrt{2L\Delta_*}D}{\sqrt{T}}. 
   \end{align*}
\if{ 
Let $f_{\xi}$ be convex for all $\xi$. 
Assume local \emph{expected smoothness}~\eqref{eq:expsmooth-sps}
holds.
Let $x_t$ be the iterates of  \cref{alg:IAM} (\IAM{}) with $\lambda_t=t$. It holds 
   \begin{align}\label{eq:theosmooth}
        \E{f(x_{T-1}) -f(x_*)}
       &\le \frac{2LD^2 (\log(T)+1)}{T} \nonumber  \\
       &\quad + \frac{\sqrt{2L\sigma_*^2}D}{\sqrt{T}}.
   \end{align}
}\fi
\end{restatable}
Analogous to the \SPS* result in \Cref{cor:spssmooth}, the above 
shows that \IAM{} benefits from interpolation, since \Cref{theo:smooth} gives a $\tilde{\mathcal{O}}(1/T)$ convergence in the case of interpolation ($\Delta_* =0$).  
In contrast to \SPS{}* 
the rate of convergence of \IAM{}  has an additional $\log(T+1)$  on the non-dominant $\mathcal{O}(\tfrac{1}{T+1})$ term.

\section{Experiments}

Here we present several numerical results. First, we test the extent of our convergence theory for \SPS* and \IAM{}. According to Remark \ref{rem:finitesum}, both \SPS* and \IAM{} will converge for differentiable convex finite-sum problems, even when the loss is non-smooth and non-Lipschitz. We test this on Poisson regression in~\Cref{sec:exp-convex-only}, where we show that \IAM{} converges to a loss value comparable to \texttt{L-BFGS}, and to \SGD{} with the best step size chosen from a grid.
In~\cref{sec:exp-misspecification} we investigate how \IAM{} behaves when $f_{\xi}(x_*)$ is wrongly specified (or guessed inaccurately).
Finally, in~\Cref{sec:distillation} we use \IAM{} and an \texttt{Adam} variant of \IAM{} for model distillation. 


\subsection{Black-box Model Distillation}
\label{sec:distillation}

Here we consider a variant of knowledge distillation where the goal is to train a small model (called \emph{student}) while having access to a pretrained, large model (called \emph{teacher}).

The main idea we propose here is that, when training the student, the loss of teacher model for a given batch $\xi$ can be used as an approximation of $f_\xi(x_*)$.

For a given batch $\xi\sim \mathcal{P}$ from the training set of the student, denote by $f_{\xi}^s(x)$ the loss function\footnote{This is usually the cross-entropy loss for the language modeling tasks we consider.} of the student with weights $x$ for batch $\xi$. Denote by $f_{\xi}^\tau$ the loss of the pretrained teacher model for the same batch.
Since the teacher is a significantly larger and more expressive model, we can assume that even after training the student, its loss will not fall below $f_{\xi}^\tau$. Thus, we use $f_{\xi}^s(x_*) \approx f_{\xi}^\tau$ for  the \IAM{} method (\cref{alg:IAM}) to train the student.

Many variations of knowledge distillation have been proposed \citep{Hinton2015,Beyer2022,hsieh-etal-2023-distilling}.
The variant we present here is slightly different to previous works in that it requires only access to the batch loss of the teacher model (and not to the logits). We discuss this relationship in more detail in \cref{asec:distildetails}.



We use three different datasets, \texttt{tinyShakespeare}, \texttt{PTB} and \texttt{Wikitext2}.
As teacher model we use a pretrained \GPT{} model with 774M parameters \citep{radford2019language,Wolf2020}. The student models are much smaller \GPT{} architectures. All details are deferred to  \Cref{asec:distildetails}.
Our results are in Figure~\ref{fig:distill}.
We compare \IAM{} and \texttt{IAM-Adam} (\IAM{} with an \texttt{Adam} preconditioner, see~\Cref{asec:iam-adam}) to \texttt{SGD} and \texttt{Adam} with (i) constant learning rate, and (ii) \textit{warmup+cosine-decay} schedule; tuning procedures are detailed in \cref{asec:distildetails}. We trained for only one pass over the data (epoch), and consequently we used no weight decay.

We find that both versions of \IAM{} achieve the best resulting loss on all three problems.
Consequently, when we are able to load a suitable pretrained teacher model, we find that \IAM{} is able to efficiently train a student model without any hyperparameter tuning.
\vspace{-0.3cm}
\section{Limitations}
The limitation of our methods is that they require the batch loss at an optimal point. Because of this, outside of applications that interpolate, or our model distillation setup, it could be hard to find direct applications for \SPS* and \IAM{}.


\bibliography{references}

\begin{thebibliography}{71}
\providecommand{\natexlab}[1]{#1}
\providecommand{\url}[1]{\texttt{#1}}
\expandafter\ifx\csname urlstyle\endcsname\relax
  \providecommand{\doi}[1]{doi: #1}\else
  \providecommand{\doi}{doi: \begingroup \urlstyle{rm}\Url}\fi

\bibitem[Armijo(1966)]{armijo1966minimization}
Armijo, L.
\newblock Minimization of functions having {Lipschitz} continuous first partial derivatives.
\newblock \emph{Pacific Journal of Mathematics}, 16:\penalty0 1--3, 1966.

\bibitem[Barr\'e et~al.(2020)Barr\'e, Taylor, and d'Aspremont]{pmlr-v125-barre20a}
Barr\'e, M., Taylor, A., and d'Aspremont, A.
\newblock Complexity guarantees for {Polyak} steps with momentum.
\newblock In \emph{Proceedings of Thirty Third Conference on Learning Theory}, volume 125 of \emph{Proceedings of Machine Learning Research}, pp.\  452--478. PMLR, 09--12 Jul 2020.

\bibitem[Bauschke \& Combettes(2017)Bauschke and Combettes]{BauCom}
Bauschke, H.~H. and Combettes, P.~L.
\newblock \emph{Convex {{Analysis}} and {{Monotone Operator Theory}} in {{Hilbert Spaces}}}.
\newblock {Springer}, 2nd edition edition, 2017.

\bibitem[Berrada et~al.(2020)Berrada, Zisserman, and Kumar]{ALI-G}
Berrada, L., Zisserman, A., and Kumar, M.~P.
\newblock Training neural networks for and by interpolation.
\newblock In \emph{Proceedings of the 37th International Conference on Machine Learning}, volume 119 of \emph{Proceedings of Machine Learning Research}, pp.\  799--809, 13--18 Jul 2020.

\bibitem[Beyer et~al.(2022)Beyer, Zhai, Royer, Markeeva, Anil, and Kolesnikov]{Beyer2022}
Beyer, L., Zhai, X., Royer, A., Markeeva, L., Anil, R., and Kolesnikov, A.
\newblock Knowledge distillation: A good teacher is patient and consistent.
\newblock In \emph{2022 IEEE/CVF Conference on Computer Vision and Pattern Recognition (CVPR)}, pp.\  10915--10924, 2022.

\bibitem[Bollapragada et~al.(2024)Bollapragada, Chen, and Ward]{Bollapragadamomentum2024}
Bollapragada, R., Chen, T., and Ward, R.
\newblock {On the fast convergence of minibatch heavy ball momentum}.
\newblock \emph{IMA Journal of Numerical Analysis}, 2024.

\bibitem[Carmon \& Hinder(2022)Carmon and Hinder]{carmon2022making}
Carmon, Y. and Hinder, O.
\newblock Making {SGD} parameter-free.
\newblock In \emph{Proceedings of Thirty Fifth Conference on Learning Theory}, volume 178 of \emph{Proceedings of Machine Learning Research}, pp.\  2360--2389. PMLR, 02--05 Jul 2022.

\bibitem[Combettes(2017)]{Combettes2017}
Combettes, P.~L.
\newblock Perspective functions: Properties, constructions, and examples.
\newblock \emph{Set-Valued and Variational Analysis}, 26\penalty0 (2):\penalty0 247--264, April 2017.

\bibitem[Cottier et~al.(2024)Cottier, Rahman, Fattorini, Maslej, and Owen]{cottier2024risingcoststrainingfrontier}
Cottier, B., Rahman, R., Fattorini, L., Maslej, N., and Owen, D.
\newblock The rising costs of training frontier {AI} models, 2024, arXiv:2405.21015.

\bibitem[Cutkosky(2019{\natexlab{a}})]{cutkosky19artifical}
Cutkosky, A.
\newblock Artificial constraints and hints for unbounded online learning.
\newblock In \emph{Proceedings of the Thirty-Second Conference on Learning Theory}, volume~99 of \emph{Proceedings of Machine Learning Research}, pp.\  874--894. PMLR, 25--28 Jun 2019{\natexlab{a}}.

\bibitem[Cutkosky(2019{\natexlab{b}})]{pmlr-v97-cutkosky19a}
Cutkosky, A.
\newblock Anytime online-to-batch, optimism and acceleration.
\newblock In \emph{Proceedings of the 36th International Conference on Machine Learning}, volume~97 of \emph{Proceedings of Machine Learning Research}, pp.\  1446--1454. PMLR, 09--15 Jun 2019{\natexlab{b}}.

\bibitem[Defazio \& Gower(2021)Defazio and Gower]{pmlr-v157-defazio21a}
Defazio, A. and Gower, R.~M.
\newblock The power of factorial powers: New parameter settings for (stochastic) optimization.
\newblock In \emph{Proceedings of The 13th Asian Conference on Machine Learning}, volume 157 of \emph{Proceedings of Machine Learning Research}, pp.\  49--64. PMLR, 17--19 Nov 2021.

\bibitem[Defazio \& Mishchenko(2023)Defazio and Mishchenko]{defazio23learning}
Defazio, A. and Mishchenko, K.
\newblock Learning-rate-free learning by {D}-adaptation.
\newblock In \emph{Proceedings of the 40th International Conference on Machine Learning}, volume 202 of \emph{Proceedings of Machine Learning Research}, pp.\  7449--7479. PMLR, 2023.

\bibitem[Drori \& Teboulle(2016)Drori and Teboulle]{drori2016optimal}
Drori, Y. and Teboulle, M.
\newblock An optimal variant of kelley’s cutting-plane method.
\newblock \emph{Mathematical Programming}, 160:\penalty0 321--351, 2016.

\bibitem[Duchi et~al.(2011)Duchi, Hazan, and Singer]{duchi2011adaptive}
Duchi, J., Hazan, E., and Singer, Y.
\newblock Adaptive subgradient methods for online learning and stochastic optimization.
\newblock \emph{Journal of Machine Learning Research}, 12\penalty0 (Jul):\penalty0 2121--2159, 2011.

\bibitem[Efron et~al.(2004)Efron, Hastie, Johnstone, and Tibshirani]{efron2004least}
Efron, B., Hastie, T., Johnstone, I., and Tibshirani, R.
\newblock Least angle regression.
\newblock \emph{The Annals of Statistics}, 32\penalty0 (2):\penalty0 407--499, 2004.

\bibitem[Fanaee-T \& Gama(2014)Fanaee-T and Gama]{fanaee2014bike}
Fanaee-T, H. and Gama, J.
\newblock Bike sharing dataset.
\newblock UCI Machine Learning Repository, 2014.

\bibitem[Garrigos \& Gower(2023)Garrigos and Gower]{garrigos2023handbook}
Garrigos, G. and Gower, R.~M.
\newblock Handbook of convergence theorems for (stochastic) gradient methods.
\newblock \emph{arXiv preprint arXiv:2301.11235}, 2023.

\bibitem[Garrigos et~al.(2023)Garrigos, Gower, and Schaipp]{garrigos2023function}
Garrigos, G., Gower, R.~M., and Schaipp, F.
\newblock Function value learning: Adaptive learning rates based on the {Polyak} stepsize and function splitting in {ERM}.
\newblock \emph{arXiv preprint arXiv:2307.14528}, 2023.

\bibitem[Ghadimi et~al.(2015)Ghadimi, Feyzmahdavian, and Johansson]{ghadimi2014HB}
Ghadimi, E., Feyzmahdavian, H.~R., and Johansson, M.
\newblock Global convergence of the heavy-ball method for convex optimization.
\newblock In \emph{2015 European Control Conference (ECC)}, pp.\  310--315, 2015.

\bibitem[Gower et~al.(2021)Gower, Sebbouh, and Loizou]{SGDstruct}
Gower, R., Sebbouh, O., and Loizou, N.
\newblock {SGD} for structured nonconvex functions: Learning rates, minibatching and interpolation.
\newblock In \emph{Proceedings of The 24th International Conference on Artificial Intelligence and Statistics}, volume 130 of \emph{Proceedings of Machine Learning Research}, pp.\  1315--1323. PMLR, 13--15 Apr 2021.

\bibitem[Gower et~al.(2019)Gower, Loizou, Qian, Sailanbayev, Shulgin, and Richtárik]{gower_sgd_2019}
Gower, R.~M., Loizou, N., Qian, X., Sailanbayev, A., Shulgin, E., and Richtárik, P.
\newblock {SGD}: {General} {Analysis} and {Improved} {Rates}.
\newblock In \emph{Proceedings of the 36th {International} {Conference} on {Machine} {Learning}}, volume~97, pp.\  5200--5209. PMLR, June 2019.

\bibitem[Gower et~al.(2020)Gower, Richt{\'{a}}rik, and Bach]{GowerRichBach2018}
Gower, R.~M., Richt{\'{a}}rik, P., and Bach, F.
\newblock Stochastic quasi-gradient methods: Variance reduction via {Jacobian} sketching.
\newblock \emph{Mathematical Programming}, 2020.

\bibitem[Gower et~al.(2022)Gower, Blondel, Gazagnadou, and Pedregosa]{slackpolyak}
Gower, R.~M., Blondel, M., Gazagnadou, N., and Pedregosa, F.
\newblock Cutting some slack for {SGD} with adaptive {Polyak} stepsizes, 2022, arXiv:2202.12328.

\bibitem[Hazan \& Kakade(2019)Hazan and Kakade]{hazan2022revisiting}
Hazan, E. and Kakade, S.
\newblock Revisiting the {Polyak} step size, 2019, arXiv:1905.00313.

\bibitem[Hinton et~al.(2015)Hinton, Vinyals, and Dean]{Hinton2015}
Hinton, G., Vinyals, O., and Dean, J.
\newblock Distilling the knowledge in a neural network.
\newblock \emph{arXiv preprint arXiv:1503.02531}, 2015.

\bibitem[Hsieh et~al.(2023)Hsieh, Li, Yeh, Nakhost, Fujii, Ratner, Krishna, Lee, and Pfister]{hsieh-etal-2023-distilling}
Hsieh, C.-Y., Li, C.-L., Yeh, C.-k., Nakhost, H., Fujii, Y., Ratner, A., Krishna, R., Lee, C.-Y., and Pfister, T.
\newblock Distilling step-by-step! outperforming larger language models with less training data and smaller model sizes.
\newblock In \emph{Findings of the Association for Computational Linguistics: ACL 2023}, pp.\  8003--8017. Association for Computational Linguistics, July 2023.

\bibitem[Ivgi et~al.(2023)Ivgi, Hinder, and Carmon]{ivgi23dog}
Ivgi, M., Hinder, O., and Carmon, Y.
\newblock {DoG} is {SGD}’s best friend: {A} parameter-free dynamic step size schedule.
\newblock In \emph{Proceedings of the 40th International Conference on Machine Learning}, volume 202 of \emph{Proceedings of Machine Learning Research}, pp.\  14465--14499. PMLR, 2023.

\bibitem[Karpathy(2015)]{tinyshakespeare}
Karpathy, A.
\newblock char-rnn.
\newblock \url{https://github.com/karpathy/char-rnn}, 2015.

\bibitem[Kavis et~al.(2019)Kavis, Levy, Bach, and Cevher]{kavis2019unixgrad}
Kavis, A., Levy, K.~Y., Bach, F., and Cevher, V.
\newblock {UniXGrad}: A universal, adaptive algorithm with optimal guarantees for constrained optimization.
\newblock \emph{Advances in neural information processing systems}, 32, 2019.

\bibitem[Khaled et~al.(2023)Khaled, Mishchenko, and Jin]{khaled2023dowg}
Khaled, A., Mishchenko, K., and Jin, C.
\newblock {DoWG} unleashed: An efficient universal parameter-free gradient descent method.
\newblock In \emph{Thirty-seventh Conference on Neural Information Processing Systems}, 2023.

\bibitem[Kingma \& Ba(2015)Kingma and Ba]{adam}
Kingma, D.~P. and Ba, J.
\newblock Adam: {A} method for stochastic optimization.
\newblock In \emph{3rd International Conference on Learning Representations, {ICLR} 2015, San Diego, CA, USA, May 7-9, 2015, Conference Track Proceedings}, 2015.

\bibitem[Lacoste-Julien et~al.(2012)Lacoste-Julien, Schmidt, and Bach]{lacostesimple1t}
Lacoste-Julien, S., Schmidt, M., and Bach, F.
\newblock A simpler approach to obtaining an {$O(1/t$)} convergence rate for the projected stochastic subgradient method, 2012, arXiv:1212.2002.

\bibitem[Latafat et~al.(2024)Latafat, Themelis, Stella, and Patrinos]{latafat2024adaptive}
Latafat, P., Themelis, A., Stella, L., and Patrinos, P.
\newblock Adaptive proximal algorithms for convex optimization under local {Lipschitz} continuity of the gradient.
\newblock \emph{Mathematical Programming}, 10 2024.

\bibitem[Lee et~al.(2024)Lee, Cheng, Paquette, and Paquette]{LeePaquettePaquette2024}
Lee, K., Cheng, A.~N., Paquette, C., and Paquette, E.
\newblock Trajectory of mini-batch momentum: batch size saturation and convergence in high dimensions.
\newblock In \emph{Proceedings of the 36th International Conference on Neural Information Processing Systems}, 2024.

\bibitem[Levy et~al.(2018)Levy, Yurtsever, and Cevher]{levy2018online}
Levy, K.~Y., Yurtsever, A., and Cevher, V.
\newblock Online adaptive methods, universality and acceleration.
\newblock \emph{Advances in Neural Information Processing Systems}, 31, 2018.

\bibitem[Li \& Lan(2023)Li and Lan]{li2023simple}
Li, T. and Lan, G.
\newblock A simple uniformly optimal method without line search for convex optimization.
\newblock \emph{arXiv preprint arXiv:2310.10082}, 2023.

\bibitem[Liu et~al.(2022)Liu, Zhu, and Belkin]{LIU202285}
Liu, C., Zhu, L., and Belkin, M.
\newblock Loss landscapes and optimization in over-parameterized non-linear systems and neural networks.
\newblock \emph{Applied and Computational Harmonic Analysis}, 59:\penalty0 85--116, 2022.

\bibitem[Liu \& Nocedal(1989)Liu and Nocedal]{liu1989limited}
Liu, D.~C. and Nocedal, J.
\newblock On the limited memory {BFGS} method for large scale optimization.
\newblock \emph{Mathematical Programming}, 45\penalty0 (1-3):\penalty0 503--528, 1989.

\bibitem[Loizou et~al.(2021)Loizou, Vaswani, Laradji, and Lacoste-Julien]{SPS}
Loizou, N., Vaswani, S., Laradji, I.~H., and Lacoste-Julien, S.
\newblock Stochastic {Polyak} step-size for {SGD}: An adaptive learning rate for fast convergence.
\newblock In \emph{AISTATS}, volume 130 of \emph{Proceedings of Machine Learning Research}, pp.\  1306--1314. PMLR, 2021.

\bibitem[Loshchilov \& Hutter(2017)Loshchilov and Hutter]{Loshchilov2017}
Loshchilov, I. and Hutter, F.
\newblock {SGDR:} stochastic gradient descent with warm restarts.
\newblock In \emph{5th International Conference on Learning Representations, {ICLR} 2017, Toulon, France, April 24-26, 2017, Conference Track Proceedings}. OpenReview.net, 2017.

\bibitem[Ma et~al.(2018)Ma, Bassily, and Belkin]{pmlr-v80-ma18a}
Ma, S., Bassily, R., and Belkin, M.
\newblock The power of interpolation: Understanding the effectiveness of {SGD} in modern over-parametrized learning.
\newblock In \emph{Proceedings of the 35th International Conference on Machine Learning}, volume~80 of \emph{Proceedings of Machine Learning Research}, pp.\  3325--3334. PMLR, 10--15 Jul 2018.

\bibitem[Malitsky \& Mishchenko(2020)Malitsky and Mishchenko]{malitsky20adaptive}
Malitsky, Y. and Mishchenko, K.
\newblock Adaptive gradient descent without descent.
\newblock In \emph{Proceedings of the 37th International Conference on Machine Learning}, volume 119 of \emph{Proceedings of Machine Learning Research}, pp.\  6702--6712. PMLR, 13--18 Jul 2020.

\bibitem[Marcus et~al.(1993)Marcus, Santorini, and Marcinkiewicz]{PTB}
Marcus, M.~P., Santorini, B., and Marcinkiewicz, M.~A.
\newblock Building a large annotated corpus of {E}nglish: The {P}enn {T}reebank.
\newblock \emph{Computational Linguistics}, 19\penalty0 (2):\penalty0 313--330, 1993.

\bibitem[Merity et~al.(2016)Merity, Xiong, Bradbury, and Socher]{wikitext-2}
Merity, S., Xiong, C., Bradbury, J., and Socher, R.
\newblock Pointer sentinel mixture models, 2016, arXiv:1609.07843.

\bibitem[Mhammedi \& Koolen(2020)Mhammedi and Koolen]{mhammedi2020lipschitz}
Mhammedi, Z. and Koolen, W.~M.
\newblock Lipschitz and comparator-norm adaptivity in online learning.
\newblock In \emph{Conference on Learning Theory}, pp.\  2858--2887. PMLR, 2020.

\bibitem[Mishchenko \& Defazio(2024)Mishchenko and Defazio]{mishchenko2024prodigy}
Mishchenko, K. and Defazio, A.
\newblock Prodigy: An expeditiously adaptive parameter-free learner.
\newblock In \emph{Forty-first International Conference on Machine Learning}, 2024.

\bibitem[Nesterov(2014)]{nesterov14universal}
Nesterov, Y.
\newblock Universal gradient methods for convex optimization problems.
\newblock \emph{Mathematical Programming}, 152\penalty0 (1-2):\penalty0 381--404, 2014.

\bibitem[Oikonomou \& Loizou(2024)Oikonomou and Loizou]{oikonomou2024stochastic}
Oikonomou, D. and Loizou, N.
\newblock Stochastic {Polyak} step-sizes and momentum: Convergence guarantees and practical performance.
\newblock \emph{arXiv preprint arXiv:2406.04142}, 2024.

\bibitem[Orabona(2019)]{orabonabook}
Orabona, F.
\newblock A modern introduction to online learning, 2019, arXiv:1912.13213.

\bibitem[Orabona \& P\'{a}l(2016)Orabona and P\'{a}l]{orabona16coin}
Orabona, F. and P\'{a}l, D.
\newblock Coin betting and parameter-free online learning.
\newblock In \emph{Proceedings of the 30th International Conference on Neural Information Processing Systems}, pp.\  577–585, 2016.

\bibitem[Orvieto \& Xiao(2024)Orvieto and Xiao]{orvieto2024NGN}
Orvieto, A. and Xiao, L.
\newblock An adaptive stochastic gradient method with non-negative {Gauss-Newton} stepsizes, 2024, arXiv:2407.04358.

\bibitem[Orvieto et~al.(2022)Orvieto, Lacoste-Julien, and Loizou]{Orvieto2022}
Orvieto, A., Lacoste-Julien, S., and Loizou, N.
\newblock Dynamics of {SGD} with stochastic {P}olyak stepsizes: Truly adaptive variants and convergence to exact solution.
\newblock In \emph{Advances in Neural Information Processing Systems}, volume~35, pp.\  26943--26954. Curran Associates, Inc., 2022.

\bibitem[Pedregosa \& Schaipp(2023)Pedregosa and Schaipp]{pedregosa2023sps2}
Pedregosa, F. and Schaipp, F.
\newblock Stochastic {Polyak} step-size, faster rates under strong convexity.
\newblock \url{http://fa.bianp.net/blog/2023/sps2/}, 2023.

\bibitem[Peypouquet(2015)]{Pey}
Peypouquet, J.
\newblock \emph{Convex {{Optimization}} in {{Normed Spaces}}}.
\newblock {{SpringerBriefs}} in {{Optimization}}. {Springer International Publishing}, {Cham}, 2015.

\bibitem[Polyak(1964)]{Polyak1964}
Polyak, B.~T.
\newblock \emph{USSR Computational Mathematics and Mathematical Physics}, 4\penalty0 (5):\penalty0 1--17, 1964.

\bibitem[Polyak(1987)]{Polyak87}
Polyak, B.~T.
\newblock \emph{Introduction to Optimization}.
\newblock Optimization Software, New York, 1987.

\bibitem[Radford et~al.(2019)Radford, Wu, Child, Luan, Amodei, and Sutskever]{radford2019language}
Radford, A., Wu, J., Child, R., Luan, D., Amodei, D., and Sutskever, I.
\newblock Language models are unsupervised multitask learners.
\newblock \emph{OpenAI}, 2019.

\bibitem[Richtárik et~al.(2024)Richtárik, Giancola, Lubczyk, and Yadav]{richtarik2024localcurvaturedescentsqueezing}
Richtárik, P., Giancola, S.~M., Lubczyk, D., and Yadav, R.
\newblock Local curvature descent: Squeezing more curvature out of standard and {Polyak} gradient descent, 2024, arXiv:2405.16574.

\bibitem[Rodomanov et~al.(2024)Rodomanov, Kavis, Wu, Antonakopoulos, and Cevher]{rodomanov2024universal}
Rodomanov, A., Kavis, A., Wu, Y., Antonakopoulos, K., and Cevher, V.
\newblock Universal gradient methods for stochastic convex optimization.
\newblock \emph{arXiv preprint arXiv:2402.03210}, 2024.

\bibitem[Romero et~al.(2015)Romero, Ballas, Kahou, Chassang, Gatta, and Bengio]{Romero2014}
Romero, A., Ballas, N., Kahou, S.~E., Chassang, A., Gatta, C., and Bengio, Y.
\newblock Fitnets: Hints for thin deep nets.
\newblock In \emph{3rd International Conference on Learning Representations}, 2015.

\bibitem[Schaipp et~al.(2023)Schaipp, Gower, and Ulbrich]{schaipp2023a}
Schaipp, F., Gower, R.~M., and Ulbrich, M.
\newblock A stochastic proximal {Polyak} step size.
\newblock \emph{Transactions on Machine Learning Research}, 2023.

\bibitem[Schaipp et~al.(2024)Schaipp, Ohana, Eickenberg, Defazio, and Gower]{Schaipp2024a}
Schaipp, F., Ohana, R., Eickenberg, M., Defazio, A., and Gower, R.~M.
\newblock {M}o{M}o: Momentum models for adaptive learning rates.
\newblock In \emph{Proceedings of the 41st International Conference on Machine Learning}, volume 235 of \emph{Proceedings of Machine Learning Research}, pp.\  43542--43570. PMLR, 21--27 Jul 2024.

\bibitem[Schmidt et~al.(2021)Schmidt, Schneider, and Hennig]{pmlr-v139-schmidt21a}
Schmidt, R.~M., Schneider, F., and Hennig, P.
\newblock Descending through a crowded valley - benchmarking deep learning optimizers.
\newblock In \emph{Proceedings of the 38th International Conference on Machine Learning}, volume 139 of \emph{Proceedings of Machine Learning Research}, pp.\  9367--9376. PMLR, 18--24 Jul 2021.

\bibitem[Sebbouh et~al.(2021)Sebbouh, Gower, and Defazio]{pmlr-v134-sebbouh21a}
Sebbouh, O., Gower, R.~M., and Defazio, A.
\newblock Almost sure convergence rates for stochastic gradient descent and stochastic heavy ball.
\newblock In \emph{Proceedings of Thirty Fourth Conference on Learning Theory}, volume 134 of \emph{Proceedings of Machine Learning Research}, pp.\  3935--3971. PMLR, 15--19 Aug 2021.

\bibitem[Streeter \& McMahan(2012)Streeter and McMahan]{streeter12noregret}
Streeter, M. and McMahan, H.~B.
\newblock No-regret algorithms for unconstrained online convex optimization.
\newblock In \emph{Proceedings of the 25th International Conference on Neural Information Processing Systems - Volume 2}, pp.\  2402–2410, 2012.

\bibitem[Takezawa et~al.(2024)Takezawa, Bao, Sato, Niwa, and Yamada]{takezawa2024}
Takezawa, Y., Bao, H., Sato, R., Niwa, K., and Yamada, M.
\newblock Polyak meets parameter-free clipped gradient descent, 2024, arXiv:2405.15010.

\bibitem[Vaswani et~al.(2019)Vaswani, Bach, and Schmidt]{fastersgd2019}
Vaswani, S., Bach, F., and Schmidt, M.
\newblock Fast and faster convergence of {SGD} for over-parameterized models and an accelerated perceptron.
\newblock In \emph{Proceedings of the Twenty-Second International Conference on Artificial Intelligence and Statistics}, volume~89 of \emph{Proceedings of Machine Learning Research}, pp.\  1195--1204. PMLR, 16--18 Apr 2019.

\bibitem[Wang(2021)]{mesh-transformer-jax}
Wang, B.
\newblock {Mesh-Transformer-JAX: Model-Parallel Implementation of Transformer Language Model with JAX}.
\newblock \url{https://github.com/kingoflolz/mesh-transformer-jax}, May 2021.

\bibitem[Wang et~al.(2023)Wang, Johansson, and Zhang]{pmlr-v202-wang23l}
Wang, X., Johansson, M., and Zhang, T.
\newblock Generalized {Polyak} step size for first order optimization with momentum.
\newblock In \emph{Proceedings of the 40th International Conference on Machine Learning}, volume 202 of \emph{Proceedings of Machine Learning Research}, pp.\  35836--35863. PMLR, 23--29 Jul 2023.

\bibitem[Wolf et~al.(2020)Wolf, Debut, Sanh, Chaumond, Delangue, Moi, Cistac, Rault, Louf, Funtowicz, Davison, Shleifer, von Platen, Ma, Jernite, Plu, Xu, Scao, Gugger, Drame, Lhoest, and Rush]{Wolf2020}
Wolf, T., Debut, L., Sanh, V., Chaumond, J., Delangue, C., Moi, A., Cistac, P., Rault, T., Louf, R., Funtowicz, M., Davison, J., Shleifer, S., von Platen, P., Ma, C., Jernite, Y., Plu, J., Xu, C., Scao, T.~L., Gugger, S., Drame, M., Lhoest, Q., and Rush, A.~M.
\newblock Transformers: State-of-the-art natural language processing.
\newblock In \emph{Proceedings of the 2020 Conference on Empirical Methods in Natural Language Processing: System Demonstrations}, pp.\  38--45. Association for Computational Linguistics, October 2020.

\end{thebibliography}
\bibliographystyle{icml/icml2025}

\paragraph{Acknowledgements}
Dimitris Oikonomou MINDS Fellowship. Nicolas Loizou acknowledges support from CISCO Research and Vector Institute for Artificial Intelligence.

\newpage
\appendix
\onecolumn

\tableofcontents

\newpage

\section{Bits of Convex Analysis and Details on our Hypotheses}
\label{sec:convex}

Here we introduce and define some of the more technical bits of convex analysis we need throughout the paper.
In particular we make precise the technical assumptions that we are making on the functions $f_\xi$, which correspond to the assumptions made in the Section 9 of \citet{garrigos2023handbook}.

\subsection{Subgradients}

Throughout our paper, we consider for every sampled data $\xi$ a loss function  $f_{\xi} : \mathbb{R}^d \to \mathbb{R}$ taking finite values.
We also always assume that $f_\xi$ is convex, which implies that it is continuous on $\mathbb{R}^d$ (see Proposition  3.5 in \cite{Pey}).
Nevertheless, we do not always assume that our loss functions $f_\xi$ are differentiable.
For example, $f_{\xi}(x)$ could be defined with an absolute value, such as $f_{\xi}(x) = | w_{\xi}^\top x -y_{\xi}|$ where $ w_{\xi}$ is a sample feature vector and $y_{\xi}$ a target value. 
In general, instead of using gradients we will making use of  \emph{subgradients}, which play a similar role.

\begin{definition}\label{D:subdifferential convex}
Let $f : \mathbb{R}^d \to \mathbb{R}$, and $x \in \mathbb{R}^d$.
We say that $g \in  \mathbb{R}^d$ is a \textbf{subgradient} of $f$ at $x\in\R^d$ if 
\begin{equation*}\label{eq:defsubgrad}
    \mbox{for every }y \in \mathbb{R}^d, \quad
    f(y) - f(x) - \langle g, y-x \rangle \geq 0.
\end{equation*}
\end{definition}

Since our loss functions $f_\xi$ are convex and continuous, we are guaranteed that at every $x \in \mathbb{R}^d$, there exists some subgradient that we will note $g_\xi(x)$
(the existence of such subgradient is stated in {[Prop.\ 3.25]\cite{Pey}} and [Cor.\ 8.40]\cite{BauCom}).
In our proofs we will often need to take the expectation of these subgradients $g_{\xi}(x)$ with respect to $\xi$.
To be able to do this, we must formally assume throughout that the function $\xi \mapsto g_\xi(x)$ is measurable for every $x \in \R^d$. 
This will for instance allow us to say that the expectation of $g_{\xi}(x)$ is a subgradient of $f$ at $x$ (see Lemma 9.5 in \cite{garrigos2023handbook}).
 

\subsection{Local Lipschitzness}

\begin{definition}
    We say that $f : \mathbb{R}^d \to \mathbb{R}^p$ is locally Lipschitz continuous if, for every $x \in \mathbb{R}^d$, there exists a neighbourhood $\mathbb{B}_\delta(x)$ and $G_x \geq 0$ such that $f$ is $G_x$-Lipschitz continuous on $\mathbb{B}_\delta(x)$.
\end{definition}

As a matter of fact, convex functions are locally Lipschitz continuous.

\begin{proposition}\label{P:convex functions are locally lipschitz}
    If $f : \mathbb{R}^d \to \mathbb{R}$ is convex, then it is locally Lipschitz continuous.
\end{proposition}

\begin{proof}
    See Corollary~8.41 in \cite{BauCom}.
\end{proof}

By compactness, this definition  is equivalent to ask for Lipschitzness over any bounded set.

\begin{lemma}\label{L:local lip equiv liip on bounded}
    A function $f : \mathbb{R}^d \to \mathbb{R}^p$ is locally Lipschitz continuous if and only if for every bounded set $\mathcal{B} \subset \mathbb{R}^d$ there exists $G_\mathcal{B} \geq 0$ such that $f$ is $G_\mathcal{B}$-Lipschitz continuous on $\mathcal{B}$.
\end{lemma}

\begin{proof}
    One implication is trivial. The other implication is standard. Assuming that $f$ is locally Lipschitz, and given any bounded set $\mathcal{B} \subset \mathbb{R}^d$, we can prove that $f$ is Lipschitz on $\mathcal{B}$.
    To see this, consider $\mathcal{K}$ the closure of $\mathcal{B}$ which is compact. 
    From the local Lipschitz assumption we are given for every $x \in \mathcal{K}$ a certain $\delta_x$ and $G_x$ such that $f$ is $G_x$-Lipschitz over $\mathbb{B}_{\delta_x}(x)$.
    It is clear that 
    \begin{equation*}
        \mathcal{B} \subset \mathcal{K} \subset \bigcup\limits_{x \in \mathcal{B}} \text{int } \mathbb{B}_{\delta_x}(x).
    \end{equation*}
    From the compactness of $\mathcal{K}$, there exists a finite number of points $x_1, \dots, x_n$ 
    such that 
    \begin{equation*}
        \mathcal{B} \subset \mathcal{K} \subset \bigcup\limits_{i=1}^n \text{int } \mathbb{B}_{\delta_{x_i}}(x_i).
    \end{equation*}
    The conclusion follows by taking $G_\mathcal{B} = \max\limits_{i=1, \dots, n} G_{x_i}$.
\end{proof}

Lipschitzness of a convex function is tightly connected to the boundedness of its subgradients, as we see next.

\begin{definition}
    Let $f: \mathbb{R}^d \to \mathbb{R}$ be convex, and $G \geq 0$.
    We say that $f$ has $G$-bounded subgradients over $\mathcal{B} \subset \mathbb{R}^d$ if, for every $x \in \mathcal{B}$, for every $g \in \partial f(x)$, we have $\Vert g \Vert \leq G$.
\end{definition}

\begin{lemma}\label{L:loclip equiv bounded gradient}
    Let $f : \mathbb{R}^d \to \mathbb{R}$ be convex. Let $\mathcal{U} \subset \mathbb{R}^d$ be any open set. Then the following is equivalent:
    \begin{enumerate}
        \item $f$ is $G$-Lipschitz over $\mathcal{U}$;
        \item $f$ has $G$-bounded subgradients over $\mathcal{U}$.
    \end{enumerate}
\end{lemma}

\begin{proof}
    This is a standard result, we provide a proof for completeness, which is taken from Proposition 16.20 in \cite{BauCom}.
    If $f$ is $G$-Lipschitz on $\mathcal{U}$, then for every $x \in \mathcal{U}$ and every $g \in \partial f(x)$ we can define $y = x + \gamma g$ for $\gamma >0$ small enough so that $y \in \mathcal{U}$ (we exploit here the fact that $\mathcal{U}$ is open).
    Therefore we can write
    \begin{equation*}
        \Vert g \Vert^2
        =
        \langle g, \tfrac{y-x}{\gamma} \rangle = \tfrac{1}{\gamma} \langle g, y-x \rangle \leq \tfrac{1}{\gamma} G \Vert y- x \Vert = G \Vert g \Vert,
    \end{equation*}
    and conclude that $f$ has $G$-bounded subgradients on $\mathcal{U}$.
    If we assume that $G$-bounded subgradients on $\mathcal{U}$, then for every $x,y \in \mathcal{U}$ we can take $g \in \partial f(x)$ and write
    \begin{equation*}
        \vert f(y) - f(x) \vert = f(y) - f(x) 
        \leq 
        \langle g, y-x \rangle 
        \leq \Vert g \Vert \Vert y-x \Vert \leq G \Vert y- x \Vert.
    \end{equation*}
    Note that we assumed $\vert f(y) - f(x) \vert = f(y) - f(x)$ here, which is always possible by eventually swapping $x$ and $y$.
    This proves the claim.
\end{proof}

\begin{proposition}
    Let $f : \mathbb{R}^d \to \mathbb{R}$ be convex and locally Lipschitz continuous, and let $\mathcal{B} \subset \mathbb{R}^d$ be any bounded set.
    Then there exists $G_\mathcal{B} \geq 0$ such that
    \begin{equation*}
        \text{for every $x \in \mathcal{B}$, for every $g \in \partial f(x)$, } \Vert g \Vert \leq G_\mathcal{B}.
    \end{equation*}
\end{proposition}

\begin{proof}
    Given $\mathcal{B} \subset \mathbb{R}^d$ bounded, consider any $\mathcal{B}'$ which is open and contains $\mathcal{B}$, for instance $\mathcal{B}' = \text{int} \left( \mathcal{B} + \mathbb{B}(0,1) \right)$.
    Then it suffices to apply Lemma \ref{L:local lip equiv liip on bounded} to obtain that $f$ is $G_{\mathcal{B}'}$-Lipschitz over $\mathcal{B}'$.
    We conclude with Lemma \ref{L:loclip equiv bounded gradient} and the fact that $\mathcal{B} \subset \mathcal{B}'$.
\end{proof}

\subsection{Local Lipschitzness in expectation}

\begin{definition}\label{D:local lipschitz expectation}
    We say that the family $(f_\xi)$ is locally Lipschitz in expectation if for every bounded set $\mathcal{B} \subset \mathbb{R}^d$ there exists constants $G_\mathcal{B}(\xi) \geq 0$ such that $f_\xi$ is $G_\mathcal{B}(\xi)$-Lipschitz over $\mathcal{B}$, and moreover $\mathbb{E}_\xi \left[ G_\mathcal{B}(\xi)^2 \right]< + \infty$.
\end{definition}

We recall that convex finite functions are always Lipschitz on bounded sets.
So if the $f_\xi$ are convex, we already know that the constants $G_\mathcal{B}(\xi)$ exist, and this definition only require the expectation $\mathbb{E}_\xi \left[ G_\mathcal{B}(\xi)^2 \right]$ to be finite.
This is always true for a finite family.

\begin{proposition}\label{P:finite family always loclip expectation}
    Let $f_1, \dots, f_m$ be a finite family of convex functions from $\mathbb{R}^d$ to $\mathbb{R}$.
    Then this family is locally Lipschitz in expectation.
\end{proposition}

\begin{proof}
    Each $f_i$ is locally Lipschitz continuous according to Proposition \ref{P:convex functions are locally lipschitz}.
    Therefore they are Lipschitz over any bounded set, according to Lemma \ref{L:local lip equiv liip on bounded}.
    Whatever Lipschitz constants $G_i$ we take, the expectation $\mathbb{E}_i \left[ G_i^2 \right]$ will be bounded by $\max\limits_{i=1,\dots,m} G_i^2$ which is finite.
\end{proof}

We conclude this section with the technical result at the core of Corollary \ref{cor:spsnonsmooth}, and which involves the measurable selection function $g_\xi : x \in \mathbb{R}^d \mapsto g_\xi(x) \in \partial f_\xi (x)$.

\begin{proposition}\label{P:local lipschitz expectation implies expected gradient bounded}
    Suppose that the family of functions $(f_\xi)$ is locally Lipschitz in expectation, and convex.
    Let $\mathcal{B} \subset \mathbb{R}^d$ be bounded.
    Then there exists $G_\mathcal{B} \geq 0$ such that
    \begin{equation*}
        \text{for every $x \in \mathcal{B}$, } \ 
        \mathbb{E}_\xi \left[ \Vert g_\xi(x) \Vert^2 \right] \leq G_\mathcal{B}.
    \end{equation*}
\end{proposition}

\begin{proof}
    Given $\mathcal{B} \subset \mathbb{R}^d$ bounded, consider any $\mathcal{B}'$ which is open and contains $\mathcal{B}$, for instance $\mathcal{B}' = \text{int} \left( \mathcal{B} + \mathbb{B}(0,1) \right)$.
    From the definition \ref{D:local lipschitz expectation}, we know constants $G_\xi$ such that $f_\xi$ is $G_\xi$-Lipschitz on $\mathcal{B}'$, with $G_\mathcal{B} := \mathbb{E}_\xi\left[ G_\xi \right]< +\infty$.
    From Lemma \ref{L:loclip equiv bounded gradient} we know that $f_\xi$ has $G_\xi$-bounded subgradients over $\mathcal{B}'$, so $\Vert g_\xi(x) \Vert \leq  G_\xi$.
    Taking the square and then expectation leads to the desired bound.
\end{proof}

\subsection{Local Smoothness}

We now give some technical details about locally smooth functions, which is the assumption made in \Cref{cor:spsnonsmooth}.

\begin{definition}
We say that $f : \mathbb{R}^d \to \mathbb{R}$ is locally smooth if it is differentiable and if $\nabla f : \mathbb{R}^d \to \mathbb{R}^d$ is locally Lipschitz continuous.
\end{definition}

Note that this definition is equivalent to require $\nabla f$ to be Lipschitz continuous over any bounded subset of $\mathbb{R}^d$.
A simple example of locally smooth functions are $C^2$ functions: 

\begin{proposition}\label{P:C2 implies locally smooth}
    If $f : \mathbb{R}^d \to \mathbb{R}$ is of class $C^2$, then it is locally smooth.
\end{proposition}

\begin{proof}
    The function $f$ being $C^2$ entails that it is differentiable.
    Moreover, for every open convex bounded set $\mathcal{U} \subset \mathbb{R}^d$, the mean value inequality says that for every $x,y \in \mathcal{U}$,
    \begin{equation*}
        \Vert \nabla f(y) - \nabla f(x) \Vert \leq \sup\limits_{z \in \mathcal{U}} \Vert \nabla^2 f(z) \Vert \Vert y-x \vert.
    \end{equation*}
    Because $\nabla^2 f$ is supposed continuous, and because $\mathcal{U}$ is bounded, we know that $\sup\limits_{z \in \mathcal{U}} \Vert \nabla^2 f(z) \Vert < + \infty$, which proves that $\nabla f$ is locally Lipschitz continuous.
\end{proof}

\begin{lemma}[Local descent lemma]\label{L:local descent lemma}
If $f$ is locally smooth, then for every bounded set $\mathcal{B} \subset \mathbb{R}^d$ there exists $L_\mathcal{B} \geq 0$ such that
\begin{equation}\label{eq:local descent lemma}
\text{for all } x,y \in \mathcal{B}, \quad
f(y) - f(x) - \langle \nabla f(x) , y-x \rangle \leq \frac{L_\mathcal{B}}{2} \Vert y - x \Vert^2.
\end{equation}
\end{lemma}

\begin{proof}
This is just a local version of the classic proof of the descent lemma, see e.g. Lemma 1.30 from \cite{Pey}.
Without loss of generality, we can assume that $\mathcal{B}$ is convex and compact (simply replace $\mathcal{B}$ with its closed convex hull).
By compactness, we know that $\nabla f$ is Lipschitz continous on $\mathcal{B}$, for some constant $L_\mathcal{B} \geq 0$.
We can then start the proof and fix $x,y \in \mathcal{B}$.
Define the auxiliary function $g(t) = f((1-t)x+ty) - t \langle \nabla f(x), y-x \rangle$ for $t \in [0,1]$.
It is differentiable and verifies
\begin{equation*}
g(1) - g(0) = \int_0^1 g'(t) \ dt
\end{equation*}
which is equivalent, by definition of $g$, to
\begin{equation*}
f(y) - f(x) - \langle \nabla f(x), y-x \rangle = \int_0^1 \langle \nabla f((1-t)x+ty) - \nabla f(x), y-x \rangle \ dt.
\end{equation*}
Now we use the Cauchy-Schwarz inequality, together with the Lipschitzness of $\nabla f$ (note that $z:=(1-t)x+ty)$ belongs to $B$ which is convex!), to obtain
\begin{eqnarray*}
&& f(y) - f(x) - \langle \nabla f(x), y-x \rangle \\
& \leq &
\int_0^1 \Vert \nabla f((1-t)x+ty) - \nabla f(x) \Vert \Vert y-x \Vert \ dt \\
& \leq &
\int_0^1 L_\mathcal{B}\Vert (1-t)x+ty) - x \Vert \Vert y-x \Vert \ dt \\
& = &
\int_0^1 L_\mathcal{B} t  \Vert y-x \Vert^2 \ dt \\
& = &
 \frac{L_\mathcal{B}}{2}  \Vert y-x \Vert^2.
\end{eqnarray*}
\end{proof}

Convex locally smooth functions verify locally the following cocoercivity inequality:

\begin{proposition}\label{P:local cocoercivity inequality}
If $f : \mathbb{R}^d \to \mathbb{R}$ is locally smooth and convex, then for every bounded set $\mathcal{B} \subset \mathbb{R}^d$ there exists $L_\mathcal{B} >0$ such that
\begin{equation*}
\text{ for all } y,x \in \mathcal{B}, \quad
\frac{1}{2L_\mathcal{B}}\Vert \nabla f(y) - \nabla f(x) \Vert^2 \leq f(y) - f(x) - \langle \nabla f(x), y-x   \rangle.
\end{equation*}
\end{proposition}

\begin{proof}
This proof is just an adaptation of a classical result (see e.g. Lemma 2.29 from \cite{garrigos2023handbook}) by making use of additional local arguments.
Here again, without loss of generality, we can assume that $\mathcal{B}$ is compact (see the proof of Lemma \ref{L:local descent lemma}).
Let $x,y \in \mathcal{B}$ be fixed, and let $L_\mathcal{B}$ be the local smoothness constant provided by the local descent lemma \ref{L:local descent lemma}.
Define $T : \mathbb{R}^d \times \mathbb{R}^d \times \mathbb{R} \to \mathbb{R}^d$ be the map defined by
\begin{equation*}
T(x,y,\gamma) = y - \gamma \nabla f(y) + \gamma \nabla f(x).
\end{equation*}
Because $\nabla f$ is supposed continuous, we know that $T$ is continuous. 
Now we define
\begin{equation*}
K := \{ T(x,y,\gamma) \ | \ x \in \mathcal{B}, y \in \mathcal{B},  \gamma \in [0, \tfrac{1}{L_\mathcal{B}}]\} \subset \mathbb{R}^d.
\end{equation*}
From our definitions it is clear that $K = T(\mathcal{B} \times \mathcal{B} \times [0, \tfrac{1}{L_B}])$.
In other words, it is the image of a compact set by a continuous function, which means that $K$ is compact.
It is also clear that $K$ contains $\mathcal{B}$, simply observe that $T(x,x,0) =x$.
Now we can use again the local descent lemma \ref{L:local descent lemma} to obtain that $f$ verifies \eqref{eq:local descent lemma} on $K$ with a constant $L_K$.
Without loss of generality, we can assume that $L_K \geq L_\mathcal{B}$ (simply replace $L_K$ with $\max\{ L_K, L_\mathcal{B}\}$).
Now we can proceed with the classic arguments of the proof.

Let $x,y \in \mathcal{B}$ be fixed, and define $z :=  y - \tfrac{1}{L_K}\nabla f(y) + \tfrac{1}{L_K} \nabla f(x)$.
By construction, $y \in \mathcal{B} \subset K$ and $z = T(x,y,\tfrac{1}{L_K}) \in K$.
So now we can use the convexity of $f$ together with the descent lemma inequality on $K$ to write
\begin{eqnarray*}
    && f(x) - f(y) \\
    & =& 
    f(x) - f(z) + f(z) - f(y) \\
    & \leq &
    - \langle \nabla f(x), z-x \rangle + \langle \nabla f(y), z-y \rangle + \frac{L_K}{2} \Vert z-y\Vert^2.
\end{eqnarray*}
Now we use the fact that $z-x = y-x + \tfrac{1}{L_k}(\nabla f(x) - \nabla f(y))$ and 
$z-y = \tfrac{1}{L_k}(\nabla f(x) - \nabla f(y))$:
\begin{eqnarray*}
    && f(x) - f(y) \\
    & \leq &
    - \langle \nabla f(x), y-x \rangle 
    - \frac{1}{L_K}\Vert \nabla f(x) - \nabla f(y) \Vert^2
    + \frac{1}{2L_K} \Vert \nabla f(x) - \nabla f(y)\Vert^2 \\
    &=&
    - \langle \nabla f(x), y-x \rangle 
    - \frac{1}{2L_K} \Vert \nabla f(x) - \nabla f(y)\Vert^2.
\end{eqnarray*}
The above inequality is equivalent to the claimed one, which ends the proof.
\end{proof}

We conclude this section with a weaker  result.

\begin{proposition}\label{P:local smooth nesterov inequality}
If $f : \mathbb{R}^d \to \mathbb{R}$ is locally smooth and bounded from below, then for every bounded set $\mathcal{B} \subset \mathbb{R}^d$ there exists $L_\mathcal{B} >0$ such that
\begin{equation*}
\text{ for all } x \in \mathcal{B}, \quad
\frac{1}{2L_\mathcal{B}}\Vert \nabla f(x) \Vert^2 \leq f(x) - \inf f.
\end{equation*}
\end{proposition}

\begin{proof}
Note that this result can be seen as an immediate consequence of \cref{P:local cocoercivity inequality} by taking $y=x$ and $x$ as some minimizer of $f$. 
But we actually do not need to assume that ${\rm{argmin}}~f \neq \emptyset$ for this result to be true.
We highlight that the following proof is just an adaptation of a classical result (see e.g. Lemma 2.28 from \cite{garrigos2023handbook}) by making use of additional local arguments.
Here again, without loss of generality, we can assume that $\mathcal{B}$ is compact (see the proof of Lemma \ref{L:local descent lemma}).
Let $L_\mathcal{B}$ be the local smoothness constant provided by the local descent lemma \ref{L:local descent lemma}.
Define $T : \mathbb{R}^d \times \mathbb{R} \to \mathbb{R}^d$ be the map defined by
\begin{equation*}
T(x,\gamma) = x - \gamma \nabla f(x).
\end{equation*}
Because $\nabla f$ is supposed continuous, we know that $T$ is continuous. Now we define
\begin{equation*}
K := \{ x - \gamma \nabla f(x) \ | \ x \in \mathcal{B}, \gamma \in [0, \tfrac{1}{L_B}]\} \subset \mathbb{R}^d.
\end{equation*}
From our definitions it is clear that $K = T(\mathcal{B} \times [0, \tfrac{1}{L_B}])$.
In other words, it is the image of a compact set by a continuous function, which means that $K$ is compact.
It is also clear that $K$ contains $\mathcal{B}$ (simply take $\gamma=0$).
Now we can use again the local descent lemma \ref{L:local descent lemma} to obtain that $f$ verifies \eqref{eq:local descent lemma} with a constant $L_K$.
Without loss of generality, we can assume that $L_K \geq L_\mathcal{B}$ (simply replace $L_K$ with $\max\{ L_K, L_\mathcal{B}\}$).
Now we can end the proof.
Let $x \in B$ be fixed, and define $y := x - \tfrac{1}{L_K}\nabla f(x)$.
By construction, $x \in \mathcal{B} \subset K$ and $y = T(x, \tfrac{1}{L_K}) \in K$.
So we can use the descent lemma inequality on $K$ to obtain
\begin{equation*}
f(x - \tfrac{1}{L_K}\nabla f(x)) - f(x) - \langle \nabla f(x) ,  - \tfrac{1}{L_K}\nabla f(x) \rangle \leq \frac{L_K}{2} \Vert  \tfrac{1}{L_K}\nabla f(x) \Vert^2.
\end{equation*}
Rewriting and reorganizing terms, we obtain further
\begin{equation*}
f(x - \tfrac{1}{L_K}\nabla f(x)) - f(x) \leq -\frac{1}{2L_K} \Vert \nabla f(x) \Vert^2.
\end{equation*}
We obtain the desired result by observing that $f(x - \tfrac{1}{L_K}\nabla f(x)) \geq \inf f$.
\end{proof}

\subsection{Uniform Local Smoothness}

\begin{definition}\label{D:locally smooth uniformly}
We say that the family $(f_\xi)$ is uniformly locally smooth if each function $f_\xi$ is differentiable, and for every bounded set $\mathcal{B} \subset \mathbb{R}^d$, there exists a constant $L_\mathcal{B} >0$ independent of $\xi$ such that each gradient $\nabla f_\xi$ is $L_\mathcal{B}$-Lipschitz continuous over $\mathcal{B}$.
\end{definition}

It is easy to see that any \emph{finite} family of locally smooth functions is uniformly locally smooth: simply take the maximum of the local smoothness constants.
In particular, any finite family of $C^2$ functions is uniformly locally smooth:

\begin{proposition}\label{P:C2 sum finite is uniformly locally smooth}
    Let $f_1, \dots, f_m$ be a finite family of convex $C^2$ functions from $\mathbb{R}^d$ to $\mathbb{R}$.
    Then this family is uniformly locally smooth.
\end{proposition}

\begin{proof}
    Each $f_i$ is $C^2$ therefore it is locally smooth (see Proposition \ref{P:C2 implies locally smooth}).
    So, for every bounded set $\mathcal{B}$, there exists constants $L_i$ such that $f_i$ is $L_i$-smooth over $\mathcal{B}$.
    The conclusion follows after taking $L_\mathcal{B} := \max\limits_{i=1,\dots, m} L_i$.
\end{proof}

\begin{proposition}\label{L:local smooth uniformly expected smoothness}
Suppose that the family of functions $(f_\xi)$ is uniformly locally smooth, and convex.
Let $f = \mathbb{E}\left[ f_\xi \right]$.
Then, for every bounded set $\mathcal{B} \subset \mathbb{R}^d$, there exists $L_\mathcal{B} >0$ such that
\begin{equation}\label{eq:local smooth uniformly expected smoothness}
\text{ for all } y,x \in \mathcal{B}, \quad
\frac{1}{2L_\mathcal{B}}\mathbb{E}\left[ \Vert \nabla f_\xi(y) - \nabla f_\xi(x) \Vert^2 \right] \leq f(y) - f(x) - \langle \nabla f(x), y-x \rangle.
\end{equation}
\end{proposition} 

\begin{proof}
By definition of uniformly locally smooth functions, there exists $L_\mathcal{B} \geq 0$ such that  each function $f_\xi$ is $L_\mathcal{B}$-smooth on $\mathcal{B}$, which means that we can use Proposition \ref{P:local cocoercivity inequality} to write
\begin{equation*}
\text{ for all } y,x \in \mathcal{B}, \quad
\frac{1}{2L_\mathcal{B}}\Vert \nabla f_\xi(x) - \nabla f_\xi(y) \Vert^2 \leq f_\xi(y) - f_\xi(x) - \langle \nabla f_\xi(x), y-x \rangle.
\end{equation*}
The conclusion follows after taking expectation with respect to $\xi$.
\end{proof}

\begin{proposition}\label{L:local smooth uniformly Nesterov}
Suppose that the family of functions $(f_\xi)$ is uniformly locally smooth and bounded from below.
Let $f = \mathbb{E}\left[ f_\xi \right]$.
Then, for every bounded set $\mathcal{B} \subset \mathbb{R}^d$, there exists $L_\mathcal{B} >0$ such that
\begin{equation}\label{eq:local smooth uniformly Nesterov}
\text{ for all } x \in \mathcal{B}, \quad
\frac{1}{2L_\mathcal{B}}\mathbb{E}\left[ \Vert \nabla f_\xi(x) \Vert^2 \right] \leq f(x) - \mathbb{E}\left[\inf f_\xi \right].
\end{equation}
\end{proposition} 

\begin{proof}
By definition of uniformly locally smooth functions, there exists $L_\mathcal{B} \geq 0$ such that  each function $f_\xi$ is $L_\mathcal{B}$-smooth on $\mathcal{B}$, which means that we can use Proposition \ref{P:local smooth nesterov inequality} to write
\begin{equation*}
\text{ for all } x \in \mathcal{B}, \quad
\frac{1}{2L_\mathcal{B}}\Vert \nabla f_\xi(x) \Vert^2 \leq f_\xi(x) - \inf f_\xi.
\end{equation*}
The conclusion follows after taking expectation with respect to $\xi$.
\end{proof}

\begin{proposition}\label{P:local variance transfer:gradient noise}
Suppose that the family of functions $(f_\xi)$ is uniformly locally smooth, and convex.
Let $f = \mathbb{E}\left[ f_\xi \right]$, and assume that $x_* \in {\rm{argmin}}~f \neq \emptyset$.
Let $\mathcal{B} \subset \mathbb{R}^d$ be a bounded set, and let $L_\mathcal{B} >0$ be the constant appearing in \eqref{eq:local smooth uniformly expected smoothness}.
Then
\begin{equation*}
    \text{for all $x \in \mathcal{B}$}, \quad
    \mathbb{E} \left[ \Vert \nabla f_\xi(x) \Vert^2 \right] \leq A(f(x) - \inf f) +B,
\end{equation*}
with $A = 4L_\mathcal{B}$ and $B = 2\sigma_*^2$, where $\sigma_*^2 = \mathbb{E} \left[ \Vert \nabla f_\xi(x_*) \Vert^2 \right]$.
\end{proposition}

\begin{proof}
    For any $x \in \mathcal{B}$, write
    \begin{eqnarray*}
        && \mathbb{E} \left[ \Vert \nabla f_\xi(x) \Vert^2 \right] \\
        &=&
        \mathbb{E} \left[ \Vert \nabla f_\xi(x) - \nabla f_\xi(x_*) + \nabla f_\xi(x_*) \Vert^2 \right] \\
        & \leq &
        2 \mathbb{E} \left[ \Vert \nabla f_\xi(x) - \nabla f_\xi(x_*)  \Vert^2 \right]+2\sigma_*^2 \\
        & \leq &
        4 L_\mathcal{B}\left( f(x) - \inf f \right)+2\sigma_*^2,
    \end{eqnarray*}
    where in the inequalities we first used $\Vert a+b \Vert^2 \leq 2 \Vert a \Vert^2 + \Vert b \Vert^2$ and then \eqref{eq:local smooth uniformly expected smoothness} from \cref{P:local cocoercivity inequality}.
\end{proof}

\begin{proposition}\label{P:local variance transfer:function noise}
Suppose that the family of functions $(f_\xi)$ is uniformly locally smooth, and bounded from below.
Let $f = \mathbb{E}\left[ f_\xi \right]$, and assume that $x_* \in {\rm{argmin}}~f \neq \emptyset$.
Let $\mathcal{B} \subset \mathbb{R}^d$ be a bounded set, and let $L_\mathcal{B} >0$ be the constant appearing in \eqref{eq:local smooth uniformly Nesterov}.
Then
\begin{equation*}
    \text{for all $x \in \mathcal{B}$}, \quad
    \mathbb{E} \left[ \Vert \nabla f_\xi(x) \Vert^2 \right] \leq A(f(x) - \inf f) +B,
\end{equation*}
with $A = 2L_\mathcal{B}$ and $B = 2L_\mathcal{B} \Delta_*$, where $\Delta_* = \inf f - \mathbb{E} \left[ \inf f_\xi  \right]$.
\end{proposition}

\begin{proof}
    For any $x \in \mathcal{B}$, use \eqref{eq:local smooth uniformly Nesterov} to write
    \begin{equation*}
        \mathbb{E} \left[ \Vert \nabla f_\xi(x) \Vert^2 \right] 
        \leq 
        2 L_\mathcal{B}\left( f(x) - \mathbb{E} \left[ \inf f_\xi  \right] \right)
        =
        2 L_\mathcal{B}\left( f(x) - \inf f + \Delta_*  \right).      
    \end{equation*}
\end{proof}

\section{Proofs for \SPS* : \SGD{} with Polyak stepsizes}
\subsection{Auxiliary Lemmas}

\begin{lemma}\label{L:psi rate reciprocal}
    Let $A,B \geq 0$ which are not simultaneously zero.
    Let
    $\psi(t) = \tfrac{t^2}{At+B}$ be defined for $t \geq 0$.
    Then $\psi$ is convex and increasing over $[0,+\infty)$, and its inverse is $\psi^{-1}(s) = \tfrac{1}{2}(sA + \sqrt{s^2 A^2 + 4sB})$.  
\end{lemma}

\begin{proof}
    The function $\psi$ is twice differentiable over $[0,+\infty)$, and we can compute
    \begin{equation*}
        \psi'(t) = \frac{At^2 + 2Bt}{(At+B)^2}
        \quad \text{ and } \quad 
        \psi''(t) = 
        \frac{(2At+2B)(At+B)^2 - 2(At^2+2Bt)(At+B)A}{(At+B)^4}
        =
        \frac{2B^2}{(At+B)^3}.
    \end{equation*}
    It is immediate to see that $\psi'$ and $\psi''$ are positive, from which we deduce that $\psi$ is convex and increasing.
    
    Next, consider two cases. If $At+B=0$, this implies $t=0$ and $\psi(0)=0$, thus $\psi^{-1}(0)=0$. If, however, $At+B \neq 0$, for $s \geq 0$ it holds
    \begin{equation*}
        \psi(t) = s 
        \Longleftrightarrow
        \frac{t^2}{At + B} = s
        \Longleftrightarrow
        t^2 - Ast - Bs = 0.
    \end{equation*}
    The last equation has a unique nonnegative solution which is $t=\tfrac{1}{2}(sA + \sqrt{s^2 A^2 + 4sB})$, from which we deduce the expression for $\psi^{-1}$.
\end{proof}

We will use the following lemma which is often used to study methods AdaGrad~type methods.
\begin{lemma}\label{lem:adagrad}
    Let $c_0,\dotsc, c_k\ge 0$ be some non-negative numbers with $c_0>0$, and denote $S_t = \sum_{i=0}^t c_i$, then
    \begin{equation}
        \sqrt{S_t}
        \le \sum_{k=0}^t \frac{c_k}{\sqrt{S_k}}.\label{eq:adagrad_lemma}
    \end{equation}
\end{lemma}
\begin{proof}
    The proof of the lemma can be found in various sources, for instance in the Appendix A of \citet{levy2018online}, but since it is very short, we will provide it here for completeness as well. Observe that for any $\alpha\in [0, 1]$, it holds $\alpha\ge 1 - \sqrt{1-\alpha}$. Substituting $\alpha = c_k / S_k\in[0, 1]$, we get
    \[
        \frac{c_k}{S_k}
        \ge  1 - \sqrt{1 - \frac{c_k}{S_k}} \Longrightarrow 
        \frac{c_k}{\sqrt{S_k}}
        \ge \sqrt{S_k} - \sqrt{S_k - c_k}
        = \sqrt{S_k} - \sqrt{S_{k-1}}. 
    \]
    Summing the last inequality from $k=1$ to $k=t$ and using $\sqrt{S_0} = \frac{c_0}{\sqrt{S_0}}$, we get the claim.
\end{proof}

We also rely on the following result.
\begin{lemma}[Extended Titu's Lemma] \label{lem:titu}
	For any random variable $X$ and positive-valued random variable $Y$, it holds
	\begin{equation}
		\E{\frac{(X)_+^2}{Y}} \ge \frac{\left(\E{X}\right)_+^2}{\E{Y}}. \label{eq:titu_expectation}
	\end{equation}
	In addition, for any numbers $a_0,\dotsc, a_k$ and positive numbers $b_0,\dotsc, b_k$, we have
	\begin{equation}
		\sum_{t=0}^k \frac{(a_t)_+^2}{b_t} \ge \frac{\bigl(\sum_{t=0}^k a_t\bigr)_+^2}{\sum_{t=0}^k b_t}. \label{eq:titu_numbers}
	\end{equation}
\end{lemma}

\begin{proof}
    The proof follows from applying Jensen's inequality to the function $\varphi(x,y) = (x)_+^2/y$. To prove that $\varphi$ is convex takes some work, and it is given in Lemma~A.4 in~\citet{garrigos2023handbook}. We also provide a different proof that $\varphi(x,y)$ is convex in the following~\Cref{lem:relux-squared-over-y-convex} by viewing $\varphi(x,y)$ as a perspective function.
The discrete result~\eqref{eq:titu_numbers} follows from applying \eqref{eq:titu_expectation} with uniform distribution over $\{a_0,\ldots,a_k\}$ and $\{b_0,\ldots,b_k\}$.
\end{proof}

\begin{lemma}\label{lem:relux-squared-over-y-convex}
    Consider the function $\varphi: \R \times \R \to \R,~ (x,y) \mapsto \varphi(x,y)$, where 
    \begin{align}\label{eqn:def-perspective-of-relu}
        \varphi(x,y) := \begin{cases}
            \frac{(x)_+^2}{y} \quad &\text{if}~ y > 0,\\
            0 \quad &\text{if}~ (y = 0) \wedge (x \leq 0),\\
            + \infty &\text{else}.
        \end{cases}
    \end{align}
    Then, $\varphi$ is closed, proper and convex on $\R\times \R$.
\end{lemma}
\begin{proof}
    Define the convex function $h(x):=(x)_+^2$. From \citet[Def.\ 2.1]{Combettes2017}, it follows that $\varphi(x,y)$ defined as in \eqref{eqn:def-perspective-of-relu} is the perspective function of $h$, that is, 
    for $y>0$ we have $\varphi(x,y)=y h(x/y)$; for $y=0$, we compute $\lim_{\alpha\to \infty} \frac{(\alpha x)_+^2}{\alpha} = 0$ if $x\leq 0$ and $+\infty$ otherwise. The perspective functions of closed, proper, convex functions is convex itself \citep[Prop.\ 2.3]{Combettes2017}.
\end{proof}

\if{ 
\guillaume{The lemma below is certainly useless now. not cited anywhere, and can be replaced by more local lemmas}

Here we show that the expected smoothness bound \eqref{eq:expsmooth-sps} is a consequence of assuming that $f_{\xi}$ is almost surely $L$--smooth.
\begin{lemma}\label{lem:smoothgradbnd}
Let $f_{\xi}$ be $L$--smooth for every $\xi$, that is let
\begin{equation}
    f_{\xi}(y) \; \leq \; f_{\xi}(x) + \dotprod{\nabla f_{\xi}(x), y-x} + \frac{L}{2}\norm{y-x}^2.
\end{equation}
As a consequence we have that
\begin{equation}\label{eq:expsmooth}
    \E{\norm{\nabla f_{\xi}(x)}^2} \leq  2 L \big( f(x)-\inf f    + \sigma_*^2\big),
\end{equation} 
where
\begin{eqnarray*} 
    \sigma_*^2 := f(x_*) - \E{ \inf f_{\xi}} \geq 0.
\end{eqnarray*}
\end{lemma}
The proof can be found in~\citet[Lem.\ 4.19]{garrigos2023handbook}. 
}\fi

\subsection{Sketch proof of~\cref{thm:better-bounds-sps} about rates of \SPS* under abstract assumptions}
\label{sec:sketchproof}

Here we give a sketch of the proof of~\cref{thm:better-bounds-sps} so that we can better highlight the main ideas behind the proof, and the main novelty. 

\theospsgen*
\noindent \emph{Proof Sketch.}
Plugging the \SPS* step size~\eqref{eq:SPS} into~\eqref{eq:iterdistance} and re-arranging gives
\begin{eqnarray*}
\frac{(f_{t}(x_t) - f_{t}(x_*))_+^2}{\norm{g_t}^2} 
& \leq & \norm{x_t -x_*}^2   -\norm{x_{t+1} -x_*}^2.
\end{eqnarray*}
Taking expectation conditioned on $x_t$, and 
 using that the map $(z_1,z_2) \mapsto (z_1)_+^2/z_z$ is jointly convex on $\R\times \R_{\geq 0}$ (cf.\ \cref{lem:relux-squared-over-y-convex}) together with Jensen's inequality, we get 
\begin{eqnarray*}
\frac{(f(x_t) - f(x_*))_+^2}{\EE{t}{\norm{g_t}^2}} 
& \leq & \norm{x_t -x_*}^2   - \EE{t}{\norm{x_{t+1} -x_*}^2}.
\end{eqnarray*}

We can then use our main assumption~\eqref{spscv2} to bound the denominator of the left hand side giving
\begin{eqnarray*}
\frac{(f(x_t) - f(x_*))^2}{A(f(x_t) - f(x_*)) +B} 
& \leq & \norm{x_t -x_*}^2   -\EE{t}{\norm{x_{t+1} -x_*}^2}.
\end{eqnarray*}
Taking expectation again, and averaging both sides over $t=0,\ldots, T-1$ and telescoping we have that
\[\frac{1}{T}\sum_{t=0}^{T-1}\E{\frac{(f(x_t) - f(x_*))^2}{A(f(x_t) - f(x_*)) +B} } \leq  \frac{\norm{x_0 -x_*}^2}{T} -\frac{\E{\norm{x_T -x_*}}^2}{T} \leq \frac{\|x_0-x_*\|^2}{T}.  \]
The final step of the proof, and the main technical novelty, follows by defining the function $\psi(r) = \frac{r^2}{A r+ B}$ for $r\geq0$, and noting that the left hand side of the above is equal to $\frac{1}{T}\sum_{t=0}^{T-1}\E{\psi(f(x_t) - f(x_*)) }$. We then apply \Cref{L:psi rate reciprocal} in the appendix that shows that $\psi$ is a convex monotone function. Being convex, we can bring the average over $t$ and the expectation inside $\psi$ giving
\[\psi(\E{f(\bar{x}_t) - f(x_*)}) \leq \frac{\|x_0-x_*\|^2}{T}.  \]
Finally,  \Cref{L:psi rate reciprocal} also proves that $\psi$ has an inverse given by
$$\psi^{-1}(s) = \tfrac{1}{2}(sA + \sqrt{s^2 A^2 + 4sB}).$$
Applying this inverse to both sides and using that $\psi^{-1}$ is monotone, gives the result. \hfill $\qed$

 Next we give the complete and detailed proof of~\cref{thm:better-bounds-sps}.
 
\subsection{Proof of~\cref{thm:better-bounds-sps} about rates of \SPS* under abstract assumptions} \label{sec:better-bounds-sps}

\theospsgen*

\begin{proof}
    For short-hand we use $f_t := f_{\xi_t}$ to be the stochastic function sampled at iteration $t$.
    Expanding the squares, using the definition of the algorithm and using the convexity of $f_{\xi}$, we have that
    \begin{eqnarray*}
        \norm{x_{t+1} -x_*}^2 - \norm{x_t -x_*}^2
        & = &
        2 \gamSPS_t \langle g_t , x_* - x_t \rangle 
        + (\gamSPS_t)^2 \norm{  g_t}^2 \\
        & \leq & 
        -2\gamSPS_t (f_{t}(x_t) -f_{t}(x_*)) + (\gamSPS_t)^2 \norm{  g_t}^2.
    \end{eqnarray*}
Let us now verify that 
\begin{equation}\label{eqn:monotonicity}
    \norm{x_{t+1} -x_*}^2 \; \leq \; \norm{x_t -x_*}^2 \; \text{with probability 1}.
\end{equation}
If  $g_t = 0 $, then by definition we have that $\gamSPS_t =0$, thus the right-hand side of the above is zero, and~\eqref{eqn:monotonicity} holds.
Suppose instead that $g_t \neq 0 $. Substituting in $\gamSPS_t$ gives
\begin{eqnarray*}
\norm{x_{t+1} -x_*}^2 - \norm{x_t -x_*}^2   
& \leq &
- 2\frac{(f_{t}(x_t) - f_{t}(x_*))_+}{\norm{g_t}^2} (f_{t}(x_t) -f(x_*)) + \frac{(f_{t}(x_t) - f_{t}(x_*))_+^2}{\norm{g_t}^2} \\
& = &
-\frac{(f_{t}(x_t) - f_{t}(x_*))_+^2}{\norm{g_t}^2},
\end{eqnarray*}
where in the last equality we use the identity $z(z)_+ = (z)_+^2$.
Note that in both cases we obtained a nonpositive right-hand side, from which we deduce that~\eqref{eqn:monotonicity} holds, that is, $(x_t)_{t\geq 0}$ is Fejér monotone.

Now, let $a_t := f_t(x_t) - f_t(x_*)$ and $b_t := \norm{g_t}^2$, and define the function
\begin{equation*}
\phi(a,b) = 
\begin{cases}
\frac{(a)_+^2}{b} & \text{ if } a\in \mathbb{R}, b > 0, \\
0 & \text{ if } a \leq 0, b=0,
\end{cases}
\end{equation*}
so that the previous inequality can be rewritten as
\begin{equation}\label{sps2}
	\phi(a_t,b_t)
	\leq 
	\norm{x_t -x_*}^2  - \norm{x_{t+1} -x_*}^2.
\end{equation} 
Note that $\phi(a_t,b_t)$ is well-defined even in the case that $g_t = 0$.
Indeed, the convexity of $f_t$ implies in this case that $x_t$ minimizes $f_t$, which means that $a_t \leq 0$ while $b_t=0$.
Our main trick is to use Jensen's inequality with regard to the function $\phi$ which is convex (see \cref{lem:relux-squared-over-y-convex} or the Appendix in \citet{garrigos2023handbook} for a proof):
\begin{equation} \label{spscv1}
\phi(\mathbb{E} \left[ a_t \right],\mathbb{E} \left[ b_t \right])
\leq
\mathbb{E} \left[ \phi(a_t,b_t)  \right]
\leq
\mathbb{E} \left[ \norm{x_t -x_*}^2\right]-    \mathbb{E}\left[ \norm{x_{t+1} - x_*}^2\right].
\end{equation}
We can compute $\EE{}{a_t} = \EE{}{f_{\xi_t}(x_t) - f_{\xi_t}(x_*)} = \EE{}{ f(x_t) - \inf f}$ and $\EE{}{b_t} = \EE{}{\norm{g_t}^2}$.

For the rest of the proof, we are going to use the fact that there exist two constants $A,B \geq 0$, which are not simultaneously zero, and such that~\eqref{spscv2} holds, that is
\begin{equation}\label{spscv2-2}
    \EE{}{\norm{g_{\xi}(x)}^2}
    \leq 
    A(f(x) - \inf f) + B,
     \text{ for every }
    x \in \mathbb{B}(x_*,D).
\end{equation}

We are now going to inject this inequality \eqref{spscv2-2} into \eqref{spscv1}.
If $\EE{}{\norm{g_t}^2} \neq 0$, using the fact that $f(x_t) - \inf f \geq 0$ we obtain
\begin{equation} \label{eq:temsprmaper} 
\frac{\E{f(x_t) - \inf f}^2}{A\E{f(x_t) - \inf f} + B}
\leq
\phi(\mathbb{E} \left[ a_t \right],\mathbb{E} \left[ b_t \right])
\leq
\mathbb{E} \left[ \norm{x_t -x_*}^2\right]-    \mathbb{E} \left[ \norm{x_{t+1} - x_*}^2\right].
\end{equation}
Recall that we defined $\psi(r) = \tfrac{r^2}{Ar + B}$ for any $r \geq 0$.
Let $r_t := \E{f(x_t) - \inf f}$.
With this notation, the inequality~\eqref{eq:temsprmaper} can be rewritten as
\begin{equation}\label{spscv3}
\psi(r_t)
\leq
\mathbb{E}\left[ \norm{x_t -x_*}^2\right]-    \mathbb{E} \left[ \norm{x_{t+1} - x_*}^2\right].
\end{equation}

We observe that \eqref{spscv3} remains true when  $\EE{}{\norm{g_t}^2} = 0$.
Indeed in this case, from the variance bound we have that
\[ 0= \EE{}{\norm{g_t}^2} \geq \norm{\EE{}{g_t}}^2.\]
Furthermore it follows that $\EE{t}{g_t}$ is a subgradient of the full loss $f(x_t)$ (see Lemma 9.5 in~\cite{garrigos2023handbook}). Consequently $x_t$  minimizes $f$, meaning in this case that we would have $r_t =0$, and so $\psi(r_t)=0 = \phi(0,0)$.

For the last part of this proof, we
sum over $t=0, \dots, T-1$ and divide by $T$ to obtain, after telescoping terms:
\begin{equation*}
    \frac{1}{T}\sum_{t=0}^{T-1} \mathbb{E}\left[  \psi(r_t)\right]
    \leq 
    \frac{1}{T}
    \mathbb{E}\left[ \Vert x_0 - x_* \Vert^2 \right]
    -    
    \frac{1}{T}
    \mathbb{E}\left[ \Vert x_T - x_* \Vert^2 \right]
    \leq 
    \frac{D^2}{T}.
\end{equation*}
We now lower-bound the left-hand side term by using Jensen's inequality twice
\begin{equation*}
    \frac{1}{T}\sum_{t=0}^{T-1} \mathbb{E}\left[  \psi(r_t)\right]
    \geq 
    \psi \left(\mathbb{E}\left[ \frac{1}{T}\sum_{t=0}^{T-1}  r_t \right] \right)
    =
    \psi \left( \mathbb{E}\left[ \frac{1}{T}\sum_{t=0}^{T-1}  (f(x_t) - \inf f) \right] \right)
    \geq 
    \psi \left( \mathbb{E}\left[   f(\bar x_T) - \inf f \right] \right),
\end{equation*}
where in the first inequality we use the convexity of $\psi$, and in the second we use the convexity of $f$ together with the fact that $\psi$ is increasing, and we note the average of the iterates $\bar x_T := \tfrac{1}{T} \sum_{t=0}^{T-1} x_t$.
The reader can look at Lemma \ref{L:psi rate reciprocal} for a proof that $\psi$ is convex and monotone.
Combining the two previous inequalities, we obtain
\begin{equation*}
    \psi \left( \mathbb{E}\left[   f(\bar x_T) - \inf f \right] \right)
    \leq
    \frac{D^2}{T}.
\end{equation*}
Since $\psi$ is increasing on $[0,+\infty)$, it has an inverse which is also increasing. Applying the inverse of  $\psi$ on both sides gives
\begin{equation*}
    \mathbb{E}\left[   f(\bar x_T) - \inf f \right]
    \leq 
    \psi^{-1} \left( \frac{D^2}{T} \right).
\end{equation*}
From~\Cref{L:psi rate reciprocal} we know that $\psi^{-1}(s) = \tfrac{1}{2}(sA + \sqrt{s^2 A^2 + 4sB})$, and using the sublinearity of the square root we further have 
\begin{equation} \label{eq:psisublinear}
    \psi^{-1}(s) \;\leq \;   \tfrac{1}{2}(sA +\sqrt{s^2 A^2} + \sqrt{ 4sB}) = sA + \sqrt{sB}. 
\end{equation} 
From this we finally obtain the desired bound.
\end{proof}

\subsection{Proof of~\Cref{cor:spsnonsmooth} about rates of \SPS* in the nonsmooth setting}\label{S:proof SPS nonsmooth}
\corspsnonsmooth*
\begin{proof}
Because the losses are locally Lipschitz in expectation, we know from Proposition \ref{P:local lipschitz expectation implies expected gradient bounded} that \eqref{eq:exp-lipschitz} holds true for some $G < + \infty$.
This means that \eqref{spscv2} holds with $A=0$ and $B=G^2$. 
Thus the result follows by plugging in these constants into the rates of \cref{thm:better-bounds-sps}.
\end{proof}

\subsection{Proof of~\Cref{cor:spssmooth} about rates of \SPS* in the smooth setting}\label{S:proof SPS smooth}
\corspssmooth*

\if{
\begin{theorem}[Smooth setting, variant with $\Delta_*$]
    Consider problem~\eqref{eq:prob}, and assume that the losses $f_\xi$ are uniformly locally smooth (see Definition~\ref{D:locally smooth uniformly}). In particular there exists $L \geq 0$ s.t.
 \begin{equation}\label{eq:expsmooth-sps3}
    \EE{\xi}{\norm{ \nabla f_{\xi}(x)}^2} \leq  2 L \big( f(x)- \mathbb{E}_\xi \left[ \inf f_\xi \right] \big), 
\end{equation} 
for all $x \in \mathbb{B}_D(x_*)$. Then the averaged iterates 
$\bar x_T := \tfrac{1}{T} \sum_{t=0}^{T-1} x_t$ of \SPS* 
verify, with $\Delta_* := \inf f - \mathbb{E}_\xi \left[ \inf f_\xi \right]$:
\begin{equation*}\label{eq:spssmooth2}
    \mathbb{E}\left[ f(\bar x_T) - \inf f \right]
    \leq
    \frac{2LD^2}{T} + \frac{\sqrt{2L \Delta_*} D}{\sqrt{T}}.
\end{equation*}
\end{theorem}
}\fi 

\begin{proof}
Assuming the losses to be uniformly locally smooth ensures that we can use
\cref{L:local smooth uniformly Nesterov}, and obtain that \eqref{eq:expsmooth-sps} is true.
Then we use \cref{P:local variance transfer:function noise}, which guarantees that \eqref{spscv2} holds with $A=2L$ and $B=2L\Delta_*$, and we conclude  by plugging in these constants into the rates of~\cref{thm:better-bounds-sps}.
\end{proof}

We provide below an alternative result which makes use of a different interpolation constant.
Instead of relying on $\Delta_* = \inf f - \mathbb{E}_\xi \left[ \inf f_\xi \right]$, we use $\sigma_*^2 = \mathbb{E}_\xi \left[ \Vert \nabla f_\xi(x_*)\Vert^2 \right]$.
There are a couple of connections between those constants.
First they are both interpolation constants, in the sense that they are nonnegative, and equal to zero if and only if interpolation holds (see Section 4.3 in \cite{garrigos2023handbook}).
Second, the former dominates the latter through the inequality $\sigma_*^2 \leq 2 L \Delta_*$ for some $L \geq 0$ (see Proposition \ref{L:local smooth uniformly Nesterov} with $x=x^*$).
In particular it follows from our assumption $\mathbb{E}_\xi \left[ \inf f_\xi \right] > - \infty$ that $\sigma_*^2$ is finite.
Actually one could argue that assuming $\sigma^2_* <+ \infty$ is an even weaker assumption than $\mathbb{E}_\xi \left[ \inf f_\xi \right] > - \infty$.

\begin{theorem}[Smooth setting, variant with $\sigma^2_*$]\label{T:CV SPS SGD smooth gradient noise}
    Consider problem~\eqref{eq:prob}, and assume that the losses $f_\xi$ are uniformly locally smooth (see Definition~\ref{D:locally smooth uniformly}). In particular there exists $L \geq 0$ s.t.
 \begin{equation}\label{eq:expsmooth-sps2}
    \EE{\xi}{\norm{ \nabla f_{\xi}(x) - \nabla f_{\xi}(x^*)}^2} \leq  2 L \big( f(x)-\inf f \big), 
\end{equation} 
for all $x \in \mathbb{B}_D(x_*)$. Then the averaged iterates 
$\bar x_T := \tfrac{1}{T} \sum_{t=0}^{T-1} x_t$ of \SPS* 
verify, with $\sigma_*^2 := \mathbb{E}_\xi \left[ \Vert \nabla f_\xi(x_*)\Vert^2 \right]$:
\begin{equation*}
    \mathbb{E}\left[ f(\bar x_T) - \inf f \right]
    \leq
    \frac{4LD^2}{T} + \frac{\sqrt{2} D\sigma_*}{\sqrt{T}}.%
\end{equation*}
\end{theorem}

\begin{proof}
Assuming the losses to be uniformly locally smooth ensures that we can use
\cref{L:local smooth uniformly expected smoothness} with $y=x$ and $x=x_*$, and obtain that \eqref{eq:expsmooth-sps2} is true.
Then we use \cref{P:local variance transfer:gradient noise}, which guarantees that \eqref{spscv2} holds with $A=4L$ and $B=2\sigma_*^2$, and we conclude  by plugging in these constants into the rates of~\cref{thm:better-bounds-sps}.
\end{proof}

A last comment: from the inequality $\sigma_*^2 \leq 2 L \Delta_*$, one could think that it is enough to prove results using $\sigma^2_*$ and then use this inequality to transform it into a result using $\Delta_*$.
But the situation is not so simple. 
Indeed, applying this inequality to the bound of Theorem \ref{T:CV SPS SGD smooth gradient noise} gives
\begin{equation*}
    \mathbb{E}\left[ f(\bar x_T) - \inf f \right]
    \leq
    \frac{4LD^2}{T} + \frac{\sqrt{4L\Delta_*} D}{\sqrt{T}}.%
\end{equation*}
We see that the multiplicative constants are less good (by a factor $2$) than the ones from \Cref{cor:spssmooth}.
We claim that this is due to the use  of two different smoothness inequalities (one for proving Theorem \ref{T:CV SPS SGD smooth gradient noise} and one for using the inequality $\sigma_*^2 \leq 2 L \Delta_*$), while the direct proof of \Cref{cor:spssmooth} uses smoothness only once.

\subsection{Additional Result: Convergence Rates of \SPS* with Local Strongly Convexity} \label{sec:stronglycvx}

If we assume our loss functions is (locally) strongly convex, then we can improve the rate of convergence of \SPS* from $\cO{1/\sqrt{t}}$ to $\cO{1/t}$.

\begin{restatable}[Convergence of \SPS*]{theorem}{theospsstrong}\label{thm:better-bounds-sps-strong}
Consider~\eqref{eq:prob} and 
let the iterates $(x_t)_{t \geq 0}$ be given by \eqref{eqn:sps-iter}, and let $D:= \Vert x_0 - x_* \Vert$. Assume that $f_\xi$ is convex for any $\xi$. Let $f(x)$ be convex and satisfy the $\mu$--quadratic growth bound
\begin{eqnarray}\label{eqn:str-convex}
   \frac{\mu}{2} \norm{x-x_*}^2 \leq f(x)- \inf f,\quad   \text{ for every }
    x \in \mathbb{B}(x_*,D)
\end{eqnarray}
and the expected smoothness bound 
\begin{equation}\label{spscv2-str}
    \EE{\xi}{\norm{g_{\xi}(x)}^2}
    \leq 
    A(f(x) - \inf f) + B, \quad
     \text{ for every }
    x \in \mathbb{B}(x_*,D).
\end{equation}
Let $T_0:=  \frac{4A}{\mu} \log\left(\frac{D^2\mu^2}{16B}\right)$.
It follows that 
\begin{equation}\label{eq:spsgen-str-cvx}
     \mathbb{E}  \norm{x_t-x_*}^2 
     \leq \frac{16B}{\mu^2} \frac{1}{t+1 -T_0},
\quad \forall t \geq \frac{2A}{\mu} \left(2\log\left(\frac{D^2\mu^2}{16B} \right) +1\right).
\end{equation}

\end{restatable}

In the non-smooth setting where $A =0$ and $B=G^2$ we get
\begin{equation}\label{eq:spsnonsmooth-strcvx}
    \E{\norm{x_t -x_*}^2} \leq  \frac{16G^2}{\mu^2} \frac{1}{t+1}, \quad \mbox{for }t \geq 0. 
\end{equation} 
This matches the rate given by~\citet{pedregosa2023sps2} for the finite sum setting upto a factor of $4.$

In the smooth setting where $A = 4L$ and $B = \sigma_*^2$ we get
\begin{equation} \label{eq:spssmooth-strcvx}
\E{\norm{x_t -x_*}^2} \leq  \frac{64 \sigma_*^2}{\mu^2} \frac{1}{t+1-T_0}, \quad \mbox{for } t \geq \frac{8L}{\mu} \left(2\log\left(\frac{D^2\mu^2}{16\sigma_*^2} \right) +1\right).  \end{equation} 

\begin{proof}
Let 
$\delta_t: = \mathbb{E} \left[ \norm{x_t -x_*}^2\right]. $
We start the proof from~\eqref{spscv3}, which we repeat here for convenience:
\begin{equation}\label{spscv3-2}
\psi(r_t)
\leq \delta_t- \delta_{t+1},
\end{equation}
where $r_t = \E{f(x_t)-f(x_*)}$ and $\psi(r):=\frac{r^2}{Ar+B}$ for $r\geq 0$. Due to the monotonicity of the iterates (recall~\eqref{eqn:monotonicity}) we have that $\delta_t -\delta_{t+1} \geq 0.$
Applying~\Cref{L:psi rate reciprocal} together with~\eqref{eq:psisublinear} gives
\begin{align*}
    \E{f(x_t)- f(x_*)} & \leq \psi^{-1}(\delta_t- \delta_{t+1})\\
    & \leq A(\delta_t- \delta_{t+1})+\sqrt{B(\delta_t- \delta_{t+1})}.
\end{align*}

Using the quadratic growth bound $\frac{\mu}{2} \norm{x_t-x_*}^2 \leq f(x_t)- f(x_*) $ gives
\begin{align}\label{eq:tesmtpoensr}
   \frac{\mu}{2} \delta_t
    & \leq A \left(\delta_t- \delta_{t+1}\right)+\sqrt{B\left(\delta_t- \delta_{t+1}\right)}.
\end{align}


Our proofs will consider two cases by comparing the two terms on the right hand side of~\eqref{eq:tesmtpoensr}. To this end, note that
\begin{align}\label{eq:zjl9djzedzedze}
     A \left(\delta_t- \delta_{t+1}\right) \leq \sqrt{B\left(\delta_t- \delta_{t+1}\right)} \quad \iff \quad 
     \delta_t- \delta_{t+1} \leq \frac{B}{A^2}.
\end{align}

The remainder of the proof is divided into two parts. First we show that for $t_0 :=  \lceil \frac{4A}{\mu} \log\left(\frac{D^2\mu^2}{16B}\right) \rceil$, we have that $\delta_t \leq \frac{16B}{\mu^2}$. For the second part we prove by induction that for $t \geq t_0 $ $\delta_{t+1} \leq \frac{16B}{\mu^2} \frac{1}{t+1}$, where the first part will serve as the base case of the induction.

\textit{Base case:} First we prove that for all $t \geq \frac{4A}{\mu} \log\left(\frac{D^2\mu^2}{16B}\right)$ we have that $\delta_t \leq \frac{16B}{\mu^2}.$ We divide this proof also into two cases based on the comparison~\eqref{eq:zjl9djzedzedze}.
If  $\delta_t- \delta_{t+1} \leq \frac{B}{A^2}$ for any $t<\frac{4A}{\mu} \log\left(\frac{D^2\mu^2}{16B}\right) $ then by~\eqref{eq:tesmtpoensr} and~\eqref{eq:zjl9djzedzedze} we have that
\begin{align}\label{eq:tesmtpoensr2}
   \frac{\mu}{2} \delta_t
    & \leq A \left(\delta_t- \delta_{t+1}\right)+\sqrt{B\left(\delta_t- \delta_{t+1}\right)}  \; \leq 2\sqrt{B\left(\delta_t- \delta_{t+1}\right)}\leq 2\sqrt{B\delta_t} \nonumber \quad \implies \\
    \delta_t & \leq     \frac{16B}{\mu^2},
\end{align}
which would prove our result.

Alternatively, suppose that  $\delta_t- \delta_{t+1} \geq \frac{B}{A^2}$ for every $t \leq \frac{4A}{\mu} \log\left(\frac{D^2\mu^2}{16B}\right)$. 
By~\eqref{eq:tesmtpoensr} and~\eqref{eq:zjl9djzedzedze} we have that
\begin{align}
   \frac{\mu}{2} \delta_t
    &\leq  A \left(\delta_t- \delta_{t+1}\right)+\sqrt{B\left(\delta_t- \delta_{t+1}\right)}  \; \leq  2A \left(\delta_t- \delta_{t+1}\right).
    \end{align}
    Re-arranging the above gives
   \begin{equation} \label{eq:templziedzez}
       \delta_{t+1} \leq \left(1-\frac{\mu}{4A}\right)\delta_t. 
   \end{equation} 
Unrolling this for every $t \leq \frac{4A}{\mu} \log\left(\frac{D^2\mu^2}{16B}\right)$ gives
   \begin{equation} \label{eq:templziedzez223}
       \delta_{t} \leq \left(1-\frac{\mu}{4A}\right)^t \delta_0. 
   \end{equation} 
It now follows by taking logarithm and using standard techniques (for example Lemma A.2 in~\citet{garrigos2023handbook}) that $$t \geq \frac{4A}{\mu} \log\left(\frac{D^2\mu^2}{16B}\right) \quad  \implies  \quad \delta_{t} \leq \left(1-\frac{\mu}{4A}\right)^t\delta_0 \leq \frac{16B}{\mu^2}.$$ 


\textit{Induction step:} Now, for ease of notation, let us re-name our iterates so that $\delta_0$ is the first iterate for which $\delta_{0} \leq \frac{16B}{\mu^2}. $


If  $\delta_t- \delta_{t+1} \leq \frac{B}{A^2}$ then by~\eqref{eq:tesmtpoensr} and~\eqref{eq:zjl9djzedzedze} we have that
\begin{align}\label{eq:tesmtpoensr2232}
   \frac{\mu}{2} \delta_t
    &\leq  A \left(\delta_t- \delta_{t+1}\right)+\sqrt{B\left(\delta_t- \delta_{t+1}\right)}  \;\leq 2\sqrt{B\left(\delta_t- \delta_{t+1}\right)} \quad \Leftrightarrow  \nonumber\\
     \frac{\mu^2}{4} \delta_t^2 & \leq 4B \left(\delta_t- \delta_{t+1}\right)   \quad \Leftrightarrow \nonumber \\
   \delta_{t+1}  & \leq (1-\frac{\mu^2}{16B} \delta_t) \delta_t.
\end{align}


Let $a_t = \frac{\mu^2}{16B} \delta_t$. Multiplying both sides of~\eqref{eq:tesmtpoensr2232} by $\frac{\mu^2}{16B} $ and using the induction hypothesis 
$$a_t = \frac{\mu^2}{16B} \delta_t \leq  \frac{\mu^2}{16B}  \frac{16B}{\mu^2} \frac{1}{t+1} =  \frac{1}{t+1}$$
gives
\[ a_{t+1  } \leq (1-a_t) a_t \leq \max_{x \in [0, \; \tfrac{1}{t+1}]}  (1-  x) x = \left(1-\frac{1}{t+1}\right) \frac{1}{t+1} \leq \frac{1}{t+2}.\]
  
Alternatively if $\delta_t- \delta_{t+1} \geq \frac{B}{A^2}$ then 
by~\eqref{eq:tesmtpoensr} and~\eqref{eq:zjl9djzedzedze} we have that
\begin{align}
   \frac{\mu}{2} \delta_t
    &\leq  A \left(\delta_t- \delta_{t+1}\right)+\sqrt{B\left(\delta_t- \delta_{t+1}\right)}  \; \leq  2A \left(\delta_t- \delta_{t+1}\right).
    \end{align}
    Re-arranging the above gives
   \[\delta_{t+1} \leq \left(1-\frac{\mu}{4A}\right)\delta_t. \] 
Using the induction hypothesis and  $t \geq \frac{2A}{\mu}$ we have that
\[ \delta_{t+1} \leq \left(1-\frac{\mu}{4A}\right)\delta_t \leq   \left(1-\frac{\mu}{4A}\right) \frac{16B}{\mu^2} \frac{1}{t+1} \leq \frac{16B}{\mu^2} \frac{1}{t+2}  \]
  where the last inequality follows from 
  \[ \left(1-\frac{\mu}{4A}\right)\frac{1}{t+1} \leq \frac{1}{t+2} \quad \Leftrightarrow \quad  t \geq \frac{2A}{\mu}-2 \quad \Leftarrow  \quad t \geq \frac{2A}{\mu}.\]
\end{proof}

\section{Proofs for \IAM{} : Momentum with Polyak stepsizes}

\subsection{Momentum vs. Iterate Moving Average (\IAM{})}
\label{Asec:iterate-averaging}

Here we detail the relationship between momentum and iterate averaging, which hinges on the following lemma.
\begin{restatable}{lemma}{momentumisIMA}(\citet[Lemma 7.3]{garrigos2023handbook} and \citet[Theorem 1]{pmlr-v157-defazio21a})\label{L:momentum is IMA} \label{lem:mom-and-iterate-av}
The iterates $(x_t)_{t\geq 0}$ generated by ~\eqref{eq:momentum} and the \emph{iterate-moving-average} (\IAM{}) are equivalent to if  $z_{-1} = x_0$, $m_{-1} = 0$ and the  $(\gamma_t,\beta_t)$ parameters of momentum and the \IAM{} parameters $(\eta_t, \lambda_t)$ satisfy
\begin{equation}\label{eq:param-equiv}
    \beta_t  =  \frac{\lambda_t}{1+\lambda_{t}} \frac{\eta_{t-1} }{\eta_t}, \quad  \text{ and } \quad 
    \gamma_t =\frac{\eta_t}{1+ \lambda_{t+1}}, \quad \forall t \geq 0.
\end{equation} 
\end{restatable}
As an example of using the above lemma, a constant learning rate $\eta_t \equiv \eta$ and $\lambda_t=t$ in the \IAM{} method~(\ref{eq:zup}--\ref{eq:xup}) corresponds to  a decreasing learning rate $\gamma_t = \frac{\eta}{1+t}$ and an increasing momentum $\beta_t = \frac{t}{1+t}$ in the momentum method~\eqref{eq:momentum}.

\begin{proof}
    The proof is by induction.
    Our induction hypothesis is that 
 $x_t$ iterates in~\eqref{eq:xup} and~\eqref{eq:momentum} are equivalent upto step $t$ and that the $z_t$ iterates in~\eqref{eq:zup} and $m_t$ in~\eqref{eq:momentum} satisfy
 \begin{eqnarray} \label{eq:ztmt}
     z_t = x_t -(1+\lambda_{t+1})\gamma_t m_t.
 \end{eqnarray}
    For the base case $t=0$ we have from~\eqref{eq:zup} that 
    \begin{align*}
        z_0 &= z_{-1} - \eta_0 g_0\\
        &= x_0 - (1+\lambda_1)\gamma_0 g_0,
    \end{align*}
where in the second equality we used  $z_{-1} = x_0$ and~\eqref{eq:param-equiv}. Since $m_{-1}=0$, we  have from~\eqref{eq:momentum} that $m_0= g_0,$ which proves~\eqref{eq:ztmt} for the base case. As for the  $x_t$ iterates in~\eqref{eq:xup} and~\eqref{eq:momentum} being equivalent for $t=0$ from~\eqref{eq:xup} and~\eqref{eq:ztmt}  we have that
    \begin{align*}
        x_1 &= \frac{\lambda_{1}}{1+\lambda_{1}} x_{0}+\frac{1}{1+\lambda_{1}}z_{0} \\
        &= \frac{\lambda_{1}}{1+\lambda_{1}} x_{0}-\frac{1}{1+\lambda_{1}}(x_{0}-(1+\lambda_1)\gamma_0 m_0) \\
        &= x_0 - \gamma_0 m_0,
    \end{align*}
    which is equivalent to the first step of~\eqref{eq:momentum}.

Suppose now that $x_t$ iterates in~\eqref{eq:xup} and~\eqref{eq:momentum} are equivalent and~\eqref{eq:ztmt} holds upto time $t$. From~\eqref{eq:zup} at step $t+1$ we have that
\begin{align*}
    z_{t+1} &= z_t - \eta_{t+1} g_{t+1}  \\
    & = x_t -(1+\lambda_{t+1})\gamma_t m_t- \eta_{t+1} g_{t+1} . &\mbox{Using~\eqref{eq:ztmt} } \\
    &= x_t -(1+\lambda_{t+1})\gamma_t m_t- (1+\lambda_{t+2})\gamma_{t+1} g_{t+1}   & \mbox{Using~\eqref{eq:param-equiv} } \\
    &=  x_t -\gamma_t m_t+(1+\lambda_{t+2})\gamma_{t+1} \left( \frac{\lambda_{t+1}}{1+\lambda_{t+2}}\frac{\gamma_t}{\gamma_{t+1}} m_t + g_{t+1} \right) \\
    &=   x_{t+1} -(1+\lambda_{t+2})\gamma_{t+1} \left( \beta_{t+1} m_t + g_{t+1} \right) & \mbox{Using~\eqref{eq:xup} and~\eqref{eq:param-equiv} } \\
    &= x_{t+1} - (1+\lambda_{t+2})\gamma_{t+1}m_{t+1}&\mbox{Using~\eqref{eq:momentum} },
\end{align*}
which shows that~\eqref{eq:ztmt} holds at time $t+1.$
Finally $t+1$ step. 
From~\eqref{eq:xup} and~\eqref{eq:zup} we have that
\begin{align*}
x_{t+1} & = \frac{\lambda_{t+1}}{1+\lambda_{t+1}} x_{t}+\frac{1}{1+\lambda_{t+1}}z_{t} \\
&=  \frac{\lambda_{t+1}}{1+\lambda_{t+1}} x_{t}+\frac{1}{1+\lambda_{t+1}}( x_t - (1+\lambda_{t+1})\gamma_{t}m_{t})\\
&= x_t - \gamma_t m_t,
\end{align*}
which is equivalent to~\eqref{eq:momentum}, and thus concludes the proof.
\end{proof}

\subsection{Preliminary Bounds for \IAM{} in both nonsmooth and smooth settings}

Our proofs all start from the following Lemma.
\begin{lemma}\label{lem:descent}
    Consider the iterates of \cref{alg:IAM} with $\lambda_t > 0 $. 
    Assume that $g_t \neq 0$ for all $t\geq 0$. Let $g(x)$ denote the subgradient of $f(x).$
    Denote by $\mathcal{F}_t$ the filtration generated by $\xi_0,\ldots,\xi_{t-1}$.
    If $f_\xi$ is convex for every $\xi$, then:
    \begin{enumerate}[label=(\roman*)]
        \item \label{item:as-bounded}(Almost sure boundedness). With probability one, we have $\|z_t - x_*\| \leq \|x_0-x_*\|$ and $\|x_t - x_* \| \leq \|x_0-x_*\|$ for all $t\geq 0$.
        \item \label{item:single}(Single recurrence) It holds for any $t\geq 0$
        \begin{align}\label{eq:recur-after-cond-exp}
           \E{\norm{z_{t} - x_*}^2 ~\vert~ \mathcal{F}_t}\ & \leq
           \norm{z_{t-1} - x_*}^2 -\frac{\big(f(x_t) -f(x_*)  + \dotprod{g(x_t),z_{t-1}-x_t} \big)_+^2}{ \E{\norm{g_t}^2~\vert~ \mathcal{F}_t}}. 
        \end{align}
        \item \label{item:summed}(Summed recurrence) It holds for any $k\geq 0$
        \begin{align}\label{eq:summed-recurrence}
        \E{\norm{z_{k} - x_*}^2}\ & \leq
       \norm{z_{0} - x_*}^2 - \frac{\big(\sum_{t=0}^{k}\E{f(x_t) -f(x_*) + \dotprod{g(x_t),z_{t-1}-x_t}}\big)_+^2}{ \sum_{t=0}^{k} \E{\norm{g_t}^2}}.
        \end{align}
    \end{enumerate}
    
\end{lemma}
\begin{proof}
Substituting~\eqref{eq:IMA-step}  back into the bound~\eqref{eq:expsquare} gives
\begin{align*}
   \norm{z_{t} - x_*}^2\ & \leq
   \norm{z_{t-1} - x_*}^2 -\frac{\big(f_{\xi_t}(x_t) -f_{\xi_t}(x_*)  + \dotprod{g_t,z_{t-1}-x_t} \big)_+^2}{ \norm{g_t}^2}. 
\end{align*}
This shows that $\|z_t-x_*\| \leq \|z_0-x_*\| = \|x_0-x_*\|$ almost surely for all $t\geq 0$.
Since $x_{t+1}$ is a convex combination of $x_t$ and $z_t$ (see line~\ref{ln:xup} in \cref{alg:IAM}) this also shows by a straightforward induction that $\|x_t - x_* \| \leq \|x_0-x_*\|$ almost surely for all $t\geq 0$.

To prove \ref{item:single}, we apply conditional expectation on the above inequality and using \cref{lem:titu},~\eqref{eq:titu_expectation} we obtain
\begin{align*}
   \E{\norm{z_{t} - x_*}^2 ~\vert~ \mathcal{F}_t}\ & \leq
   \norm{z_{t-1} - x_*}^2 -\frac{\big(f(x_t) -f(x_*)  + \dotprod{\E{ g_t ~\vert~ \mathcal{F}_t},z_{t-1}-x_t} \big)_+^2}{ \E{\norm{g_t}^2~\vert~ \mathcal{F}_t}}. 
\end{align*}
Using that the expectation with respect to this filtration is independent of $x_t$, we have that the stochastic subgradient $\E{g_t  ~\vert~ \mathcal{F}_t}$ is a subgradient of $f(x_t)$,
see Lemma 9.5 in~\citet{garrigos2023handbook} for details\footnote{Very formally, here we need to assume the subgradients $g_{\xi}(x)$ are measurable in $\xi$ so that this expectation is well defined.}. Thus we can write $g(x_t) = \E{g_t  ~\vert~ \mathcal{F}_t}.$

Now, define $a_t := f(x_t) -f(x_*) + \dotprod{g(x_t),z_{t-1}-x_t}$ and $b_t = \E{\norm{g_t}^2~\vert~ \mathcal{F}_t}$.
Using \eqref{eq:recur-after-cond-exp} subsequently for $t=0,\ldots,k$ and using the tower property, we obtain
\begin{align*}
    \E{\norm{z_{k} - x_*}^2}\ & \leq
   \norm{z_{0} - x_*}^2 - \E{\sum_{t=0}^{k} \frac{(a_t)_+^2}{b_t}}.
\end{align*}
Now using \cref{lem:titu},~\eqref{eq:titu_numbers} yields
\[\sum_{t=0}^{k} \frac{(a_t)_+^2}{b_t} \geq \frac{\big(\sum_{t=0}^{k}a_t\big)_+^2}{\sum_{t=0}^{k}b_t},\]
which implies, using \eqref{eq:titu_expectation}, that 
\begin{align*}
    \E{\sum_{t=0}^{k} \frac{(a_t)_+^2}{b_t}} \geq \E{\frac{\big(\sum_{t=0}^{k}a_t\big)_+^2}{\sum_{t=0}^{k}b_t} } \geq \frac{\big(\sum_{t=0}^{k}\E{a_t}\big)_+^2}{ \sum_{t=0}^{k} \E{b_t}}.
\end{align*}
Altogether, we obtain \ref{item:summed}, that is
\begin{align*}
    \E{\norm{z_{k} - x_*}^2}\ & \leq
   \norm{z_{0} - x_*}^2 - \frac{\big(\sum_{t=0}^{k}\E{f(x_t) -f(x_*)  +\dotprod{g(x_t),z_{t-1}-x_t}}\big)_+^2}{ \sum_{t=0}^{k} \E{\norm{g_t}^2}}.
\end{align*}
\end{proof}
For our forthcoming proofs we will also make use of a \emph{Bregman viewpoint} of the \IAM{} step size.
\begin{lemma}[Bregman View] \label{lem:bregmanview}
For any $x_t,x_{t-1},x_* \in \R^d$ and $\lambda_t\geq 0$ it holds
\begin{align}\label{eq:bregmanview}
    \begin{split}
    &\hspace{-4ex} f(x_t) -f(x_*)+\dotprod{g(x_t),z_{t-1}-x_t} \\
    & = (1+\lambda_t)(f_{\xi_t}(x_t) -f_{\xi_t}(x_*))-\lambda_t(f_{\xi_{t}}(x_{t-1})-f_{\xi_t}(x_*)) +\lambda_t B_{f_{\xi_t}}(x_{t-1},x_t), 
     \end{split}
    \end{align}
    where $B_{f_{\xi}}(x,y)$ is the Bregman divergence
    \[B_{f_{\xi}}(x,y) := f_{\xi}(x) -f_{\xi}(y)  -\dotprod{g_{\xi}(y),x-y}. \]
\end{lemma}
\begin{proof}
By re-arranging~\eqref{eq:xup} at time $t-1$ we have that
    \begin{equation}
        z_{t-1}-x_t \; = \; -\lambda_t(x_{t-1}-x_t).
    \end{equation}
    Consequently
    \[f(x_t) -f(x_*)+\dotprod{g(x_t),z_{t-1}-x_t} = f(x_t) -f(x_*)-\lambda_t\dotprod{g(x_t),x_{t-1}-x_t}.\]
    
The proof follows by adding and subtracting $\lambda_t f_{\xi_t}(x_{t-1}) $ as follows
\begin{align*}
    &f_{\xi_t}(x_t) -f_{\xi_t}(x_*)  -\lambda_t \dotprod{g_t,x_{t-1}-x_t} \\
    &=  (1+\lambda_t)f_{\xi_t}(x_t) -f_{\xi_t}(x_*)  -\lambda_t f_{\xi_t}(x_{t-1}) +\lambda_t \left( f_{\xi_t}(x_{t-1})-f_{\xi_t}(x_t)- \dotprod{g_t,x_{t-1}-x_t}\right)\\
    &=   (1+\lambda_t)(f_{\xi_t}(x_t) -f_{\xi_t}(x_*))-\lambda_t(f_{\xi_{t}}(x_{t-1})-f_{\xi_t}(x_*)) +\lambda_t B_{f_{\xi_t}}(x_{t-1},x_t) .
\end{align*}
\end{proof}

%
\begin{lemma}\label{lem:estimate-nominator}
    Consider the iterates of \cref{alg:IAM} with $\lambda_t = t$ and assume that $f_\xi$ is convex for every $\xi$, with subgradients $g_{\xi}.$ Let $g(x)$ be subgradients of $f(x).$ It holds
    \begin{align*}
        \sum_{t=0}^k f(x_t) -f(x_*)+\dotprod{g(x_t),z_{t-1}-x_t} = (k+1)[f(x_k) -f(x_*)]  + \sum_{t=1}^k\lambda_t B_{f}(x_{t-1},x_t),
    \end{align*}
    where $B_{f}$ is defined as in \cref{lem:bregmanview}. In particular, it holds $B_{f}(x_{t-1},x_t) \geq 0$.
\end{lemma}
\begin{proof}
Note that for this proof, we need an additional, and artificial iterate $x_{-1} = x_0.$    Summing over $t=0,\ldots, k$ in~\eqref{eq:bregmanview} we have that 
    \begin{align*}
    &\hspace{-4ex}\sum_{t=0}^k \left(f(x_t) -f(x_*) +\dotprod{g(x_t),z_{t-1}-x_t} \right) \\
    &\overset{\eqref{eq:bregmanview}}{=} \sum_{t=0}^k(1+\lambda_t)(f(x_t) -f(x_*))-\lambda_t(f(x_{t-1})-f(x_*)) + \sum_{t=0}^k\lambda_t B_{f}(x_{t-1},x_t) \\
    &= \sum_{t=0}^k \lambda_{t+1}(f(x_t) -f(x_*))-\lambda_t(f(x_{t-1})-f(x_*))   + \sum_{t=0}^k\lambda_t B_{f}(x_{t-1},x_t)\\
    &= (k+1)[f(x_k) -f(x_*)]  + \sum_{t=1}^k\lambda_t B_{f}(x_{t-1},x_t),
    \end{align*}
    where the second step used $1+\lambda_t = 1+t=\lambda_{t+1}$ , and the last step we used telescoping and the fact that $\lambda_0=0$.
\end{proof}

\subsection{Proof of Theorem~\ref{theo:nonsmooth} about rates for \IAM{} in the nonsmooth setting}
\theononsmooth*
\begin{proof}
We start by applying \cref{lem:descent}, which states that $x_t\in D$ and $z_t\in D$ almost surely for all $t\geq 0$.
Further, \cref{lem:descent}, \ref{item:single} implies that 
\begin{align}\label{eq:tmp-recurrence-i}
           \E{\norm{z_{t} - x_*}^2 ~\vert~ \mathcal{F}_t}\ & \leq
           \norm{z_{t-1} - x_*}^2 -\frac{\big(f(x_t) -f(x_*)  + \dotprod{g(x_t),z_{t-1}-x_t} \big)_+^2}{ \E{\norm{g_t}^2~\vert~ \mathcal{F}_t}}. 
\end{align}
For the denominator of \eqref{eq:tmp-recurrence-i}, we can therefore estimate
\[\E{\norm{g_t}^2~\vert~ \mathcal{F}_t } \leq G^2.\]
Applying expectation, and summing from $t=0,\ldots,k$ (recall that $z_{-1}=x_0$), we get
\begin{align}\label{eq:tmp-recurrence-ii}
    \sum_{t=0}^{k} \E{\big(f(x_t) -f(x_*)  + \dotprod{g(x_t),z_{t-1}-x_t} \big)_+^2} \leq G^2\big[\norm{x_0 - x_*}^2 - \E{\norm{z_{k} - x_*}^2} \big].
\end{align}
Now, applying \eqref{eq:titu_numbers} with $b_t=1$ we get for any $a_0,\ldots,a_k$ that 
$\sum_{t=0}^k (a_t)_+^2 \geq \frac{1}{k+1} \big(\sum_{t=0}^{k}a_t\big)_+^2$. 
Therefore, we conclude
\begin{align*}
&\sum_{t=0}^{k} \big(f(x_t) -f(x_*)  + \dotprod{g(x_t),z_{t-1}-x_t} \big)_+^2 \geq \frac{1}{k+1} \Big(\sum_{t=0}^k f(x_t) -f(x_*)+\dotprod{g(x_t),z_{t-1}-x_t} \Big)_+^2\\
&\hspace{3ex} \geq \frac{1}{k+1} \Big((k+1)[f(x_k) -f(x_*)]  + \sum_{t=1}^k\lambda_t B_{f}(x_{t-1},x_t) \Big)_+^2\\
&\hspace{3ex}= \Big(\sqrt{k+1}[f(x_k) -f(x_*)] +  \sum_{t=1}^k\frac{\lambda_t}{\sqrt{k+1}} B_{f}(x_{t-1},x_t)\Big)^2,
\end{align*}
where we used \cref{lem:estimate-nominator} in the second step, and non-negativity of all terms in the third step.
Define $\bar{B}_k := \sum_{t=1}^k\lambda_t B_{f}(x_{t-1},x_t) \geq 0$. Plugging this into \eqref{eq:tmp-recurrence-ii}, we get
\begin{align*}
    \E{\Big(\sqrt{k+1}[f(x_k) -f(x_*)] +  \frac{1}{\sqrt{k+1}}\bar{B}_k\Big)^2}
    \leq G^2\big[\norm{x_0 - x_*}^2 - \E{\norm{z_{k} - x_*}^2} \big].
\end{align*}
Now, using Jensen's inequality $\E{X}^2 \leq \E{X^2}$, taking the square-root, and dividing by $\sqrt{k+1}$, we finally obtain
\begin{align*}
    \E{f(x_k) -f(x_*)} + \frac{1}{k+1}\E{\bar{B}_k} \leq \frac{G\norm{x_0 - x_*}}{\sqrt{k+1}}.
\end{align*}
\end{proof}

\subsection{Proof of Theorem~\ref{theo:smooth}  about rates for \IAM{} in the smooth setting}
\theosmooth*
\begin{proof} 
    We start the proof by applying \cref{lem:descent},~\ref{item:summed}, which yields
    \begin{align*}
        \E{\norm{z_{k} - x_*}^2}\ & \leq
       \norm{z_{0} - x_*}^2 - \frac{\big(\sum_{t=0}^{k}\E{f(x_t) -f(x_*) + \dotprod{\nabla f(x_t),z_{t-1}-x_t}}\big)_+^2}{ \sum_{t=0}^{k} \E{\norm{g_t}^2}}.
    \end{align*}
    For the numerator of the last term, use \cref{lem:estimate-nominator} and the fact that $(\cdot)_+^2$ is monotonic to obtain
    \begin{align*}
        \big(\sum_{t=0}^{k}\E{f(x_t) -f(x_*) + \dotprod{\nabla f(x_t),z_{t-1}-x_t}}\big)_+^2 &\geq \big(\E{(k+1)(f(x_k)-f(x_*)}\big)_+^2 \\
        &= (k+1)^2 \E{f(x_k)-f(x_*)}^2.
    \end{align*}
    For the denominator, use \eqref{eq:expsmooth-sps} to write that $\E{\norm{g_t}^2}\le 2L \E{(f(x_t) - f(x_*) + \Delta_*)} $.
    Thus, using $z_0=x_0$, we get
    \begin{align*}
	  \E{ \norm{z_{k} - x_*}^2} 
        \le \norm{x_{0} - x_*}^2  - \frac{(k+1)^2\E{(f(x_k) -f(x_*))}^2}{2L\sum_{t=0}^k \E{(f(x_t) - f(x_*) + \Delta_*)}}.
   \end{align*}

   Let $c_t = \mathbb{E}[f(x_t) - f(x_*)] + \Delta_*$ and $S_k = \sum_{t=0}^k c_t$, then we can rewrite the above as 
    \begin{align*}
	  \frac{\E{(f(x_k) -f(x_*))}^2}{S_k}
        \le \frac{2L}{(k+1)^2}(\norm{x_{0} - x_*}^2  - \E{ \norm{z_{k} - x_*}^2} )
        \le \frac{2L \norm{x_{0} - x_*}^2}{(k+1)^2}.
   \end{align*}
   Taking the square-root yields
   \begin{align}\label{eq:tempolznerze}
       \frac{\E{f(x_k) -f(x_*)}}{\sqrt{S_k}}
       \le \frac{\sqrt{2L} \norm{x_{0} - x_*}}{k+1}.
   \end{align}
   Finally, notice that $\E{f(x_k) -f(x_*)} = c_k - \Delta_*$, so we arrive at the inequality
   \begin{align*}
       \frac{c_t}{\sqrt{S_t}}
       \le \frac{\sqrt{2L} \norm{x_{0} - x_*}}{t+1} + \frac{\Delta_*}{\sqrt{S_t}}.
   \end{align*}
   Summing this from $t=0$ to $k$ and then applying $\sum_{t=0}^k \frac{1}{t+1}\le \log (k+1) + 1$ and $S_t\ge (t+1)\Delta_*$ gives
   \begin{align*}
        \sqrt{S_k}
        &\overset{\eqref{eq:adagrad_lemma}}{\le} \sum_{t=0}^k\frac{c_t}{\sqrt{S_t}}
        \le \sum_{t=0}^k\frac{\sqrt{2L} \norm{x_{0} - x_*}}{t+1} + \sum_{t=0}^k\frac{\Delta_*}{\sqrt{S_t}} \\
        &\le \sqrt{2L} \norm{x_{0} - x_*}(\log(k+1)+1) + \sum_{t=0}^k\frac{\sqrt{\Delta_*}}{\sqrt{t+1}}.
   \end{align*}
   Furthermore,  it holds $\sum_{t=0}^k\frac{1}{\sqrt{t+1}} \le 2\sqrt{k+1}$, so we finally get
   \begin{align*}
       \sqrt{S_k}
       \le \sqrt{2L} \norm{x_{0} - x_*}(\log(k+1)+1) + \sqrt{\Delta_*}\sqrt{k+1}.
   \end{align*}
Using the above inequalities in~\eqref{eq:tempolznerze} gives
   \begin{align*}
        \E{f(x_k) -f(x_*)}
       \le \frac{\sqrt{2L} \norm{x_{0} - x_*}\sqrt{S_k}}{k+1}
       \le \frac{2L\norm{x_{0} - x_*}^2 (\log(k+1)+1)}{k+1} + \frac{\sqrt{2L\Delta_*}\norm{x_{0} - x_*}}{\sqrt{k+1}}.
   \end{align*}
\end{proof}

\subsection{Additional Result: Convergence for \IAM{} in the Nonsmooth Setting with Decreasing $\lambda_t$}

\begin{theorem}
    \label{thm:iaam-relax-nonsmooth}
  Consider the setting of~\Cref{theo:nonsmooth}, except that $\lambda_0 =0$ and $(\lambda_t)_{t=1}^k$ is any decreasing sequence of nonnegative reals. 
   It follows that
    \begin{align*}
        \mathbb{E}[f(\overline{x}_k)-f(x_*)]+\frac{1}{k+1}\sum_{t=0}^k\lambda_t
        \mathbb{E}[B_{f}(x_{t-1},x_t)]\leq\frac{G\|x_0-x_*\|}{\sqrt{k+1}}+\frac{\lambda_1}{k+1}\mathbb{E}[f(x_0)-f(x_*)].
    \end{align*}    
\end{theorem}
The advantage of this result is that it holds for any constant $\lambda_t=\lambda$.
Translating this to the momentum method~\eqref{eq:momentum}, this  allows for other parameter setting of $(\gamma_t, \beta_t)$. 
In particular the setting $\lambda_t=\lambda=0$ which corresponds to no momentum. In this setting we retrieve the exact same rate as the \SPS* method in~\Cref{cor:spsnonsmooth}. However, as mentioned earlier the price we need to pay for this, is that this result holds for the Cesaro average and not of the last iterate. 

\begin{proof}
  Starting  from \cref{lem:descent}:
\begin{align}\label{eq:tmp-recurrence-iii}
           \norm{z_{t} - x_*}^2 \ & \leq
           \norm{z_{t-1} - x_*}^2 -\frac{\big(f(x_t) -f(x_*)  + \dotprod{g_t,z_{t-1}-x_t} \big)_+^2}{ \norm{g_t}^2}. 
\end{align}
    Taking expectation, and using our extended Titu's ~\Cref{lem:titu} and Bregman viewpoint~\Cref{lem:bregmanview}  to get 
    \begin{align*}
        \mathbb{E}\|z_t-x_*\|^2
        &\leq\mathbb{E}\|z_{t-1}-x_*\|^2-\frac{\mathbb{E}[f(x_t)-f(x_*)+\langle g(x_t),z_{t-1}-x_t\rangle]_+^2}{\mathbb{E}\|g_t\|^2}\\
        &\leq\mathbb{E}\|z_{t-1}-x_*\|^2-\frac{\mathbb{E}[(1+\lambda_t)[f(x_t)-f(x_*)]-
        \lambda_t[f(x_{t-1})-f(x_*)]+\lambda_tB_{f}(x_{t-1},x_t)]_+^2}{\mathbb{E}\|g_t\|^2}\\
        &\leq\mathbb{E}\|z_{t-1}-x_*\|^2-\frac{\mathbb{E}[(1+\lambda_t)[f(x_t)-f(x_*)]-
        \lambda_t[f(x_{t-1})-f(x_*)]+\lambda_tB_{f}(x_{t-1},x_t)]_+^2}{G^2}.
    \end{align*}
    Multiplying through by $G^2$ gives
    \begin{align*}
        \mathbb{E}[(1+\lambda_t)[f(x_t)-f(x_*)]-\lambda_t[f(x_{t-1})-f(x_*)]+\lambda_t
        B_{f}(x_{t-1},x_t)]_+^2\leq G^2\mathbb{E}\|z_{t-1}-x_*\|^2-G^2\mathbb{E}\|z_t-x_*\|^2.
    \end{align*}
    Now let $\Delta_t=(1+\lambda_t)[f(x_t)-f(x_*)]-\lambda_t[f(x_{t-1})-f(x_*)]+
    \lambda_tB_{f}(x_{t-1},x_t)$. Averaging both sides of the above over $t=0,\dots,k$, telescoping terms, and using Jensen's inequality with respect to the convex function $x \mapsto (x_+)^2$ gives 
    \begin{align*}
        \frac{G^2\|x_0-x_*\|^2}{k+1}
        &\geq\frac{G^2}{k+1}\left(\mathbb{E}\|x_0-x_*\|^2-\mathbb{E}\|z_{k+1}-x_*\|^2\right)\\
        &\geq\frac{1}{k+1}\sum_{t=0}^k\mathbb{E}[\Delta_t]_+^2\\
        &\geq\left(\frac{1}{k+1}\sum_{t=0}^k\mathbb{E}[\Delta_t]\right)_+^2.
    \end{align*}
Taking the square root gives
    \begin{align}\label{eq:tempieoh88zec}
        \left(\frac{1}{k+1}\sum_{t=0}^k\mathbb{E}[\Delta_t]\right)_+\leq\frac{G\|x_0-x_*\|}
        {\sqrt{k+1}}.
    \end{align}
    Now since $(\lambda_t)$ is decreasing and using Jensen's inequality with respect to $x \mapsto f(x)$ we have that
    \begin{align*}
        \sum_{t=0}^k\mathbb{E}[\Delta_t]
        &=\sum_{t=0}^k(1+\lambda_t)\mathbb{E}[f(x_t)-f(x_*)]-\lambda_t\mathbb{E}[f(x_{t-1})-
        f(x_*)]+\lambda_t\mathbb{E}[B_{f}(x_{t-1},x_t)]\\
        &=\sum_{t=0}^k\lambda_t\mathbb{E}[B_{f}(x_{t-1},x_t)]+\sum_{t=0}^k\mathbb{E}[f(x_t)-
        f(x_*)]+\sum_{t=0}^k\lambda_t\mathbb{E}[f(x_t)-
        f(x_*)]-\sum_{t=0}^k\lambda_t\mathbb{E}[f(x_{t-1})-f(x_*)]\\
        &=\sum_{t=0}^k\lambda_t\mathbb{E}[B_{f}(x_{t-1},x_t)]+\sum_{t=0}^k\mathbb{E}[f(x_t)-
        f(x_*)]+\sum_{t=1}^k\lambda_t\mathbb{E}[f(x_t)-
        f(x_*)]-\sum_{t=1}^k\lambda_t\mathbb{E}[f(x_{t-1})-f(x_*)]\\
        &=\sum_{t=0}^k\lambda_t\mathbb{E}[B_{f}(x_{t-1},x_t)]+\sum_{t=0}^k\mathbb{E}[f(x_t)-
        f(x_*)]+\sum_{t=1}^{k-1}(\lambda_t-\lambda_{t+1})\mathbb{E}[f(x_t)-f(x_*)]\\
        &\quad+\lambda_k\mathbb{E}[f(x_k)-f(x_*)]-\lambda_1\mathbb{E}[f(x_0)-f(x_*)]\\
        &\geq\sum_{t=0}^k\lambda_t\mathbb{E}[B_{f}(x_{t-1},x_t)]+\sum_{t=0}^k\mathbb{E}[f(x_t)-
        f(x_*)]-\lambda_1\mathbb{E}[f(x_0)-f(x_*)]\\
        &\geq\sum_{t=0}^k\lambda_t\mathbb{E}[B_{f}(x_{t-1},x_t)]+(k+1)\mathbb{E}[f(\overline{x}_k)-f(x_*)]-\lambda_1\mathbb{E}[f(x_0)-f(x_*)],
    \end{align*}
    where we used that $\lambda_0 =0$ and $\lambda_t-\lambda_{t+1} \geq 0$ since $(\lambda_t)$ is a decreasing sequence. Dividing through by $(k+1)$ gives
\begin{align*}
        \frac{1}{k+1}\sum_{t=0}^k\mathbb{E}[\Delta_t]
        &\geq \frac{1}{k+1}\sum_{t=0}^k\lambda_t\mathbb{E}[B_{f}(x_{t-1},x_t)]+\mathbb{E}[f(\overline{x}_k) \
        -f(x_*)] -\frac{\lambda_1}{k+1}\mathbb{E}[f(x_0)-f(x_*)].
    \end{align*}
Using the above and~\eqref{eq:tempieoh88zec} gives
    \begin{align*}
        \left(\frac{1}{k+1}\sum_{t=0}^k\lambda_t\mathbb{E}[B_{f}(x_{t-1},x_t)]+
        \mathbb{E}[f(\overline{x}_k)-f(x_*)] -\frac{\lambda_1}{k+1}\mathbb{E}[f(x_0)-f(x_*)]\right)
        &\leq\left(\frac{1}{k+1}\sum_{t=0}^k\mathbb{E}[\Delta_t]\right)_+\\
        &\leq\frac{G\|x_0-x_*\|}{\sqrt{k+1}}.
    \end{align*}
Re-arranging gives the result.
\end{proof}

\subsection{Additional Result: Deriving an Expression for \texttt{Adam} with Polyak Stepsizes} \label{asec:iam-adam}

Following an analogous reasoning used in~\Cref{sec:iams}, we can derive variants of \IAM{} that use preconditioning. This is particularily important for models such as Transformers, where using an \texttt{Adam} preconditioner is required to achieve a reasonable performance. 

   To arrive at a preconditioned version of \IAM{}, let $\mD_t \in \R^{d\times d}$ be our positive definite symmetric preconditioner, 
 and let  $ \norm{z }^2_{\mD_t} := \dotprod{\mD_t z, z}$ be the norm induced by this preconditioner. Now consider the iterative averaging method with this preconditioner:
 \begin{eqnarray}
z_{t} & = & z_{t-1}-\eta_{t}\mD_t^{-1}g_t,\label{eq:zup-precon}\\
x_{t+1} & = &\frac{\lambda_{t+1}}{1+\lambda_{t+1}} x_{t}+\frac{1}{1+\lambda_{t+1}}z_{t}. \label{eq:xup-precon}
\end{eqnarray}
Now we upper bound the distance between $z_t$ and a solution $x_*$ under the preconditioned norm via
\begin{align}
   \norm{z_{t} - x_*}^2_{\mD_t} \
    & = 
    \norm{z_{t-1} - x_*}^2_{\mD_t}   -2\eta_t \dotprod{\mD_t^{-1}g_t, z_{t-1}-x_*}_{\mD_t}  + \eta_t^2 \norm{g_t}_{\mD_t^{-1}} ^2  \nonumber \\
     & = 
    \norm{z_{t-1} - x_*}^2_{\mD_t}   -2\eta_t \dotprod{g_t, z_{t-1}-x_*}  + \eta_t^2 \norm{g_t}_{\mD_t^{-1}} ^2 \nonumber \\
    & \leq  \norm{z_{t-1} - x_*}^2_{\mD_t}   -2\eta_t \big( f_{\xi_t}(x_t) -f_{\xi_t}(x_*)  + \dotprod{g_t,z_{t-1}-x_t} \big)+ \eta_t^2 \norm{g_t}_{\mD_t^{-1}}^2,  \nonumber 
\end{align}
where in the inequality we used 
that $f_{\xi_t}$ is convex.
Minimizing the right-hand side with respect to $\eta_t$ now gives the step size given in line~\ref{ln:etaP} in Algorithm~\ref{alg:IMAP}. 
\begin{algorithm}[h!]
\begin{algorithmic}[1]
    \caption{\texttt{IAM-Adam}}
    \label{alg:IMAP}
    \STATE \textbf{Input:} $z_{-1}=x_0 \in \R^d$, $\lambda_t > 0$
    \FOR{$t=0$ to $T-1$}
    	\STATE $ \eta_t \;= \; \displaystyle \frac{\big[ f_{\xi_t}(x_t) -\ell_{\xi_t}^*+\dotprod{g_t,z_{t-1}-x_t}\big]_+}{\norm{g_t}_{\mD_t^{-1}}^2},$\label{ln:etaP}
        \STATE $ z_{t} \; = \; z_{t-1}-\eta_{t}\mD_t^{-1}g_t  $ \label{ln:zupP}
        \STATE $x_{t+1}  \;= \; \displaystyle \frac{\lambda_{t+1}}{1+\lambda_{t+1}} x_{t}+\frac{1}{1+\lambda_{t+1}}z_{t} $ \label{ln:xupP}
    \ENDFOR
    \STATE \textbf{Return:} $x_T$
\end{algorithmic}
\end{algorithm}
To arrive at our \texttt{IAM-Adam} method, we simply set $\mD_t$ to be the preconditioner used by \texttt{Adam}, that is $\mD_t= \texttt{diag}(\sqrt{v_t} + \epsilon)$  where 
\[v_{t+1} =  \beta_2 v_{t} + (1-\beta_2) g_t \odot g_t. \]

\section{Discussion}\label{S:discussion}

\subsection{Anytime Convergence Rates vs. Finite Horizon Complexity Rates}

Here we take the time to properly define what we mean by (anytime) convergence rates and (finite horizon) complexity rates.
For the sake of the discussion, we consider a certain quantity of interest $(q_t)_{t \in \mathbb{N}} \subset (0, +\infty)$ which we want to be small when $t$ grows.
Typically $q_t$ will measure an optimality gap for some algorithm, such as the function value gap $\mathbb{E}\left[f(x_t) - \inf f  \right]$.
What is the algorithm here does not really matter. What does matter is that this algorithm depends on hyperparameters, which once fixed generate a sequence which in turn define the quantities of interest $(q_t)_{t \geq 0}$.

\begin{enumerate}
    \item We say that we have an anytime \emph{convergence rate} for $q_t$ if, for every choice of hyperparameter, there exists a rate function $r : [0, +\infty) \to [0, +\infty)$ such that $q_t \leq r(t)$ holds true for every $t \in \mathbb{N}$, and $r(t) \to 0$ when $t \to + \infty$.
    \item We say that we have a finite horizon \emph{complexity rate} for $q_t$ if, for every tolerance $\varepsilon > 0$, there exists a choice of hyperparameters and $T \in \mathbb{N}$ such that $q_T \leq \varepsilon$.
\end{enumerate}

A typical example of convergence rate is the function value gap for the Gradient Descent algorithm with constant stepsize: for every choice of stepsize $\gamma \in (0, \tfrac{1}{L})$, for every starting point $x_0$,  we know that
\begin{equation*}
    f(x_t) - \inf f \leq \frac{\Vert x_0 - x_* \Vert^2 }{2 \gamma t}.
\end{equation*}
An other example is SGD with a vanishing stepsize.
It should be clear that convergence rates entail complexity rate: 
given any tolerance $\varepsilon>0$, because $r(t) \to 0$ we can find a $T$ large enough so that $r(T) \leq \varepsilon$.
The reverse is in general not true: $q_t$ can have a complexity rate without having a convergence rate.
A typical example is SGD with a constant stepsize: for every choice of stepsize $\gamma \in (0, \tfrac{1}{4L})$, for every starting point $x_0$,  we know that
\begin{equation*}
    \mathbb{E}\left[ f(\bar x_t) - \inf f \right] \leq \frac{\Vert x_0 - x_* \Vert^2 }{2\gamma t} + 2 \gamma \sigma^2_*.
\end{equation*}
This means that, for every fixed tolerance $\varepsilon > 0$, we can chose a stepsize $\gamma = O(\varepsilon)$ small enough, and take $T = O(\tfrac{1}{\varepsilon^2})$ large enough so that 
\begin{equation*}
    \mathbb{E}\left[ f(\bar x_t) - \inf f \right] \leq \varepsilon.
\end{equation*}
The key difference bewteen convergence rates and complexity rates is that for convergence rates the user does not need to fix the hyperparameters as a function of the desired tolerance.

\subsection{Complexity of \SGD{} with Respect to Interpolation}

\begin{theorem}[Complexity of \SGD{}]\label{theo:SGDadaptsigma}
    Let $f = \tfrac{1}{n} \sum_{i=1}^n f_i$ where each $f_i : \mathbb{R}^d \to \mathbb{R}$ is convex and $L$-smooth, and assume that $f$ admits a minimizer, noted $x_*$.
    Let $x_0 \in \mathbb{R}^d$, and note $D := \Vert x_0 - x_* \Vert$ and $\sigma_*^2 = \mathbb{E}\left[ \Vert \nabla f_i(x_*)\Vert^2 \right]$.
    Let $T \geq 1$ be fixed, and let $\gamma = \tfrac{ \gamma_0}{\sqrt{v T + 1}}$ where $\gamma_0 \leq \frac{1}{4L}$, and $v \geq 0$ is a variance upper estimate of $\sigma_*^2$ satisfying $\sigma_*^2 \leq C v$ for some $C>0$.
    Let $(x_t)_{t=0}^T$ be the sequence generated by the \SGD{} algorithm with constant stepsize $\gamma$.
    Then
    \begin{equation*}
        \mathbb{E}\left[ f(\bar x_T) - \inf f \right]
        \leq 
        \frac{D^2}{\gamma_0 T} + \frac{\sqrt{v}}{\sqrt{T}} \left( \frac{D^2}{\gamma_0} + 2 \gamma_0 C \right),
    \end{equation*}
    where $\bar x_T = \tfrac{1}{T}\sum_{t=0}^{T-1} x_t$.
    In particular the choice $v = \sigma_*^2$ leads to $\gamma = \tfrac{ \gamma_0}{\sqrt{\sigma_*^2 T + 1}}$ and gives the rate
    \begin{equation*}
        \mathbb{E}\left[ f(\bar x_T) - \inf f \right]
        \leq 
        \frac{D^2}{\gamma_0 T} + \frac{\sigma_*}{\sqrt{T}} \left( \frac{D^2}{\gamma_0} + 2 \gamma_0 \right).
    \end{equation*}
\end{theorem}
The above theorem shows that \SGD{} enjoys a $\cO{LD^2/T + \sigma_*LD^2/\sqrt{T}}$ complexity, which is similar to the result of \SPS* in~\Cref{cor:spssmooth}. But there are some important differences. First, \SPS* has an anytime convergence rate valid for every $T$, while \SGD{} has a complexity rate: it is a "finite horizon" rate where the horizon must be known before setting the stepsize.  
The other difference is that for \SGD{} to achieve this complexity, we need access to the smoothness constant $L$. 
In contrast, \SPS* adapts to both smoothness and non-smoothness.
Of course the main counterpart for \SPS* to achieve this is that it requires access to $f_\xi(x_*)$.
Note that we discuss in Section \ref{sec:approx} on how to implement \SPS* in practice.

The last important point is that \SGD{} needs access to the gradient variance constant $\sigma_*^2$, or at least a {faithful} upper estimate $v$ of $\sigma_*^2$. 
By \emph{faithful}, we mean here a constant $v$ which is zero whenever interpolation holds, so that the complexity reduces to $\cO{\tfrac{1}{T}}$ when interpolation holds.
An example of such faithful upper estimate is the function noise constant $\Delta_* := \inf f - \mathbb{E}_\xi \left[ \inf f_\xi \right]$, since it satisfies $\sigma_*^2 \leq 2 L \Delta_*$ and is zero if and only if interpolation holds (see Lemma 4.15 \& 4.18 in \cite{garrigos2023handbook}).
Observe that if we assume  being able to access $f_\xi(x_*)$, then we can reasonably compute $\Delta_*$. 
In that case \SGD{} can benefit from it by setting $\gamma = \tfrac{ \gamma_0}{\sqrt{\Delta_* T + 1}}$, leading to a complexity of the form
    \begin{equation*}
        \mathbb{E}\left[ f(\bar x_T) - \inf f \right]
        \leq 
        \frac{D^2}{\gamma_0 T} + \frac{\sqrt{\Delta_*}}{\sqrt{T}} \left( \frac{D^2}{\gamma_0} + 4 \gamma_0 L \right),
    \end{equation*}
which is (up to constants) as good as the one of \SPS*, except for the fact that this is a \emph{finite horizon complexity}.



\begin{proof}[Proof of \Cref{theo:SGDadaptsigma}]
    We are using the complexity rate of \SGD{} for a sum of convex smooth functions from Theorem 5 in \cite{garrigos2023handbook}. 
    The result states that, provided $\gamma \leq \tfrac{1}{4L}$, we can guarantee
    \begin{equation*}
        \mathbb{E}\left[ f(\bar x_T) - \inf f \right]
        \leq 
        \frac{D^2}{\gamma T} + {2 \gamma \sigma_*^2}.
    \end{equation*}
    With our choice of stepsize $\gamma = \tfrac{ \gamma_0}{\sqrt{v T + 1}}$, one sees that the hypothesis $\gamma \leq \tfrac{1}{4L}$ is guaranteed as long as $\gamma_0 \leq \frac{1}{4L}$.
    The desired bound then follows, after  cleaning some constants.
    More precisely,  write 
    \begin{equation*}
        \frac{1}{\gamma}
        =
        \frac{\sqrt{1 + v T}}{\gamma_0}
        \leq
        \frac{1 + v\sqrt{  T}}{\gamma_0}
        \quad \text{ and } \quad         
        2\gamma \sigma_*^2
        =
        \frac{2\gamma_0 \sigma_*^2}{\sqrt{1 + v T}} 
        \leq 
        \frac{2\gamma_0 Cv}{\sqrt{1 + v T}} 
        \leq 
        \frac{2\gamma_0 C \sqrt{v}}{\sqrt{ T}},
    \end{equation*}
    and combine all the above inequalities to conclude that
    \begin{equation*}
        \mathbb{E}\left[ f(\bar x_T) - \inf f \right]
        \leq 
        \frac{D^2}{\gamma T}
        + 2\gamma \sigma_*^2
        \leq
        \frac{D^2(1+v \sqrt{T})}{\gamma_0 T}
        + 2\frac{\gamma_0 \sigma_*^2}{\sqrt{ T}}
        =
        \frac{D^2}{\gamma_0 T} 
        +
        \frac{D^2\sigma_*^2 }{\gamma_0 \sqrt{T}}
        +
        2\frac{\gamma_0 \sigma_*^2}{\sqrt{ T}}.
    \end{equation*}
\end{proof}

\subsection{Comparison of Rates with a Stochastic Polyak Stepsize, in the Smooth Case}
\label{sec:sps-smooth-details}


Let us start with \SGD{}.  
Previously, Theorem 2.3.2 in \cite{garrigos2023function} showed a $\cO{1/{T}}$ anytime convergence rate for \SPS* under an interpolation assumption, but with no guarantees without interpolation.
Theorem\ 3.4 in \cite{SPS} provided the following general anytime bound for $\SPS{_{\max}}$, independently on interpolation holding true or not:
\begin{equation*}
    \E{f(\bar{x}_T) -f(x_*)} \leq \max \left( 1, 2 \gamma_b L \right)   \left( \frac{\Vert x_0 - x_* \Vert^2}{\gamma_b T} + 2\Delta_*  \right),
\end{equation*}
where $\gamma_b$ is the parameter appearing in the definition of $\SPS{_{\max}}$, recall eq. \eqref{eq:SPSmax-intro} ; and $\Delta_* = \inf f - \mathbb{E} \inf f_\xi$ is a constant which is zero if and only if interpolation holds (see Section 5.3 in \cite{garrigos2023handbook} for a proof).
Unfortunately, this bound is not a convergence rate.
Even worse, because the  variance term $\Delta_*$ cannot be controlled say, by a multiplicative free parameter), the above bound cannot be converted into a complexity result, even if we had access to the value of $L$.
More recently, a variant of stochastic Polyak stepsizes called NGN was proposed~\cite{orvieto2024NGN}, and provably enjoys the following complexity bound (see Theorem 4.5 in \cite{orvieto2024NGN}):
\begin{equation*}
    \E{f(\bar{x}_T) -f(x_*)} \leq  
    \cO{\frac{\Vert x_0 - x_* \Vert^2}{\gamma' T} + \gamma' \Delta_* + \gamma'(2\gamma' L - 1)_+ \mathbb{E}\left[ \inf f_\xi \right]},
\end{equation*}
where $\gamma'> 0$ is a free parameter of the NGN method.
One can see that this rate can be converted into a $\mathcal{O}(1/\sqrt{T})$ complexity, provided we take a small enough parameter $\gamma'$. 

In contrast, our smooth result in~\Cref{cor:spssmooth} is, as far as we know, the first $\cO{1/\sqrt{T}}$ \emph{anytime convergence rate} which is adaptive to interpolation for a stochastic variant of the Polyak stepsize, let alone for vanilla \SGD{}. 
For \SGD{}, it is not possible to have at the same time an anytime $\cO{1/\sqrt{T}}$ convergence rate together with adaptivity to interpolation, at least up to our knowledge.
For instance, taking vanishing stepsizes $\gamma_t \propto \tfrac{1}{\sqrt{t}}$ guarantees (see Theorem 5.7 in \cite{garrigos2023handbook}) an anytime convergence rate 
\begin{equation*}
    \E{f(\bar{x}_T) -f(x_*)} \leq 
    \cO{\frac{1 + \sigma_*^2 \log(T+1)}{\sqrt{T}}}
\end{equation*}
which is an anytime $\cO{1/\sqrt{T}}$ convergence rate, but does not revert to a $\cO{1/{T}}$ rate when interpolation holds.
This is an instance of a rate which does not benefit from interpolation.
Another choice is to fix a finite horizon $T \geq 1$ and to take $\gamma \propto \tfrac{1}{\sqrt{T}}$ which guaranteees (see Theorem 5.5 in \cite{garrigos2023handbook}) a bound 
\begin{equation*}
    \E{f(\bar{x}_T) -f(x_*)} \leq   
    \cO{\frac{1 + \sigma_*^2 }{\sqrt{T}}}
\end{equation*}
which has the same issues as the above mentioned result, on top of not being an anytime convergence rate.
Finally, if one can estimate $\sigma_*^2$, then it is possible to set the stepsize $\gamma$ as in Theorem~\ref{theo:SGDadaptsigma} to obtain a rate
\begin{equation*}
    \E{f(\bar{x}_T) -f(x_*)} \leq   
    \cO{\frac{1 }{{T}} + \frac{\sigma_*^2}{\sqrt{T}}}
\end{equation*}
which is a finite horizon (complexity) rate.

Let us now discuss and compare our result in  \cref{theo:smooth} about rates for Momentum with Polyak stepsizes, an algorithm that we call  \IAM{} (see \cref{alg:IAM}).
It is known that Momentum (as rewritten in \eqref{eq:zup}-\eqref{eq:xup}) with a learning rate $\eta \leq \tfrac{1}{4L}$ enjoys a finite iterate complexity rate (see Theorem 7.4 in \cite{garrigos2023handbook}):
\begin{equation*}
    \E{f({x}_T) -f(x_*)} \leq   
    \frac{\Vert x_0 - x_* \Vert^2}{\eta (T+1)} + 2 \eta \sigma_*^2.
\end{equation*}
This is finite horizon complexity rate, which is not adaptive with respect to $L$.
Instead, we show that using a Polyak stepsize as defined in \cref{alg:IAM} allows for a better anytime convergence rate, which is adaptive to $L$.
\if{
For \SGD{} with a learning rate of $\eta_t = \eta/\sqrt{t+1}$ where $\eta \leq \frac{1}{2L}$, the average iterate converges according to~\citet[Thm.\ 5.5]{garrigos2023handbook}: 
\begin{equation}\label{eq:SGDconv}
        \E{f(\bar{x}_T) -f(x_*)}
       \le \frac{\norm{x_0 - x_*}^2}{2\eta \sqrt{T+1}}  +  \frac{\eta\log(T+1)}{ \sqrt{T+1}}\sigma_*^2.
\end{equation}  
where
$\bar{x}_T \; := \; \sum_{t=0}^{T-1} p_{T,t} x_t$, with
	$p_{T,t} := \frac{\eta_T(1 - 2\eta_T L)}{\sum_{i=0}^{t-1}\eta_i(1 - 2\eta_i L)}$. In comparison to~\eqref{eq:SGDconv} the analysis in~\eqref{eq:theosmooth} of the
\IAM{}  method has two advantages: First, there is no additional $\log(T+1)$ term multiplying the dominating $\mathcal{O}(1/\sqrt{T+1})$ term, and it is adaptive to $L$. That is, \IAM{} does not need access to $L$ to achieve this same anytime convergence. Furthermore~\eqref{eq:SGDconv} is not adaptive to interpolation, in that when $\sigma_*^2=0$, the resulting rate of convergence is $\mathcal{O}(1/\sqrt{T+1})$, as opposed to the $\mathcal{O}(1/(T+1))$ rate that can be achieved under interpolation~\cite{pmlr-v80-ma18a,SGDstruct}.
}\fi 

\subsection{Comparison of Adaptive Methods}
In 
\Cref{tab:compare_with_others}
we make a qualitative comparison between our methods \SPS* and \IAM{} and other adaptive methods. \Cref{tab:compare_with_others} highlights how \SPS*,  to the best of our knowledge, is the only stochastic adaptive stepsize  that has favorable convergence rates in the smooth and  non-smooth problems, without having to modify the method, admits an unbounded domain and can increase the stepsize (not monotonic). The \IAM{} method shares the same benefits, but also has a fast last iterate convergence, as opposed to the average iterate for \SPS*.

\newcommand{\mycolspace}{1.5pt}
\setlength{\tabcolsep}{\mycolspace}

\begin{table}[t]
    \centering
\setlength\tabcolsep{3.5pt} 
  \begin{threeparttable}[b]
    {
	\renewcommand\arraystretch{1.8}
	\caption{A summary of related work and conceptual differences to our approach and the work in  AcceleGrad \cite{levy2018online},
            UniXGrad \cite{kavis2019unixgrad},
            AC-FGM \cite{li2023simple}, Prodigy \cite{mishchenko2024prodigy}, and
            USFGM \cite{rodomanov2024universal}.}
  \label{tab:compare_with_others}
	\centering 
        {\small 
	\begin{tabular}{c@{\hspace{\mycolspace}}c@{\hspace{\mycolspace}}c@{\hspace{\mycolspace}}c@{\hspace{\mycolspace}}c@{\hspace{\mycolspace}}c@{\hspace{\mycolspace}}c}\toprule[.1em]
            \textbf{Algorithm} & \makecell{\textbf{Last}\\ \textbf{iterate}} & \makecell{\textbf{Smooth}\\ \textbf{problems}} & \makecell{\textbf{Non-smooth}\\ \textbf{problems}}  & \makecell{\textbf{Unbounded}\\ \textbf{domain}} & \makecell{\textbf{Stoch.}\\ \textbf{gradients}} & \makecell{\textbf{Can increase}\\ \textbf{step size}} \\
            \midrule
            AcceleGrad 
             & \xmark & \xmark \tnote{\color{red}(1)} & \cmark & \xmark & \cmark & \xmark \\
            UniXGrad & \xmark & \cmark & \cmark & \xmark & \cmark & \xmark  \\
            AC-FGM & \xmark & \cmark & \cmark & \cmark & \xmark & \cmark \\
            Prodigy  & \xmark & \cmark & \cmark & \cmark & \xmark & \cmark \\
            USFGM  & \cmark & \cmark & \cmark & \xmark & \cmark & \xmark \\
            \hline 
            \SPS* (our result) & \xmark & \cmark & \cmark & \cmark & \cmark & \cmark    \\
            \IAM{} (ours) & \cmark  & \cmark & \cmark & \cmark & \cmark & \cmark \\
		\bottomrule[.1em]
    \end{tabular}
    }
    }
    \begin{tablenotes}
      {\footnotesize
        \item [\color{red}(1)] AcceleGrad's smooth analysis is for deterministic problems.
      }
    \end{tablenotes}
  \end{threeparttable}
\end{table}

\subsection{\SPS* in practice: Approximating $f_\xi(x_*)$ and Safe-guards}
\label{sec:approx}


In practice, outside of the interpolation regime, it is unlikely that we would have access to $f_{\xi}(x_*)$. To derive a practical method that is more generally applicable, we would need to estimate  $f_{\xi}(x_*)$. Let us call this estimate $\ell_{\xi}^*$, and consider the step size 
\begin{equation}\label{eq:SPSlower}
     \gamma^{\SPS}_t := 
        \frac{(f_{\xi}(x_t) - \ell_{\xi}^*)_+}{\|g_t\|^2}.
\end{equation}
The estimates $\ell_{\xi}^*$
would have to be \emph{underestimates}, otherwise the resulting method would stop early. Indeed, as $x_t \rightarrow x_*$ we have that $f_{\xi}(x_t) \rightarrow f_{\xi}(x_*)$, and the step size~\eqref{eq:SPSlower} would be zero before reaching convergence.

There are two natural underestimates for $f_{\xi}(x_*)$. The first is to use  $\inf f_{\xi}$. This is the approach used in \SPS{$\max$}~\cite{SPS}. The advantage of using $\inf f_{\xi}$ is that it often can be computed, indeed if no weight decay is being used (no L2 regularization), then often $\inf f_{\xi}=0.$ Which brings us to the second approach, which is to simply use $0$ as an underestimate, which holds for the ubiquitous case of having a positive loss~\cite{ALI-G,orvieto2024NGN}.

An issue with using an underestimate is that the step size~\eqref{eq:SPSlower} can become too large,  potentially even being unbounded if $\norm{g_t} \rightarrow 0$, which could lead to divergence.

  To safeguard against taking exceeding large step sizes, we can use clipping~\cite{SPS}, dampening~\cite{orvieto2024NGN}, or a combination of both~\cite{ALI-G}.
  By clipping, we mean to take the minimum between the stepsize in~\eqref{eq:SPSlower} and a hyperparameter $\gamma_b>0$ as is done in~\citet{SPS} in the $\SPS_{\max}$ method\footnote{Though $\SPS_{\max}$ has an additional constant $c$ in $\frac{(f_{\xi}(x_t) - \ell_{\xi}^*)_+}{c\|g_t\|^2}$. }
  \begin{equation}\label{eq:SPSmax}
     \gamma^{\SPS_{\max}}_t := 
       \min \left\{ \frac{(f_{\xi}(x_t) - \ell_{\xi}^*)_+}{\|g_t\|^2}, \gamma_b \right\}.
\end{equation}
We refer to dampening by adding an additional constant $\epsilon$ to the denominator, as is done in~\citet{orvieto2024NGN}, ~\citet{slackpolyak} and~\citet{ALI-G}:
  \begin{equation}\label{eq:SPSdamp}
     \gamma^{\SPS{dam}}_t := 
       \frac{(f_{\xi}(x_t) - \ell_{\xi}^*)_+}{\|g_t\|^2+\epsilon}.
\end{equation}
In particular in~\citet{orvieto2024NGN}, this dampening parameter depends in the iteration and is proportional to $f_{\xi}(x_t).$

Thus we can view several practical variants of $\SPS{}$ as approximations of \SPS*, where $f_{\xi}(x_*)$ is replaced by an underestimate, and a further safeguard is included to avoid large step sizes. These safeguards can also be motivated through a variational viewpoint based on solving relaxations of the interpolation condition~\cite{slackpolyak}.


\section{Experiments}

\subsection{Non-Lipschitz Non-smooth Convex Problem}
\label{sec:exp-convex-only}
To model discrete events with a Poisson regression,  we need to solve
\begin{equation}\label{eq:GLM}
    \min_{w \in \R^d} \frac{1}{n}\sum_{i=1}\left( \ell(w^\top x_i) - y_i \log\big(\ell(w^\top x_i)\big)\right),
\end{equation} 
where $\ell: \R \mapsto \R$ is called the link function. One of the most commonly used link functions is the exponential function $\ell(z) = \exp{z}.$ With this link function~\eqref{eq:GLM} becomes
\begin{equation}\label{eq:GLM-poisson}
    \min_{w \in \R^d} \frac{1}{n}\sum_{i=1}\left(\exp(w^\top x_i) - y_i  w^\top x_i \right).
\end{equation} 

We fit two different data sets. The first data set is on diabetes patients sourced from~\cite{efron2004least}, 
which is a medical dataset containing information on 442 patients ($n$), each described by 10 physiological and lifestyle features ($d$). 
The second data set is a bike sharing  records~\cite{fanaee2014bike}   in Washington, D.C., over a two-year period (2011-2012). It includes a total of 17,379 data points, and 12 features such as weather conditions, seasonal information, and temporal data. The target variable is the count of total bike rentals on an hourly basis.

As a baseline, we ran \texttt{L-BFGS}~\cite{liu1989limited} in full batch mode, and \SGD{} with constant learning rate tuned across
$$ \gamma \in 0.001 \cdot \{0.01, 0.1, 0.5, 1.0, 2.0, 5.0, 20, 50\}. $$
Each method was given the same budget in terms of epochs.  To highlight how important the choice of the learning rate is,
in Figure~\ref{fig:poisson_reg} we plot the resulting  loss ($y$-axis) of the last iterate of each method for different learning rates ($x$-axis).   We find that the \IAM{} method converges to a loss that is comparable to \texttt{LBFGS} and \SGD{} with the best possible learning rate. Furthermore, \IAM{} is the only method guaranteed to converge on this non-smooth and non-Lipschitz objective.

\begin{figure}[h!]
    \centering
    \begin{minipage}[t]{0.48\textwidth}
        \centering
        \includegraphics[width=\textwidth]{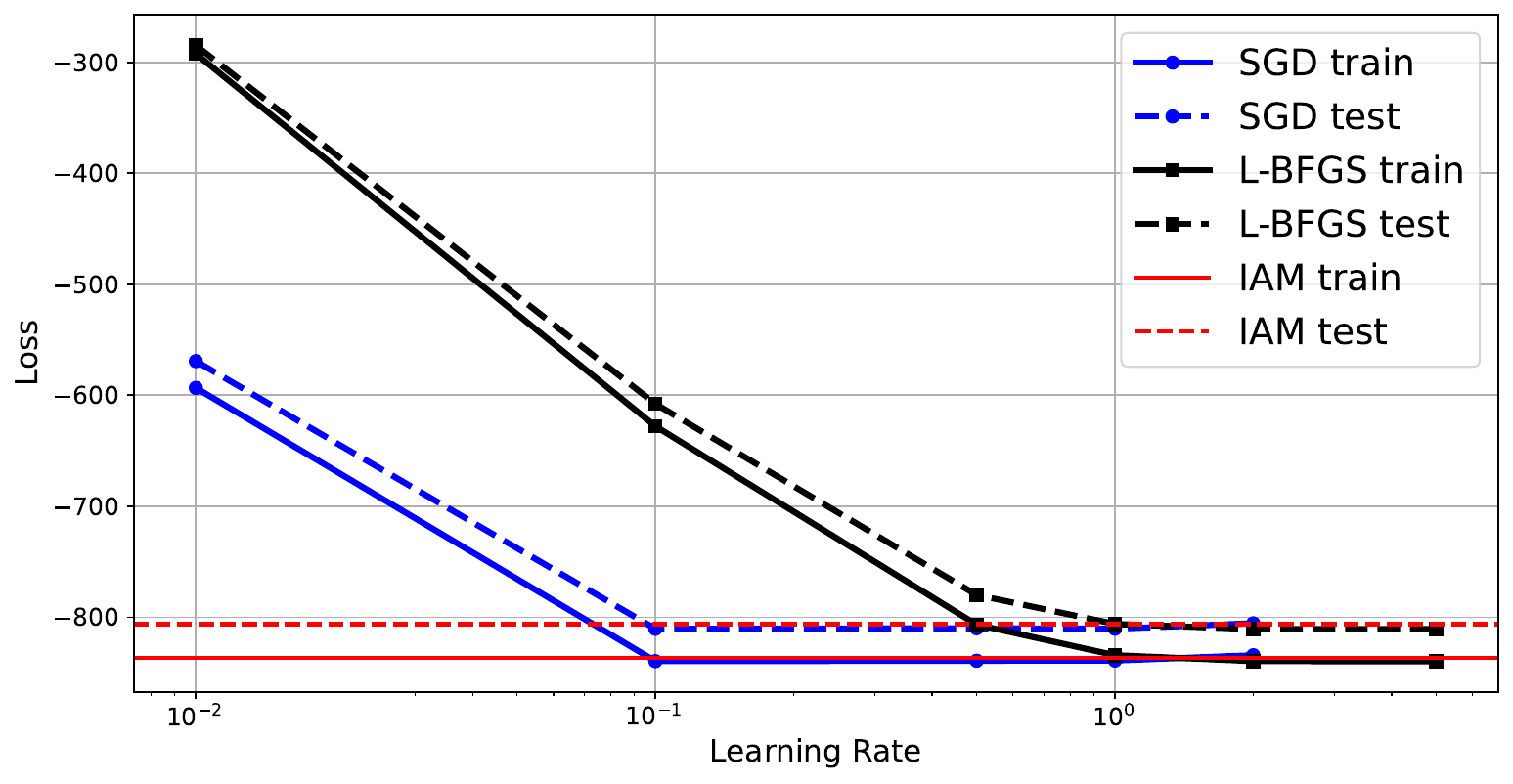}
        \caption*{Bike Sharing Data, 7 epochs}
        \label{fig:bike_sharing}
    \end{minipage}
    \hfill
    \begin{minipage}[t]{0.48\textwidth}
        \centering
        \includegraphics[width=\textwidth]{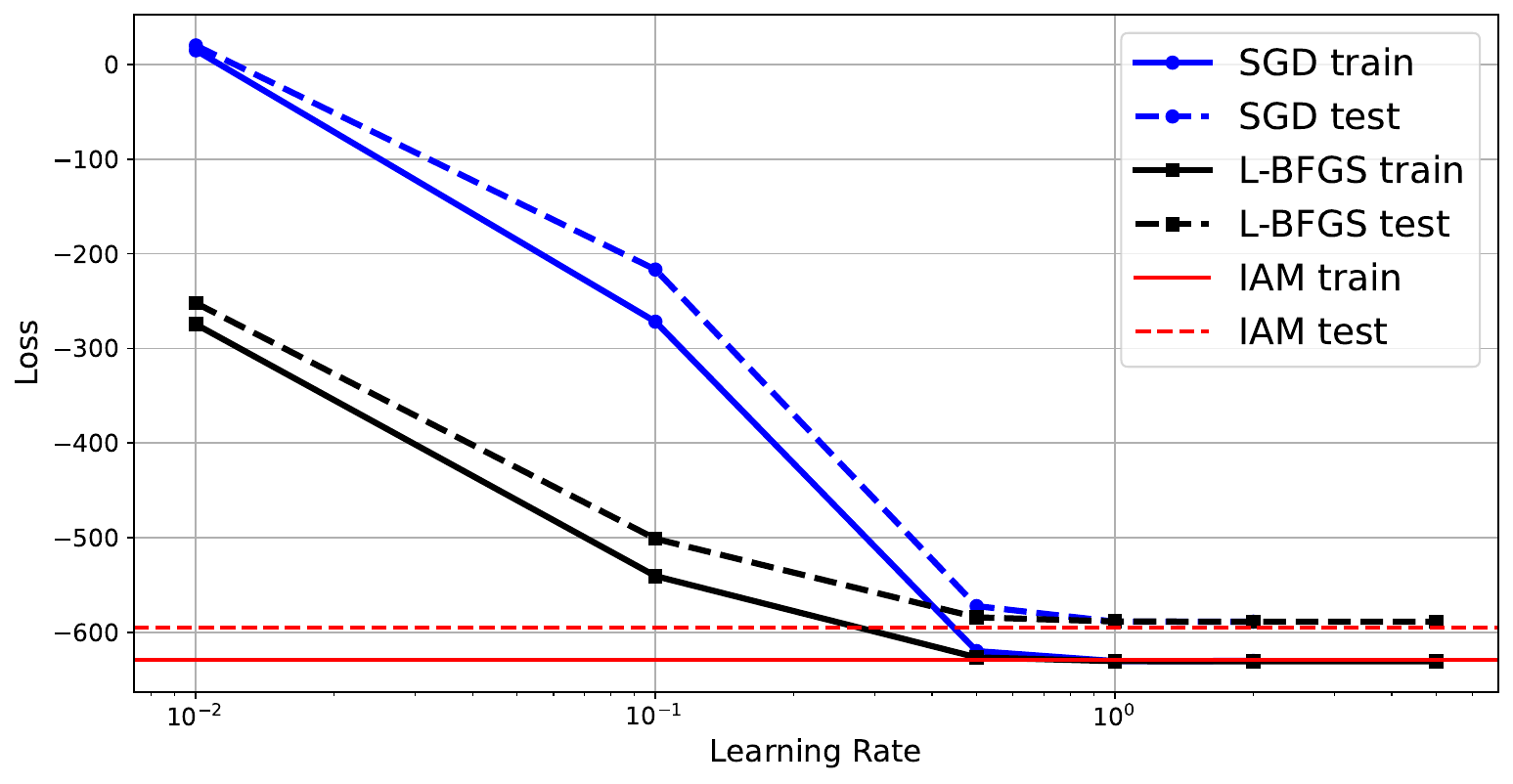}
        \caption{Diabetes Data, 15 epochs}
        \label{fig:diabetes}
    \end{minipage}
    \caption{Sensitivity to learning rate for each method. Larger learning rates diverged.}
    \label{fig:poisson_reg}
\end{figure}

\subsection{Misspecification of $f_\xi(x_*)$}
\label{sec:exp-misspecification}
In numerous machine learning applications a lower bound of $f_{\xi}(x_*)$ is known a priori, because loss functions are typically non-negative. 
We study the following three versions of \texttt{IAM}:
\begin{itemize}
    \item \textit{theoretical} version where we specify correctly $f_{\xi_t}(x_*)$ in every iteration $t$, computed from the oracle values $f_i^*,~i\in[n]$,
    \item \textit{averaged} version, where we specify $f_{\xi_t}(x_*)$ with $f(x_*)$ in every iteration,
    \item \textit{lower-bound} version, where we specify $f_{\xi_t}(x_*)$ with zero.
\end{itemize}

\paragraph{Description of experimental setup.}
Consider the following problem setup, which is adopted from \cite{Orvieto2022}: solve
\begin{align*}
    \min_{x\in \R^d} \frac{1}{n}\sum_{i=1}^n f_i(x), \quad f_i(x) := (x-x^i_*)^TH_i(x-x^i_*) + f_i^*,
\end{align*}
where $H_i \in \R^{d\times d}$ are symmetric positive definite matrices and $x_*^i\in \R^d$.
This is clearly an instance of \eqref{eq:prob}, where $\mathcal{D}$ is the uniform distribution over $[n]$ and $f_\xi(x) = f_i(x)$, and it holds $f(x) = \frac{1}{n}\sum_{i=1}^n f_i(x)$.

We consider two cases, (i) the interpolated case with $x_*^i=\bar{x}$ for all $i\in [n]$, and (ii) $x_*^i=\bar{x} + 0.05\varepsilon_i$, where $\varepsilon_i \in \R^d$ is standard normal. 
Following \cite{Orvieto2022}, we generate $H_i=A_i^TA_i / (3d)$ where the entries of $A_i \in \R^{3d\times d}$ are standard normal. We generate $f_i^*$ from a uniform distribution with mean $0.5$ and standard deviation $\nu$, followed by truncation at zero to make sure all $f_i^*$ are non-negative.

Note that in case (i) $\bar{x}$ is the minimizer of $f$ and of each $f_i$. Further, $f_{i}(x_*) = \inf_{x\in \R^d} f_i(x) = f_i^*$, and $f(x_*) = \frac{1}{n}\sum_{i=1}^n f_i^* = \inf_{x\in \R^d} f(x)$.
In the other case (ii), we compute the solution $x_*$ by solving a linear system, and then compute $f_i(x_*)$. We always compute $f_{\xi}(x_*)$ by averaging $f_i(x_*)$ over the corresponding mini-batch.

We vary the standard deviation $\nu\in \{0.01,0.1\}$ and the batch size $b\in\{4,16\}$.

We run all versions of \IAM{} with $\lambda_t=9$ for all $t\geq 0$, as suggested by our convergence Theorems~\ref{theo:smooth} and~\ref{theo:nonsmooth}. As a baseline, we compare to \texttt{SGD-M} with constant learning rate and momentum $\beta=0.9$. We set the learning rate to the theoretical value $\frac{1}{4L_{\max}}$ (cf.\ \citet{pmlr-v134-sebbouh21a}), where $L_{\max}:=\max_{i=1,\dots,n}L_i$ and $L_i:=2\lambda_{\max}(H_i)$ denotes the smoothness constant of $f_i$ (here $\lambda_{\max}$ denotes the largest eigenvalue).\footnote{Note that in \texttt{Pytorch} this requires setting \texttt{dampening}=0.9.} We further compare to \texttt{MoMo} that has access to $f_\xi(x_*)$, cf.\ \citep[Eq.\ 17]{Schaipp2024a}.

\paragraph{Discussion.}
In the interpolated case, see \cref{fig:lb-ablation-true}, the theoretical version of \IAM{} matches the rate of \texttt{SGD-M} \emph{without any tuning}. However, if $f_\xi(x_*)$ is mis-specified, the convergence stales. This effect is more pronounced if the noise is large, or the batch size is small. 
In the non-interpolated case, see \cref{fig:lb-ablation-false}, we observe that the theoretical version of \IAM{} obtains a smaller final loss than \texttt{SGD-M}. This matches our theoretical result in the smooth setting, where we showed that we get convergence even if $\sigma_*^2 > 0$.

Compared to \texttt{MoMo}, we observe roughly the same convergence behaviour, with \texttt{IAM} typically having a slightly bigger slope. As a side note, we observe that \texttt{MoMo} also converges without interpolation, even though this case is not covered by the theory of \citet{Schaipp2024a}.

\begin{figure}[!htb]
    \centering
    \includegraphics[width=0.9\columnwidth]{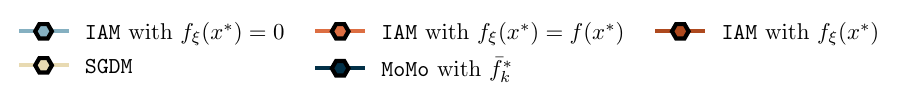}\\
    \includegraphics[width=0.32\columnwidth]{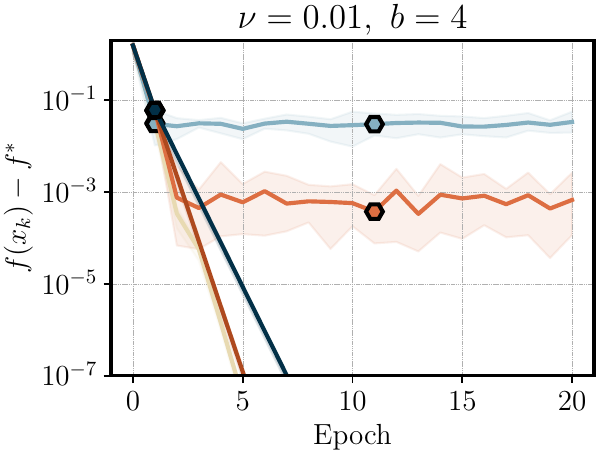}
    \includegraphics[width=0.32\columnwidth]{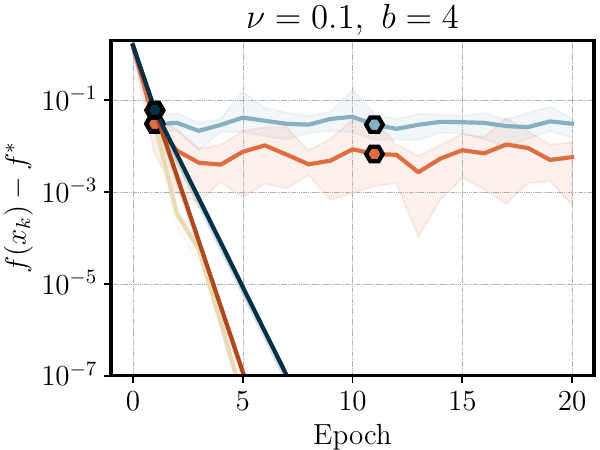}
    \includegraphics[width=0.32\columnwidth]{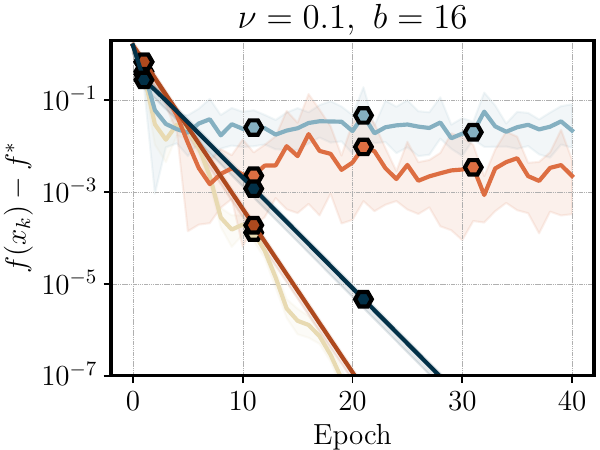}
    \caption{\textbf{Interpolation true:} \texttt{IAM} with the correct $f_{\xi_t}(x_*)$ converges as fast as \texttt{SGD-M} with the theoretical step size $\frac{1}{4L_{\max}}$. When $\nu$ is small \textbf{(left)}, the initial progress of \texttt{IAM} with the average $f(x_*)$ is equally good, before it stales. For $\nu$ large, the convergence stales earlier \textbf{(midlle)}. Increasing the batch size \textbf{(right)} slightly increases the gap between \IAM{} with $f_{\xi_t}(x_*)=0$ and $f_{\xi_t}(x_*)=f(x_*)$.}
    \label{fig:lb-ablation-true}
\end{figure}
\begin{figure}[!htb]
    \centering
    \includegraphics[width=0.9\columnwidth]{plots/lb_ablation/legend_train_loss.pdf}\\
    \includegraphics[width=0.32\columnwidth]{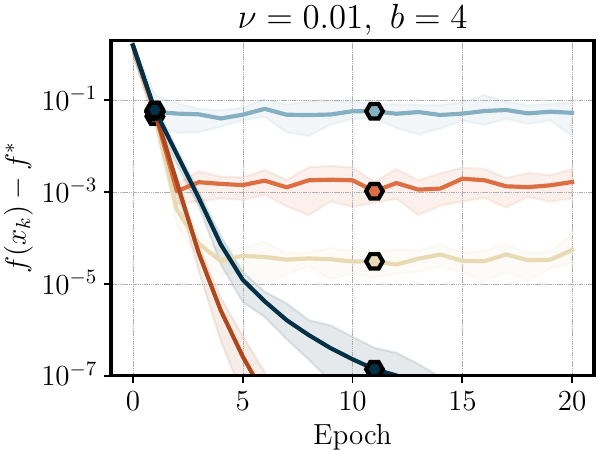}
    \includegraphics[width=0.32\columnwidth]{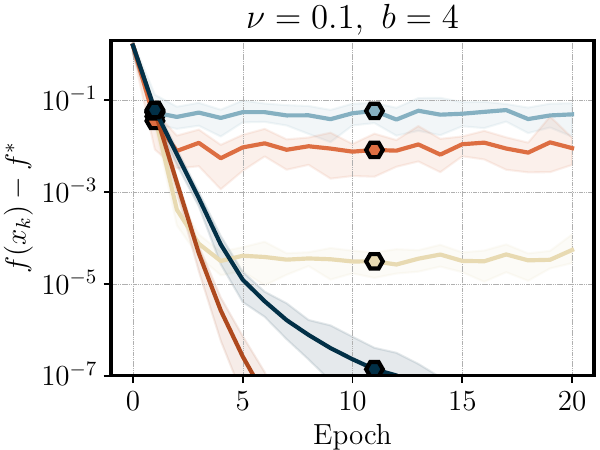}
    \includegraphics[width=0.32\columnwidth]{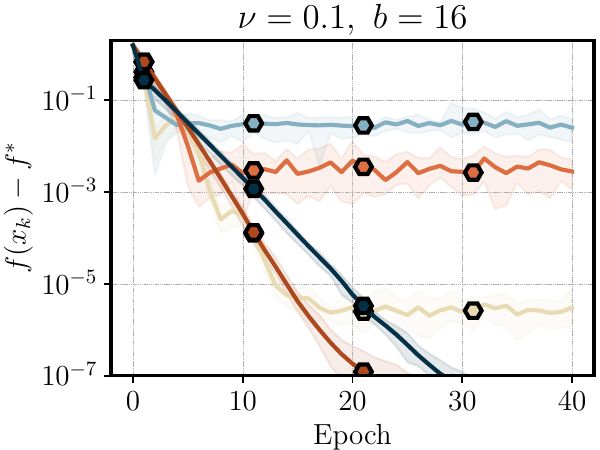}
    \caption{\textbf{Interpolation false:} see caption of \cref{fig:lb-ablation-true}. }
    \label{fig:lb-ablation-false}
\end{figure}

\subsection{Supplementary Material on Distillation Experiment}
\label{asec:distildetails}
Here we provide the complete details of our distillation experiments in \cref{sec:distillation}, together with some additional plots.

\paragraph{Datasets and models.}

The datasets we consider are below. We used the \texttt{GPT2Tokenizer} from the Transformers library.

\begin{itemize}
    \item \texttt{tinyShakespeare}~\citep{tinyshakespeare}: $40\,000$ lines from Shakespeare plays. The dataset has $303\,688$ tokens.

    Source: \url{https://huggingface.co/datasets/karpathy/tiny_shakespeare}

    \item \texttt{PTB} (Penn Treebank)~\citep{PTB}: The dataset contains $1\,094\,404$ tokens.

    Source: \url{https://huggingface.co/datasets/ptb-text-only/ptb_text_only}

    \item \texttt{Wikitext2}~\citep{wikitext-2}: This is a subset of a $100$ million token large collection of featured articles from Wikipedia. The dataset contains 
     $2\,389\,828$ tokens.

    Source: \url{https://huggingface.co/datasets/Salesforce/wikitext}
\end{itemize}

We use the module \texttt{GPT2LMHeadModel} from the HuggingFace \texttt{transformers} library \citep{Wolf2020} to define our \GPT{} models \href{https://huggingface.co/docs/transformers/v4.48.0/en/model_doc/gpt2#transformers.GPT2LMHeadModel}{[link]}. The teacher model we use for \texttt{tinyShakespeare} and \texttt{Wikitext2} is the \texttt{gpt2-large} configuration within this module, which has 774 million parameters, and was pretrained by a team from OpenAI~\citep{radford2019language}. We use the same tokenizer from the teacher model for the student model. For the \texttt{PTB} dataset, we use a different teacher model, as we found that \texttt{gpt2-large} had a poor fit, with the loss being above $5.0$ on this dataset. So instead, we used the \texttt{GPT-J-6B}  model~\citep{mesh-transformer-jax}, which has 6 billion parameters \href{https://huggingface.co/EleutherAI/gpt-j-6b}{[link]}. 

For the student model, we specify the configuration in \cref{tab:student-model-configs}.
\begin{table}[!htb]
    \centering
    \caption{Parameters of the Student GPT2 model}
    \label{tab:student-model-configs}
    \begin{tabular}{|c||c|c|c|}
        \hline
         Dataset & Embedding size & Number of layers & Number of attention heads \\
         \hline
         \texttt{tinyShakespeare} & $768$ & $2$ & $4$ \\
         \texttt{PTB} & $768$ & $2$ & $4$ \\
         \texttt{Wikitext2} & $1200$ & $12$ & $12$ \\
         \hline
    \end{tabular}
\end{table}

\paragraph{Hyperparameter tuning.}
For our methods \IAM{} (\cref{alg:IAM}) and \texttt{IAM-Adam} (\cref{alg:IMAP}) we set $\lambda_t=9$, which corresponds to using momentum $\beta =0.9$ if the learning rates were constant, see~\Cref{L:momentum is IMA}. Note that there is no hyperparameter tuning at all for \IAM{} and \texttt{IAM-Adam}.

For the baseline methods \SGD{} and \texttt{Adam}, we do the following tuning:
we run bot \SGD{} and \texttt{Adam} with constant learning rate and with a \textit{warmup+cosine decay} schedule \citep{Loshchilov2017}. This schedule does a linear warmup over the first 20\% of iterations to a peak learning-rate $\gamma$, then performs a cosine decay to $0$ over the remaining steps.
For \texttt{Adam}, we set the learning rate to its default value of $10^{-3}$ for the constant schedule, and we set $\gamma=1.5\times 10^{-3}$ for the \textit{warmup+cosine decay} schedule.

For \texttt{SGD} the learning rate needs to be tuned to get a reasonable performance: for the constant schedule, we chose the best-performing learning rate from the set 
$$\gamma_{\text{constant}}\in \{0.0001, 0.001, 0.01, 0.05, 0.1, 0.2 \}.$$
When using a scheduler with \texttt{SGD}, we take the best-performing value $\gamma_{\text{constant}}$ and then independently tune the peak learning rate within the set
$$\gamma_{\text{constant}} \cdot \{1.2, 1.5, 2, 3, 5 \}.$$
For \SGD{} we use a momentum parameter of $0.9$. In the \texttt{Pytorch} implementation of \SGD{}, we also set the \texttt{dampening} parameter to $0.9$ to ensure comparability of the tuned learning rate to the one of \IAM{}. 

\paragraph{Relationship to existing distillation techniques.} In this paragraph, we aim to give a short overview over various distillation techniques which often vary in terms of their general setup and loss function. However, as model distillation is not the main focus of this paper, we point to the references below for additional background.
In their seminal work, \citet{Hinton2015} propose to minimize the KL divergence between the teacher and student output probabilities. Follow-up works use a loss function that combines KL divergence and the standard loss for the student task (e.g., cross-entropy loss for classification, squared loss for regression) \citep{Romero2014}. 
On the other hand, \citet{hsieh-etal-2023-distilling} propose to use the teacher output as surrogate labels in case of unavailable labeled training data for the students.
We also refer to \citet{Beyer2022} for an overview of training techniques that improve the distillation performance.

The distillation setup that we propose in this paper is slightly different: we use only the final batch loss of the teacher model. The reason for this is that the \IAM{} methods we investigate rely on an accurate guess of the optimal batch loss $f_\xi(x_*)$. In the distillation setting, we can leverage the pretrained teacher model in order to approximate the optimal batch loss values. The notion of distillation we use here might of independent interest, as it only needs access to the final batch loss value, but not the output probabilities of the model (the \emph{logits}) nor its weights.

\paragraph{Additional plots.} In~\Cref{fig:distill-full-adam} we give the full plot of our distillation experiments, including the evolution of the learning rates for \texttt{IAM} and \texttt{IAM-Adam}.

\begin{figure*}[t]
    \centering
     \includegraphics[width=0.8\textwidth]{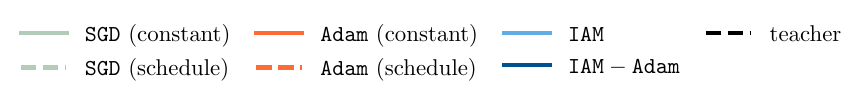}
      \begin{minipage}[t]{0.32\textwidth}
        \centering
        \includegraphics[width=\textwidth, trim=0 0.1in 0 0, clip]{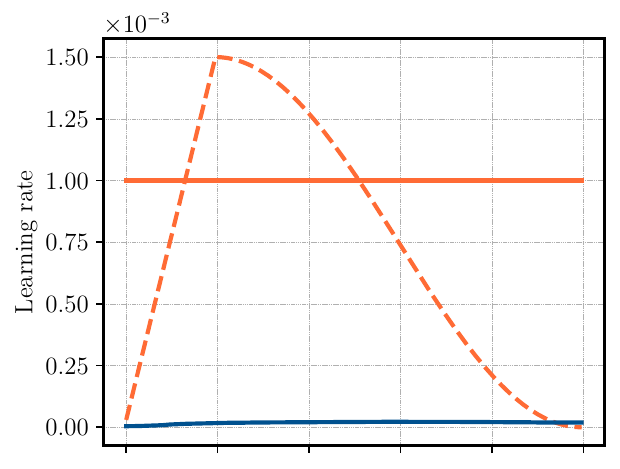}
    \end{minipage}
    \hfill
    \begin{minipage}[t]{0.32\textwidth}
        \centering
        \includegraphics[width=\textwidth, trim=0 0.1in 0 0, clip ]{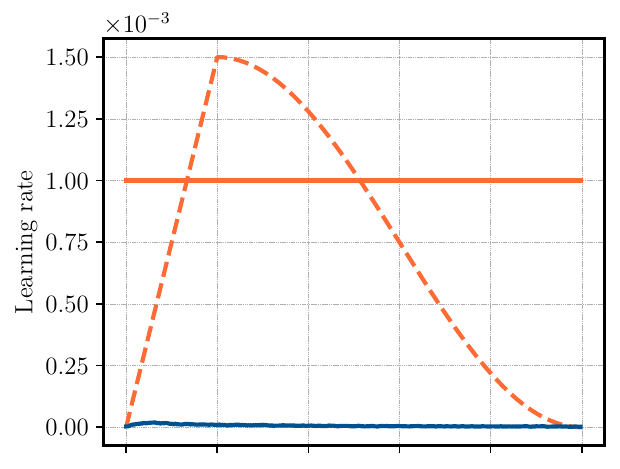}
    \end{minipage}
    \hfill
    \begin{minipage}[t]{0.32\textwidth}
        \centering
        \includegraphics[width=\textwidth, height = 0.7\textwidth, trim=0 0.1in 0 0, clip]{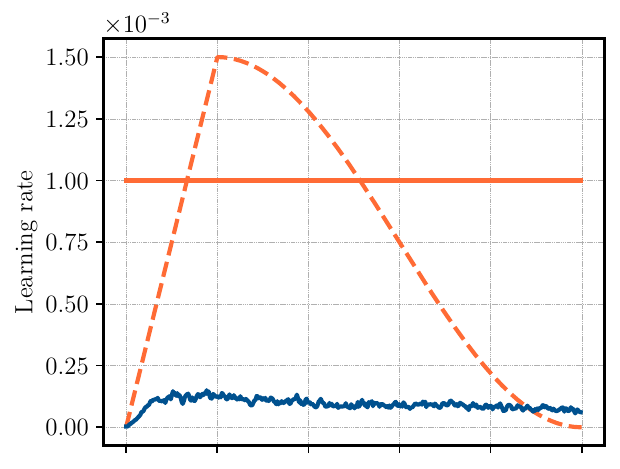}
    \end{minipage}

    \begin{minipage}[t]{0.32\textwidth}
        \centering
        \includegraphics[width=\textwidth, trim=0 0.1in 0 0, clip]{plots/GPT_distill/gpt_distill-lr-tiny_shakespeare_med.pdf}
    \end{minipage}
    \hfill
    \begin{minipage}[t]{0.32\textwidth}
        \centering
        \includegraphics[width=\textwidth, trim=0 0.1in 0 0, clip ]{plots/GPT_distill/gpt_distill-lr-ptb_text_med_j_05.pdf}
    \end{minipage}
    \hfill
    \begin{minipage}[t]{0.32\textwidth}
        \centering
        \includegraphics[width=\textwidth, trim=0 0.1in 0 0, clip]{plots/GPT_distill/gpt_distill-lr-wikitext-2.pdf}
    \end{minipage}

    \begin{minipage}[t]{0.32\textwidth}
        \centering
        \includegraphics[width=\textwidth]{plots/GPT_distill/gpt_distill-tiny_shakespeare_med.pdf}
        {\texttt{tinyShakespeare}} 
    \end{minipage}
    \hfill
        \begin{minipage}[t]{0.32\textwidth}
        \centering
        \includegraphics[width=\textwidth]{plots/GPT_distill/gpt_distill-ptb_text_med_j_05.pdf}
        {\texttt{PTB}}
    \end{minipage}
    \hfill
    \begin{minipage}[t]{0.32\textwidth}
        \centering
        \includegraphics[width=\textwidth]{plots/GPT_distill/gpt_distill-wikitext-2.pdf}
        {\texttt{Wikitext2}}
    \end{minipage}
    \caption{Full display of \cref{fig:distill}. Adaptive learning rate of \texttt{IAM-Adam} compared to \texttt{Adam} \textbf{(top)}, of \IAM{} compared to \SGD{} \textbf{(middle)}, and the cross-entropy training loss \textbf{(bottom)}. Black line marks the average teacher loss.}
    \label{fig:distill-full-adam}
\end{figure*}

In~\Cref{fig:tiny_shakespear_full} we give the distillation of several different small \GPT{} models for the \texttt{tinyShakespeare} data set.
\begin{figure*}[t]
    \centering
    \includegraphics[width=0.8\textwidth]{plots/GPT_distill/legend.pdf}
    \begin{minipage}[t]{0.32\textwidth}
        \centering
        \includegraphics[width=\textwidth, trim=0 0.1in 0 0, clip]{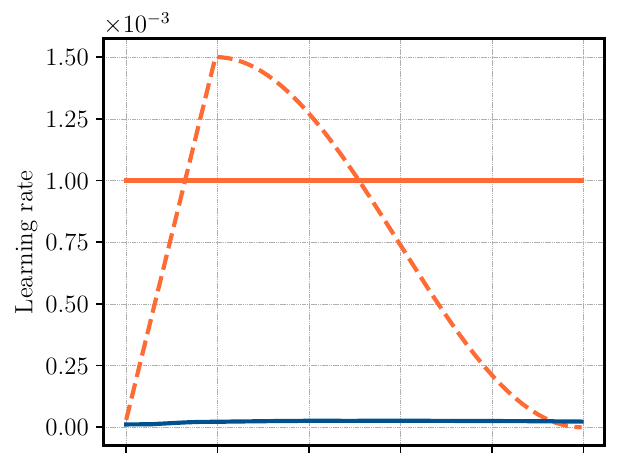}
    \end{minipage}
    \hfill
    \begin{minipage}[t]{0.32\textwidth}
        \centering
        \includegraphics[width=\textwidth, trim=0 0.1in 0 0, clip]{plots/GPT_distill/gpt_distill-lr-adam-tiny_shakespeare_med.pdf}
    \end{minipage}
    \hfill
    \begin{minipage}[t]{0.32\textwidth}
        \centering
        \includegraphics[width=\textwidth, trim=0 0.1in 0 0, clip]{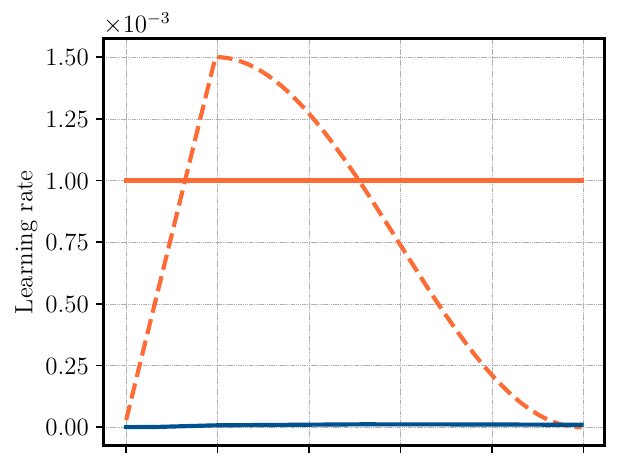}
    \end{minipage}
    \begin{minipage}[t]{0.32\textwidth}
        \centering
        \includegraphics[width=\textwidth, trim=0 0.1in 0 0, clip]{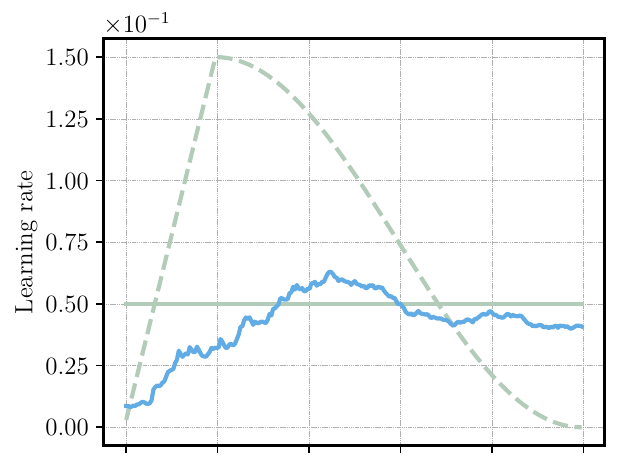}
    \end{minipage}
    \hfill
    \begin{minipage}[t]{0.32\textwidth}
        \centering
        \includegraphics[width=\textwidth, trim=0 0.1in 0 0, clip]{plots/GPT_distill/gpt_distill-lr-tiny_shakespeare_med.pdf}
    \end{minipage}
    \hfill
    \begin{minipage}[t]{0.32\textwidth}
        \centering
        \includegraphics[width=\textwidth, trim=0 0.1in 0 0, clip]{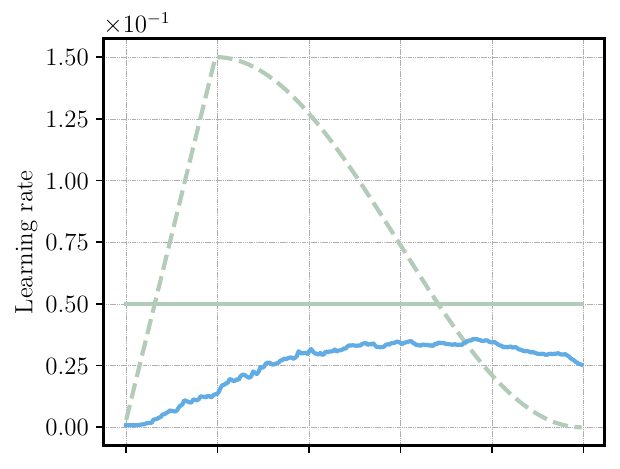}
    \end{minipage}

    \begin{minipage}[t]{0.32\textwidth}
        \centering
        \includegraphics[width=\textwidth]{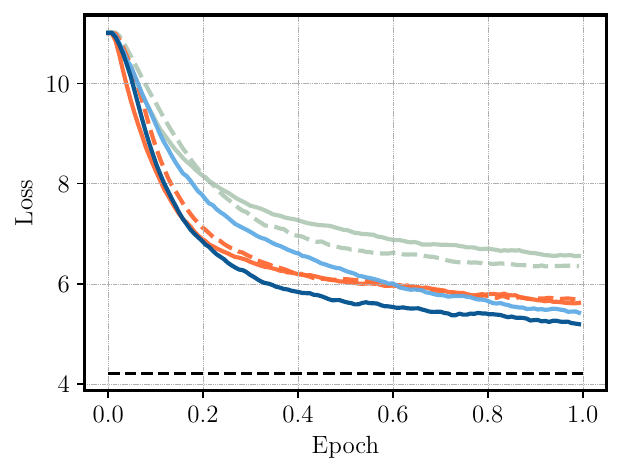}
        {\small n\_embd=768, n\_layer=2, n\_head=4}
    \end{minipage}
    \hfill
    \begin{minipage}[t]{0.32\textwidth}
        \centering
        \includegraphics[width=\textwidth]{plots/GPT_distill/gpt_distill-tiny_shakespeare_med.pdf}
        {\small n\_embd=768, n\_layer=4, n\_head=8}
    \end{minipage}
    \hfill
    \begin{minipage}[t]{0.32\textwidth}
        \centering
        \includegraphics[width=\textwidth]{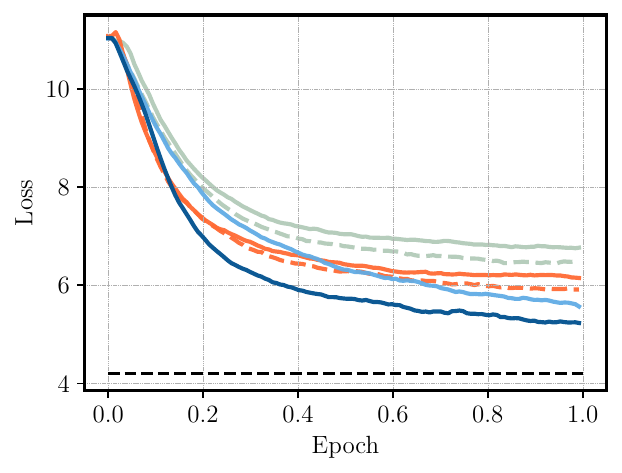}
        {\small n\_embd=1200, n\_layer=12, n\_head=12}
    \end{minipage}

    \caption{Distilling \texttt{gpt2-medium} into successively larger student models for the \texttt{tinyShakespeare} dataset.}
    \label{fig:tiny_shakespear_full}
\end{figure*}


\end{document}